\pdfoutput=1
\documentclass[10pt]{article}
\usepackage[preprint]{tmlr}

\usepackage[utf8]{inputenc}
\usepackage{amsmath,amssymb,amsthm,mathtools}
\usepackage{graphicx}
\usepackage{booktabs}
\usepackage{multirow}
\usepackage{array}
\usepackage{xcolor}
\usepackage{microtype}
\usepackage{nicefrac}
\usepackage{algorithm}
\usepackage{algpseudocode}
\usepackage{tikz}
\usetikzlibrary{arrows.meta,positioning,calc,fit}
\usepackage{flafter}
\usepackage{placeins}
\usepackage{enumitem}
\setlist{topsep=3pt,itemsep=1pt,parsep=0pt,partopsep=0pt}
\usepackage[hidelinks,pdfauthor={Prasanjit Dubey, Xiaoming Huo},pdftitle={Federated Language Models Under Bandwidth Budgets}]{hyperref}
\usepackage{xurl}
\usepackage{orcidlink}
\graphicspath{{figures/}}

\newtheorem{theorem}{Theorem}[section]
\newtheorem{proposition}[theorem]{Proposition}
\newtheorem{corollary}[theorem]{Corollary}
\newtheorem{lemma}[theorem]{Lemma}
\theoremstyle{definition}
\newtheorem{definition}[theorem]{Definition}
\theoremstyle{remark}
\newtheorem{remark}[theorem]{Remark}

\newcommand{\R}{\mathbb{R}}
\newcommand{\E}{\mathbb{E}}
\newcommand{\Prob}{\mathbb{P}}
\newcommand{\KL}{\mathrm{KL}}
\newcommand{\softmax}{\operatorname{softmax}}
\newcommand{\osc}{\operatorname{osc}}
\newcommand{\tr}{\operatorname{tr}}
\newcommand{\ones}{\mathbf{1}}
\newcommand{\Vocab}{\mathcal{V}}
\newcommand{\Ctx}{\mathcal{X}}
\newcommand{\Pstar}{P^{\star}}
\newcommand{\Phat}{\widehat{P}}
\newcommand{\Pbar}{\bar{P}}
\newcommand{\thetahat}{\widehat{\theta}}
\newcommand{\qhat}{\widehat{q}}
\newcommand{\Bcal}{B_{\mathrm{cal}}}
\newcommand{\ncal}{n_{\mathrm{cal}}}
\newcommand{\rcal}{r_{\mathrm{cal}}}
\newcommand{\Drag}{\varepsilon_{\mathrm{RAG}}}
\newcommand{\Dtrain}{\varepsilon_{\mathrm{train}}}
\newcommand{\Lm}{\mathcal{L}_m}
\newcommand{\Gm}{G_m}
\newcommand{\fmax}{f_{\max}}
\newcommand{\floc}{f_{\mathrm{loc}}}
\newcommand{\epsopt}{\varepsilon_{\mathrm{opt}}}
\newcommand{\epsfit}{\varepsilon_{\mathrm{fit}}}
\newcommand{\Cset}{\mathcal{C}}
\newcommand{\Yset}{\mathcal{Y}}
\newcommand{\DTheta}{D_{\Theta}}
\newcommand{\Smax}{S_{\max}}
\newcommand{\Nbin}{N_{\mathrm{bin}}}
\newcommand{\Xp}{X^{+}}
\newcommand{\KLbar}{\overline{\KL}}

\title{Federated Language Models Under Bandwidth Budgets:\\ Distillation Rates and Conformal Coverage}

\author{\name Prasanjit Dubey\,\orcidlink{0000-0002-3667-5507} \qquad
Xiaoming Huo\,\orcidlink{0000-0003-0101-1206}\\
\addr H.~Milton Stewart School of Industrial and Systems Engineering\\
Georgia Institute of Technology, Atlanta, GA 30332, U.S.A.}

\begin{document}

\maketitle

\begin{abstract}
Hospitals, companies, and research consortia increasingly want language models
that learn from, and answer questions over, data that must stay where it is
held. Their links carry only a limited number of bits, especially the satellite
links that serve remote clinics, ships, and field stations. The goal is to train
one shared model on every site's data and to answer questions with evidence from
every site, with a guarantee that the returned prediction set contains the correct
answer at the chosen coverage rate, while each site sends as few bits as possible.
The standard approach averages the sites'
model parameters, so each site's messages grow with the model and every site must run the
same architecture (network design). Methods that exchange predictions instead avoid both limits,
but, to our knowledge, none bounds what compressing those predictions costs; for
answering, federated conformal prediction covers finite-bit calibration messages
but not compressed per-question scores from several sites. We
introduce two methods that compress every message at a small and predictable
cost. Federated Probe-Logit
Distillation (FPLD) trains the shared language model: each site sends its model's
compressed predictions (logits) on shared unlabeled text (probes), and a central hub averages them
and trains one model to reproduce the average. Federated Conformal
Retrieval-Augmented Generation (FC-RAG) answers questions with a language model
such as the one FPLD trains: each site scores candidate answers using its own
documents (lower scores mean more plausible answers) and sends compressed scores. The hub returns every candidate whose
averaged score does not exceed a threshold calibrated on labeled examples; this
conformal prediction set contains the correct answer at the guaranteed rate. Both methods use one
compression scheme, which centers each message and rounds it after adding a
random offset known to the receiver. For logits whose spread fits a fixed clipping range, we bound the expected log-loss added by compression independently of vocabulary size; the bound depends only on the bits, the range, and the number of sites and decreases as more sites take part, so the number of bits can be chosen before deployment. We also prove that recalibrating after compression keeps the guarantee intact, and that FPLD's training error sets a bound on how far the guarantee can slip when a threshold is reused after a model update. In experiments with GPT-2 and Qwen2.5 language models,
sending 8 bits per number, a quarter of 32-bit full precision, causes no detectable loss
in the trained model or in the reliability of the answers. FPLD also trains a
shared model across different model architectures, where parameter averaging does
not apply; with one shared architecture, parameter averaging remains more
accurate.
\end{abstract}

\section{Introduction}\label{sec:intro}

Organizations that cannot pool their data increasingly need language models
that learn from all of it and answer questions over all of
it~\citep{wu2026fedllmsurvey,fedrag2025mapping}, while the links
between them carry only a limited number of bits. Hospitals, banks, and
scientific consortia keep records local for reasons of regulation, consent, or
policy, but would benefit from one model trained on every site's records and from
answers that draw on every site's documents. The bandwidth constraint is sharpest for sites linked by satellite, such as remote clinics,
ships, and field stations, whose uplinks often carry only hundreds of kbps to
about 10~Mbps~\citep{iridium2026certus700,mohan2024starlink,neff2021antarctic}, so a
full-precision upload can take from minutes to hours in every training round.
In such a \emph{federated} system~\citep{kairouz2021advances}, the sites, which we call \emph{nodes}, keep
their data and exchange messages only with a coordinating server, the
\emph{hub}. The system serves two purposes: training a common model, and
answering queries whose evidence may sit at any node, with a guarantee that the
answers are reliable.

Compressing the messages is unavoidable, but naive compression fails in both
tasks, so the number of bits must be chosen with its cost in view. For training,
each of the $K$ nodes sends its model's \emph{logits}, the unnormalized
log-probabilities that a language model assigns to every token (a word or word
piece) of its vocabulary as the next token. With GPT-2 (a language model pretrained on public text, with a 50{,}257-token
vocabulary) and 2{,}048 shared input texts, one full-precision (32-bit) message is 393~MiB per node per
round. In the first round, rounding these logits to 8 bits on a fixed range clips more than half of them and raises the trained model's \emph{perplexity}, the
exponential of its mean per-token log-loss on held-out text (lower is better),
from $40.5$ with full-precision logits to $145$ (Section~\ref{sec:exp_training}). For answering, a candidate answer is kept
when its score, a measure of how implausible the model finds it, does not exceed a threshold set from labeled examples. The reliability guarantee holds only if the scores used to answer are computed exactly
as those used to set the threshold, so rounding only the query scores voids it.

We formulate the problem as a budget-to-risk question: for each message a node
sends, how does its number of bits enter the error of the trained model and the
reliability of the answers? The system sends three kinds of messages. During
training, each node sends logits on shared unlabeled texts, the \emph{probes},
at $b$ bits per logit, in every round. Once, when the threshold is set
(calibration), each node sends a
$\Bcal$-bit summary of scores on held-out labeled examples. For each query, node
$i$ sends $B_i$ bits of scores for a fixed list of candidate answers, and the hub
returns a \emph{prediction set} of candidates that contains the correct answer
with probability at least $1-\alpha$, for a chosen miscoverage level $\alpha$ such as
$0.1$. This probability is the set's \emph{coverage}, and conformal prediction guarantees it~\citep{vovk2005,angelopoulos2023gentle}.

We measure training error by the
Kullback--Leibler (KL) risk of the trained model (its expected excess log-loss
over the true next-token distribution) and answer reliability by coverage.
Because models are updated after deployment, we also ask how far the scoring
model's training error, and through it the training bits, can move the coverage
of a threshold that is not recomputed.
Section~\ref{sec:setup} states the setting and the four problems we study, and
Figure~\ref{fig:pipeline} summarizes them.

This formulation fits the application better than existing ones: to our
knowledge, each leaves at least one of the three messages unpriced, and none ties
the coverage of a threshold reused after a model update to the training error of the model that
computes the scores (Section~\ref{sec:related}). Federated
averaging (FedAvg;~\citealp{fedavg2017}) and compressed parameter exchange~\citep{konecny2016federated,qsgd2017,karimireddy2019ef} average parameters coordinate by
coordinate, so their upload grows with the model, and they need one shared network
architecture (Section~\ref{sec:fpld}). Exchanging model outputs on probes keeps the upload
independent of model size and works across architectures that share a tokenizer
(the rule that splits text into tokens); federated distillation methods exchange
predictions in place of parameters~\citep{jeong2018fd,fedmd2019,itahara2023dsfl,cfd2022}, but none prices
their bits in a risk bound. Federated conformal prediction bounds
coverage for finite-bit calibration messages~\citep{zhu2024wfcp,wen2025dcp} but not for
compressed per-query scores from several nodes.

We introduce a training method that builds the federated language model and an
inference method that answers questions with it (Section~\ref{sec:protocols}).
The training method, Federated Probe-Logit Distillation (FPLD), pools what the
nodes learn without moving their data or their parameters: every node runs its
model on the probes and sends its logits; the hub averages the logits and trains a common \emph{student} model to reproduce
the average, a step called \emph{distillation}~\citep{hinton2015distilling}. The
inference method, Federated Conformal
Retrieval-Augmented Generation (FC-RAG), answers a question with evidence from
every node: each node retrieves passages from its own documents, as in
retrieval-augmented generation (RAG;~\citealp{lewis2020rag}), and uses a
language model, such as the FPLD student, to score the candidate answers. The hub
averages the scores and returns every candidate whose averaged score does not
exceed a threshold calibrated once on labeled examples, so that the returned set contains the correct answer with at least the chosen probability.

Both methods compress their
logits and scores with one \emph{codec} (encoder and decoder), \emph{gauge+dither}: it
\emph{centers} each logit vector (the \emph{gauge} step, a shift that leaves the
predicted probabilities unchanged) and adds a random offset shared with the hub,
the \emph{dither}, before rounding to $b$ bits. The hub subtracts the dither, so the rounding errors have mean
zero and partially cancel when it averages the nodes' messages.

In the full pipeline, FPLD trains
the model and FC-RAG deploys it as the scorer, and the two meet in one result:
when the threshold is not recalibrated after a model update, the coverage FC-RAG loses is bounded in terms of the FPLD student's training error, which depends on the training bits
(Section~\ref{sec:propagation}). FC-RAG's coverage guarantees hold for any scoring
model, so our real-model experiments test the two stages separately.

\paragraph{Contributions.}
The results let a system designer choose the bits per message before deployment and bound what that choice costs.
\begin{enumerate}
\item \textbf{The training budget can be chosen in advance, with a guarantee
(Section~\ref{sec:training}).} The cost of compressing the training message, in
expected log-loss of the nodes' averaged prediction, is at most a quantity that depends only on the compression settings and the number of nodes $K$, does not grow with the vocabulary, and falls inversely with $K$: $2.6\times10^{-3}$ nats (log-loss in
natural-log units) per probe at 8 bits and $K=4$ (Theorem~\ref{thm:codec}). A
resource law then shows how local data, probes, and bits each enter the
deployed model's risk (Theorem~\ref{thm:training}), and how much node heterogeneity
adds (Corollary~\ref{cor:het}).
\item \textbf{Compression need not cost reliability (Section~\ref{sec:inference}).}
If the decision threshold is recomputed through the compressed pipeline, the
full $1-\alpha$ coverage guarantee holds and compression affects only set size
(Proposition~\ref{prop:recal}); if the threshold calibrated on uncompressed scores is kept, the coverage loss
is bounded by a slack that, for a fixed bound on the score density, shrinks as
$K^{-1/2}$ (Theorem~\ref{thm:fixedref}). About 1~kbit per node for the one-shot
summary and 32 bits per node per query suffice in our
experiments.
\item \textbf{Training bits reach answer reliability through an explicit bound
(Sections~\ref{sec:propagation} and~\ref{sec:exp_components}).} If FC-RAG scores with the FPLD student, the coverage at stake when recalibration is skipped, for example after a model update, is bounded in terms of the student's training error (Corollary~\ref{cor:propagation}); this bound links the two methods.
\item \textbf{A quarter of the bandwidth costs no detectable accuracy or
reliability (Section~\ref{sec:experiments}).} On GPT-2 models of 124M and 355M parameters and on two corpora, 8-bit FPLD shows no detectable perplexity loss against
full-precision logits, and 4-bit FPLD, at one eighth of the upload, shows none at the
final round. Centering is needed at every bit width, and dither on coarse grids.
FPLD trains a common student across GPT-2 and GPT-Neo nodes, whose parameters cannot be averaged coordinate by coordinate. Eight-bit FC-RAG leaves coverage within $0.003$ of full precision with Qwen2.5 and GPT-2 as scoring models.
When all nodes share one architecture, averaging parameters remains more
accurate.
\end{enumerate}
We claim neither state-of-the-art accuracy, systems-level latency gains, nor
differential privacy.

\section{Problem formulation}\label{sec:setup}

\subsection{Motivating application}\label{sec:motivation}

Our motivating sites hold
private data and must send a message to a hub over slow satellite uplinks in every training round and for every query. Uplinks range from 352~kbps for Iridium
Certus~\citep{iridium2026certus700} to about 4--12~Mbps at the 75th percentile of
Starlink users~\citep{mohan2024starlink}, and a 2023 telemedicine measurement found
mean uploads of 7.6~Mbps over a geostationary service~\citep{buhl2025telemedicine}.
Some links also cap each message: Iridium Short Burst Data allows 340 bytes~\citep{groundcontrol_sbd}, which suits short per-query messages. Sending the raw data instead is ruled out by privacy and by volume: a South Pole
experiment produces more than 1~TB per day but transmits about 100~GB per day by
satellite~\citep{neff2021antarctic}.

\subsection{Setting}

\paragraph{Federated language models.}
A language model predicts the next token of a text from the tokens before it, and
a federated language model is one that the hub trains from data held at several
nodes, which share only messages. Let $\Vocab$ be a vocabulary of $V$ tokens and
$\Ctx$ the set of \emph{contexts}, the token sequences that can precede a token
(truncated to the model's maximum input length). A language model with parameters
$\theta$ maps a context $x\in\Ctx$ to a logit vector $\ell_\theta(x)\in\R^V$ and to the
next-token distribution $P_\theta(\cdot\mid x)=\softmax(\ell_\theta(x))$, where
$\softmax(\ell)_v=e^{\ell_v}/\sum_ue^{\ell_u}$; it assigns a text $y_1,y_2,\dots$ the
probability $\prod_tP_\theta(y_t\mid y_1,\dots,y_{t-1})$, so learning the
next-token distribution is learning a model of text. More generally, a
\emph{predictor} $P$ maps each context to a distribution on $\Vocab$. We write
$\KL(p\,\|\,q)=\sum_vp_v\log(p_v/q_v)$. The target is the true next-token
distribution $\Pstar(\cdot\mid x)$ with context marginal $\Pstar_X$.

Node $i\in\{1,\dots,K\}$ holds a dataset $\mathcal D_i$ of $n$ independent and
identically distributed (i.i.d.)\ examples from a law $P_i$ on
$\Ctx\times\Vocab$ and communicates only through \emph{messages}. A message from
node $i$ is a finite bit string $M_i=\mathrm{enc}_i(\mathcal D_i,U)\in\{0,1\}^{|M_i|}$,
computed by an encoder $\mathrm{enc}_i$ from the node's data and from randomness $U$
shared with the hub, and possibly from the hub's earlier broadcasts; its length
$|M_i|$ is its cost in bits. Everything the hub computes is a function of the
messages, $U$, its own broadcasts, and public inputs, never of the raw data. The
main results assume every node's data follows the target law: each
context-token pair $(X,Y)$ has $X\sim\Pstar_X$ and $Y\mid X\sim\Pstar$;
Corollary~\ref{cor:het} relaxes the second condition.

The goal of training is a
common predictor $\Phat$, the federated language model, with small \emph{KL risk},
where a predictor $P$ has
\begin{equation}\label{eq:risk}
R(P)\;:=\;\E_{X\sim\Pstar_X}\,\KL\bigl(\Pstar(\cdot\mid X)\,\big\|\,P(\cdot\mid X)\bigr),
\end{equation}
the expected excess log-loss of $P$ over the target. Experiments report
perplexity, whose logarithm is in expectation the conditional entropy of $\Pstar$
plus $R(P)$.

\paragraph{Probes.}
Training communicates only through model outputs on shared unlabeled texts, so no
private record or label leaves a node. The $m$ \emph{probe contexts}
$x_1,\dots,x_m$ are drawn i.i.d.\ from a probe distribution $Q$, shared by all
nodes, and independent of their data; public, licensed, or generated texts all
qualify. The probes must cover the target: every context the target can produce
has positive probability under $Q$ ($\Pstar_X\ll Q$), with worst-case density ratio
$\rho:=\operatorname*{ess\,sup}d\Pstar_X/dQ<\infty$, the \emph{probe overlap} ($\rho=1$
for probes drawn from the target).

\paragraph{Answers with a coverage guarantee.}
Under exchangeability, conformal prediction provides a distribution-free
guarantee on marginal coverage. At
answer time the answer space is a finite set $\Yset$ of $C$ candidates, such as the
options of a multiple-choice question. A \emph{nonconformity score} $s(x,y)$ is
large when $y$ is an implausible answer to $x$; we use negative log-probabilities.
For a miscoverage level $\alpha\in(0,1)$, split conformal prediction computes the
scores $S_j=s(X_j,Y_j)$ of $\ncal$ held-out labeled calibration examples, takes the
$k$-th smallest score, where $k=\lceil(\ncal+1)(1-\alpha)\rceil$, as the threshold $q_{1-\alpha}$
($+\infty$ if $k>\ncal$), and outputs $\Cset(X)=\{y: s(X,y)\le q_{1-\alpha}\}$;
computing the threshold is \emph{calibration}.

If calibration and test examples
are exchangeable (their joint law is invariant under permutations, as for i.i.d.\
draws), a test example $(X,Y)$ satisfies
$\Prob(Y\in\Cset(X))\ge1-\alpha$~\citep{vovk2005,lei2018distribution}, averaged over
calibration and test draws; this marginal probability is the \emph{coverage}, and smaller
sets $|\Cset(X)|$ are more informative. Sections~\ref{sec:inference}
and~\ref{sec:propagation} extend the guarantee to scores that are quantized (rounded to a finite grid) or
computed by an updated model.

\begin{figure}[t]
\centering
\begin{tikzpicture}[
  font=\scriptsize,
  node distance=3mm,
  stage/.style={align=left, text width=1.6cm, font=\scriptsize\bfseries},
  box/.style={draw=black!55, rounded corners=2pt, align=center, inner sep=2pt,
              minimum height=0.98cm},
  local/.style={box, text width=2.75cm, fill=white},
  codec/.style={box, text width=2.45cm, fill=black!6},
  hub/.style={box, text width=2.7cm, fill=white},
  risk/.style={box, dashed, text width=2.6cm},
  arr/.style={-{Stealth[length=3.5pt]}, black!70, line width=0.5pt}
]
\node[local] (n1) at (0,0) {node $i$: train\\ on local data; logits\\ on $m$ probes};
\node[stage, left=0.1cm of n1] (s1) {Training\\ budget $b$\\ {\normalfont FPLD, \S\ref{sec:fpld}}};
\node[codec, right=of n1] (c1) {gauge+dither,\\ $b$ bits per logit\\ ($mVb$ bits/round)};
\node[hub, right=of c1] (h1) {hub: average logits;\\ distill student $\Phat$};
\node[risk, right=of h1] (r1) {KL risk of $\Phat$:\\ sample, probe, codec\\ (Thms.~\ref{thm:codec}, \ref{thm:training})};

\node[local, below=0.32cm of n1] (n2) {node $i$: true-label\\ scores of its\\ calibration examples};
\node[stage, left=0.1cm of n2] (s2) {Calibration\\ budget $\Bcal$\\ {\normalfont FC-RAG, \S\ref{sec:fcrag}}};
\node[codec, right=of n2] (c2) {conservative histogram,\\ $\Bcal$ bits, once};
\node[hub, right=of c2] (h2) {hub: conformal\\ threshold $\hat q$};
\node[risk, right=of h2] (r2) {coverage loss: none\\ with a conservative\\ summary (Thm.~\ref{thm:fixedref})};

\node[local, below=0.32cm of n2] (n3) {node $i$: retrieve\\ from its corpus;\\ score $C$ candidates};
\node[stage, left=0.1cm of n3] (s3) {Query\\ budget $B_i$\\ {\normalfont FC-RAG, \S\ref{sec:fcrag}}};
\node[codec, right=of n3] (c3) {dithered scores,\\ $b_i$ bits each\\ ($B_i=Cb_i$ bits)};
\node[hub, right=of c3] (h3) {hub: average scores;\\ keep $y$ if $\tilde s(X,y)\le\hat q$};
\node[risk, right=of h3] (r3) {coverage loss $\Drag$\\ ${\propto}\,K^{-1/2}$, fixed $\fmax$\\ (Thm.~\ref{thm:fixedref})};

\foreach \a/\b in {n1/c1,c1/h1,h1/r1,n2/c2,c2/h2,h2/r2,n3/c3,c3/h3,h3/r3}
  \draw[arr] (\a) -- (\b);
\draw[arr] (h2.south) -- (h3.north);
\draw[arr, dashed] (h1.south) |- ([yshift=-0.16cm]n1.south) -- (n2.north);
\draw[arr, dashed] (r1.east) to[out=0, in=0, looseness=1.2]
  node[right, align=left, text width=1.35cm, text=black, font=\scriptsize\itshape]
  {training\\ error $\to$\\ coverage\\ loss $\Dtrain$\\ (Cor.~\ref{cor:propagation},\\ candidate-\\ normalized\\ scores)} (r3.east);
\end{tikzpicture}
\caption{Each communication budget enters one named risk term: every stage
encodes a local statistic in a finite message that the hub aggregates, and the
aggregation error appears in the KL risk or, when the threshold is not recalibrated, in the coverage. The dashed arrow
from the hub deploys the student $\Phat$ at the nodes, which use it to score answers in the
calibration and query rows. For candidate-normalized scores (negative log-probabilities of a
distribution over the candidates), Corollary~\ref{cor:propagation} turns the student's
training error into a coverage term. Section~\ref{sec:fcrag} defines the
conservative histogram. Our real-model experiments test the stages separately
(Section~\ref{sec:experiments}). Symbols: $K$ nodes; $m$ probes; $V$ vocabulary size; $b$ bits per logit and
$b_i$ bits per score at node $i$; $C$ candidates; $X$ a query and $y$ a candidate
answer; $\tilde s$ the hub's averaged score; $\qhat$ the calibrated threshold; $\fmax$ a bound on the density of the true
answer's averaged score; $\Drag$ and $\Dtrain$ the coverage slacks from quantized
query scores and from training error. Table~\ref{tab:ledger} lists the rates;
Tables~\ref{tab:notation} and~\ref{tab:glossary} (Appendix~\ref{app:assumptions}) give the
notation and a glossary.}
\label{fig:pipeline}
\end{figure}

\begin{table}[t]
\caption{Each budget shrinks its own risk term: one more bit per logit cuts
the training term more than fourfold, one more bit per score more than halves the
query slack, and a conservative calibration summary costs no coverage
(Section~\ref{sec:theory}; Figure~\ref{fig:pipeline} shows the messages). $C_K(b)=L^2/[6K(2^b-1)^2]$ is the
codec certificate, Theorem~\ref{thm:codec}'s bound on the log-loss cost of
compression, for clip level $L$; $\rho$ the probe overlap; $R(\Phat)$ the KL
risk~\eqref{eq:risk} of the distilled student $\Phat$; $\delta$ and $\rcal$ the failure
probability and radius of the threshold reconstruction, condition (B4) of
Section~\ref{sec:inference}; $\fmax$ a density bound, condition (B5) of
Section~\ref{sec:inference}; $\Smax$ the score-range width; $\Nbin\approx\Bcal/\log_2(\ncal+1)$ the
number of histogram bins. Coverage terms are losses when the threshold is not recalibrated; the
training row's bit term is informative only for fine grids (Section~\ref{sec:training}).}
\label{tab:ledger}
\centering\small
\setlength{\tabcolsep}{3.4pt}
\begin{tabular}{@{}l>{\raggedright\arraybackslash}p{0.2\textwidth}>{\raggedright\arraybackslash}p{0.33\textwidth}>{\raggedright\arraybackslash}p{0.25\textwidth}@{}}
\toprule
Stage & Budget & Risk term (result) & Scaling with the budget\\
\midrule
Training & $b$ bits per logit ($mVb$ bits per round) & ${\propto}\,\rho\,C_K(b)$ in $R(\Phat)$ (Thms.~\ref{thm:codec}, \ref{thm:training}) & $K^{-1}(2^b-1)^{-2}$\\
Calibration & $\Bcal$ bits, once & coverage: none if conservative, else $\delta+\fmax\rcal$ (Thm.~\ref{thm:fixedref}) & threshold inflation $\le\Smax/\Nbin\propto\Bcal^{-1}$\\
Query & $B_i=Cb_i$ bits per query & coverage: $\Drag$ (Thm.~\ref{thm:fixedref}) & $K^{-1/2}(2^{b_i}-1)^{-1}$ at fixed $\fmax$\\
\midrule
\begin{tabular}[t]{@{}l@{}}Training error\\$\to$ coverage\end{tabular} & --- & coverage: $\Dtrain$ (Cor.~\ref{cor:propagation}) & $\sqrt{R(\Phat)}$ for small risk\\
\bottomrule
\end{tabular}
\end{table}

\paragraph{Communication budgets.}
Each budget is the length $|M_i|$ of one kind of message, the resource a
constrained link limits. The training budget is $b$ bits per logit ($bV$ bits per
probe), the calibration budget $\Bcal$ bits per node, once, and the query budget $B_i$ bits
per query at node $i$; lengths exclude headers and transport overhead. Traffic from the nodes to the hub
is the \emph{uplink}, and traffic from the hub to the nodes the \emph{downlink};
downlinks, including the broadcast of the common model, are reported separately
(Appendix~\ref{app:exp_codec}). Figure~\ref{fig:pipeline} assembles these pieces:
each node encodes a local statistic in a finite message, the hub aggregates the
messages, and the aggregation error enters the KL risk of the trained model or, when the threshold is not recalibrated, the coverage of the answers; Table~\ref{tab:ledger} gives the rate of each term, and
Appendix~\ref{app:assumptions} collects the notation (Table~\ref{tab:notation}) and
a glossary (Table~\ref{tab:glossary}).

\subsection{Problems}

The application raises four questions, each answered by a result in
Sections~\ref{sec:theory} and~\ref{sec:experiments}.
\begin{itemize}
\item[\textbf{P1}] \emph{How many bits per logit does training need, and what does a
$b$-bit budget cost the deployed model?} A designer must fix $b$ from the link
capacity before training, so the cost must be known in advance and must not grow
with the vocabulary, which for language models has tens of thousands of tokens.
(Section~\ref{sec:training}.)
\item[\textbf{P2}] \emph{Does compressing calibration and query messages void the
coverage guarantee?} Users rely on coverage, and budgets may change after
calibration, for example when a link degrades. (Section~\ref{sec:inference}.)
\item[\textbf{P3}] \emph{How far can retraining the scoring model move the coverage of
a threshold that is not recomputed?} Models are updated after deployment, and
recalibration costs a calibration round. (Section~\ref{sec:propagation}.)
\item[\textbf{P4}] \emph{Do FPLD and FC-RAG work on real language models, including
when the nodes have different architectures?} Sites run different models, and the theory
covers a simplified model class, softmax-linear (multinomial logistic) models. (Section~\ref{sec:experiments}.)
\end{itemize}

\section{Method: compressed federated protocols}\label{sec:protocols}

Both protocols compress the logits and scores they send with one scheme, so a single analysis covers training and answering.

\subsection{The gauge+dither codec}\label{sec:codec}

\begin{definition}[Gauge-centered finite-grid subtractive-dither codec]\label{def:codec}
Fix $b\ge1$ bits, a clip level $L>0$, and the step $\Delta=2L/(2^b-1)$, so that the
grid $\{-L+j\Delta: j=0,\dots,2^b-1\}$ spans $[-L,L]$ and $L$ is the largest
centered magnitude it represents. To encode a logit row
$\ell\in\R^V$, the encoder (i) \emph{centers} it, $z=\ell-c\ones$, where $\ones$ is the all-ones vector and
$c=(\min_v\ell_v+\max_v\ell_v)/2$; (ii) draws a dither vector $U$, each coordinate
uniform on $[-\Delta/2,\Delta/2)$ and independent of $\ell$, from a pseudo-random
stream shared with the decoder; and (iii) sends the index
$J_v=\lfloor(z_v+U_v+L)/\Delta+\tfrac12\rfloor$ of the grid point nearest to $z_v+U_v$
for every $v$, that is, $bV$ bits.
The decoder returns $\tilde z_v=-L+\Delta J_v-U_v$, with no re-clipping. The
\emph{range condition} is $\max_v|z_v|\le L$, equivalently
$\max_v\ell_v-\min_v\ell_v\le2L$; coordinates that violate it are clipped to $\pm L$. For a scalar score, step (i) subtracts a fixed,
pre-shared center instead; we call this version the \emph{scalar codec}.
\end{definition}

Centering and dither cost no bits and, under the range condition, make each node's
rounding error exactly uniform, zero-mean, and independent of the logit, so errors
partially cancel when the hub averages. Centering is free because adding a constant to every logit
leaves the softmax unchanged (a choice of \emph{gauge}, in the physicist's term
for such a redundant shift), so $c$ is never sent; of all centers, the midrange
minimizes $\max_v|\ell_v-c|$. The dither is free because it is regenerated from
the shared stream. Under the
range condition nothing is clipped, and subtracting the dither makes the error
$\tilde z_v-z_v$ uniform on $[-\Delta/2,\Delta/2]$ and independent of
$z$~\citep{schuchman1964dither,gray1993dithered,gershogray1992}; among dithers
supported on one step, only the uniform law does so.

Unbiased stochastic rounding,
which rounds up with probability equal to the fractional position between grid
points, instead has an error variance that depends on the logit and reaches three
times $\Delta^2/12$, so its KL bound needs about $0.8$ more bits per logit to match
the dithered one (Appendix~\ref{app:proof_codec}). Without dither, the rounding error is a fixed
function of the logit and can bias every node the same way; with it, the
averaged error has variance $\Delta^2/(12K)$ per coordinate (Theorem~\ref{thm:codec}).

\subsection{Training: Federated Probe-Logit Distillation}\label{sec:fpld}

FPLD trains a common model from compressed outputs alone. It runs $T$ synchronous
rounds from a shared pretrained model (Algorithm~\ref{alg:fpld},
Appendix~\ref{app:pseudocode}). In each round, every node fine-tunes (continues
training) the current model on its data for $E$ epochs (passes over the data), evaluates its logits
$\ell_i(x_l)\in\R^V$ for the token following each probe, and uploads them with the codec (Algorithm~\ref{alg:codec}). The hub averages the decoded rows $\tilde z_i(x_l)$ of
Definition~\ref{def:codec}, $\tilde\ell(x_l)=K^{-1}\sum_i\tilde z_i(x_l)$, treats
$\softmax(\tilde\ell(x_l))$ as the \emph{teacher}, and fits the \emph{student}
$P_\theta$ by minimizing
\begin{equation}\label{eq:distill}
\Lm(\theta)\;=\;\frac1m\sum_{l=1}^m\KL\bigl(\softmax(\tilde\ell(x_l))\,\big\|\,P_\theta(\cdot\mid x_l)\bigr).
\end{equation}
The student has the nodes' architecture; the hub broadcasts it (downlink) for the
next round.

The uplink does not depend on the model size: for GPT-2's vocabulary and $m=2{,}048$
probes, 8-bit FPLD uploads $98.2$ MiB per node per round, whatever the number $d$
of parameters. A dense $b$-bit parameter update sends all $d$ coordinates and
costs about $db$ bits, so
8-bit FPLD uploads less than a 4-bit update once $d>2mV\approx206$ million. Over
the 352-kbps Iridium Certus uplink of Section~\ref{sec:motivation}, an 8-bit upload
takes about 39 minutes per node, against 2.6 hours with full-precision logits.

\paragraph{Nodes with different architectures.}
The output interface needs only a shared vocabulary and tokenizer, so nodes of
different architectures can take part. In \emph{function-space
synchronization}~\citep[cf.][]{taya2022consensus},
nodes keep their own architectures, the hub keeps a student of one chosen
architecture, and, from the second round on, the hub broadcasts the student's
next-token distributions on the probes (an $m\times V$ matrix, counted as
downlink) instead of parameters; each node distills toward them before training
locally (Appendix~\ref{app:pseudocode}).

\subsection{Inference: Federated Conformal RAG}\label{sec:fcrag}

FC-RAG answers a query with evidence from every node while each node sends only
a few compressed scores. It selects answers from a candidate list $\Yset$, shared in advance in a fixed order. For a query $X$, node $i$ retrieves passages
$Z_i$ from its own documents and scores every candidate $y\in\Yset$ with the deployed
model $\Phat$, such as the FPLD student:
$s_i(X,y)=-\log\Phat(y\mid X,Z_i)$, the negative log-probability of the candidate's tokens, so lower scores mean more plausible answers. Our experiments use the length-normalized negative log-probability, the average over the candidate's tokens, so that answers of different lengths are comparable. Node $i$ uploads the $C$ scores
with the scalar codec at $b_i$ bits each, so $B_i=Cb_i$. The hub averages the
decoded scores $\tilde s_i(X,y)$ into $\tilde s(X,y)=K^{-1}\sum_i\tilde s_i(X,y)$ and
returns $\Cset(X)=\{y:\tilde s(X,y)\le\qhat\}$ for a calibrated threshold $\qhat$
(Algorithm~\ref{alg:fcrag}).

In the default \emph{all-node} variant, every node scores every candidate; in the \emph{routed} variant, only the top-ranked nodes, chosen by a short relevance code
from each node (a quantized similarity between the query and the node's best
passage), score the candidates (Appendix~\ref{app:pseudocode}).

\paragraph{One-shot calibration with a finite summary.}
Beyond the per-query exchange, calibration costs each node one summary message,
and labels never leave their nodes. Each labeled calibration question is held by one node, its owner, and goes
through the per-query exchange, after which the hub returns the averaged scores
and the owner keeps the true answer's score.
Each owner then counts its scores in $\Nbin$ pre-shared bins and sends the counts
once, in $\Bcal$ bits; the hub adds the counts and returns, as $\qhat$, the upper
edge of the bin that holds the $k$-th smallest score. This \emph{conservative histogram} never
lowers the threshold and raises it by at most one bin width, $\Smax/\Nbin$ for
scores in a range of width $\Smax$ (Appendix~\ref{app:pseudocode}).

\section{Theoretical results}\label{sec:theory}

This section answers P1 (Section~\ref{sec:training}), P2
(Section~\ref{sec:inference}), and P3 (Section~\ref{sec:propagation}); proofs are in
Appendices~\ref{app:proofs_training} and~\ref{app:proof_inference}.

\subsection{The cost of the training budget (P1)}\label{sec:training}

The training results turn the budget $b$ into KL risk: Theorem~\ref{thm:codec}
prices compression from the codec settings and $K$ alone, and Theorem~\ref{thm:training}
carries that price into the deployed student. In words, (A1) makes the target a
softmax-linear (multinomial logistic) model with a standard estimator at each node, (A2) asks the
probes to cover the target and determine every parameter, (A3) makes the nodes'
data identically distributed, (A4) and (A6) allow inexact training and
distillation, and (A5) is the codec's operating condition
(Appendix~\ref{app:assumptions} discusses them).
\begin{itemize}
\item[(A1)] \emph{Softmax-linear representability.} There are features
$\Phi:\Ctx\to\R^{V\times d}$ with rows $\|\Phi(x)_v\|_2\le C_\Phi$, a convex compact
$\Theta\subset\{\theta:\|\theta\|_2\le\DTheta\}$, and an interior $\theta^\star$ with
$\Pstar(\cdot\mid x)=p_{\theta^\star}(x)$, where $p_\theta(x):=\softmax(\Phi(x)\theta)$ defines
$P_\theta$. The Fisher information $I=\E_{\Pstar_X}[\Phi(X)^\top H(p_{\theta^\star}(X))\Phi(X)]$,
with $H(p)=\operatorname{diag}(p)-pp^\top$ the Hessian of the log-sum-exp function, is
positive definite, and each node's maximum-likelihood estimator (MLE) over $\Theta$,
from its $n$ examples, is asymptotically efficient in mean square:
$n\,\E[(\thetahat^{\mathrm{MLE}}-\theta^\star)(\thetahat^{\mathrm{MLE}}-\theta^\star)^\top]\to I^{-1}$,
with bias $o(n^{-1/2})$ and tail probabilities
$\Prob(\|\thetahat^{\mathrm{MLE}}-\theta^\star\|_2>t)=o(1/n)$ for every $t>0$
(Appendix~\ref{app:assumptions}).
\item[(A2)] \emph{Probe overlap and identifiability.} The probes are i.i.d.\ from
$Q$, independent of the local data and all codec randomness; $\Pstar_X\ll Q$ with
$\rho<\infty$; and, almost surely,
with $\lambda_{\min}$ the smallest eigenvalue,
$\lambda_m:=\inf_{\theta\in\Theta}\lambda_{\min}\bigl(\tfrac1m\sum_l\Phi(x_l)^\top
H(p_\theta(x_l))\Phi(x_l)\bigr)>0$. The constant $\lambda_m$ is the smallest curvature of
the distillation objective~\eqref{eq:distill}; it is positive when no direction
of $\theta$ leaves the predicted distribution at every probe unchanged.
\item[(A3)] \emph{Homogeneous nodes.} The local datasets are independent, and
every node draws $X\sim\Pstar_X$ and $Y\mid X\sim\Pstar$.
\item[(A4)] \emph{Local optimization slack.} Node $i$'s final local parameter
$\thetahat_i\in\Theta$, after its last fine-tuning, satisfies
$\|\thetahat_i-\thetahat_i^{\mathrm{MLE}}\|_I\le\sqrt{2\epsopt}$, with $\|u\|_I^2=u^\top Iu$,
and $(\thetahat_1,\dots,\thetahat_K)$ is jointly independent of the probes, as when $T=1$,
when the last distillation draws fresh probes, or when each $\thetahat_i$ is a unique local MLE; the
analysis concerns the last round, and earlier rounds enter only through this
slack.
\item[(A5)] \emph{Codec.} Probe logits are encoded with Definition~\ref{def:codec}
under the range condition; for each probe and coordinate, the $K$ nodes' dithers
are independent of one another; and the last round's dither is independent of
the probes and $\thetahat_1,\dots,\thetahat_K$ jointly.
\item[(A6)] \emph{Distillation slack.} The student $\Phat=P_{\thetahat}$ has
$\thetahat\in\Theta$ and $\Lm(\thetahat)\le\min_{\Theta}\Lm+\epsfit$.
\end{itemize}
The assumptions cover softmax-linear models but not fully fine-tuned neural
networks such as GPT-2, so the real-model experiments test the protocols beyond the formal
scope. Condition (A1) covers tabular n-gram models and multinomial logistic
regression on fixed features, such as a softmax output layer trained on frozen
features, in each case with one reference logit fixed so that $I$ is
invertible. The range condition in (A5) holds for every $b\ge1$ once
$L\ge C_\Phi\DTheta$.

\begin{theorem}[Codec certificate]\label{thm:codec}
Fix $K,V\ge1$, a temperature $\tau>0$, which divides the logits before the softmax
($\tau=1$ in our protocol), and node logit rows
$\ell_1,\dots,\ell_K\in\R^V$ encoded as in Definition~\ref{def:codec} under the range
condition, with centered rows $z_i$, decoded rows $\tilde z_i$, and each
coordinate's dithers independent across nodes; $\E_U$ is the expectation over the
dithers. Let
$\bar z=K^{-1}\sum_iz_i$, $\tilde{\bar z}=K^{-1}\sum_i\tilde z_i$, $p=\softmax(\bar z/\tau)$, and
$\tilde p=\softmax(\tilde{\bar z}/\tau)$. Then $p=\softmax(K^{-1}\sum_i\ell_i/\tau)$, so the centers need not be transmitted, and conditionally on the logits, for every coordinate $v$,
$\E_U[\tilde{\bar z}_v-\bar z_v]=0$, $\E_U[(\tilde{\bar z}_v-\bar z_v)^2]=\Delta^2/(12K)$, and
\begin{equation}\label{eq:certificate}
\E_U\,\KL(p\,\|\,\tilde p)\;\le\;K\log\frac{\sinh a}{a}
\;\le\;\frac{\Delta^2}{24K\tau^2}\;=\;\frac{L^2}{6K\tau^2(2^b-1)^2},
\qquad a=\frac{\Delta}{2K\tau}.
\end{equation}
\end{theorem}

Theorem~\ref{thm:codec} answers the first half of P1: a designer can choose $b$
before training, because the certificate depends only on $(b,L,K)$, has no factor of $V$
at a fixed clip level, and needs no independence across coordinates; at $\tau=1$,
any $b\ge\log_2\bigl(1+L/\sqrt{6K\eta_B}\bigr)$ keeps it within a tolerance $\eta_B$. One more
bit cuts it more than fourfold, since $2^{b+1}-1>2(2^b-1)$, and doubling $K$ halves
it. It is $2.62\times10^{-3}$ nats per probe at $L=64$, $K=4$, $b=8$, $\tau=1$, the setting of those real-model runs in which no coordinate is clipped (Sections~\ref{sec:exp_training}--\ref{sec:exp_hetero}). At the
$L=40$ of our main GPT-2 experiments (Section~\ref{sec:exp_training}), the formula gives $1.0\times10^{-3}$ nats at 8 bits but $0.30$
nats at 4 bits.

The certificate bounds the log-loss that transmission adds, and to leading order
it is tight up to a known factor. At a probe $x_l$, $\tilde{\bar z}$ is the hub's
row $\tilde\ell(x_l)$, and because the dither error has mean zero,
$\E_U\KL(p\|\tilde p)$ is exactly the expected increase in log-loss that
transmission causes under any next-token law fixed with respect to the dither, the
true one included (Appendix~\ref{app:proof_codec}). With dithers also independent across coordinates, the leading term
is $1-\|p\|_2^2$ times the certificate (Lemma~\ref{lem:trace});
Corollary~\ref{cor:hp_codec} gives pathwise and high-probability versions.

The deployed model is the student, not the teacher, so the second half of P1 asks
what $b$ costs the student. Let $\bar\theta=K^{-1}\sum_i\thetahat_i$ and
$\Pbar=P_{\bar\theta}$; under (A1), where logits are linear in $\theta$, $\Pbar$ is the teacher the hub would form without
quantization.

\begin{theorem}[Training resource law]\label{thm:training}
Assume (A1)--(A6) and let $C_K(b)=L^2/[6K(2^b-1)^2]$, the certificate of
Theorem~\ref{thm:codec} at $\tau=1$.
\begin{enumerate}
\item[(i)] For every $\delta\in(0,1)$, conditionally on $\thetahat_1,\dots,\thetahat_K$, with
probability at least $1-\delta$ over the probe draw and with $\E_U$ over the last
round's dither,
\begin{equation}\label{eq:training_law}
\E_U\,R(\Phat)\;\le\;e^{\omega}\Bigl[\sqrt{R(\Pbar)}+\sqrt{\rho\bigl(\Gm(\delta)+\kappa\,(4C_K(b)+2\epsfit)\bigr)}\,\Bigr]^2,
\end{equation}
where $\Gm(\delta)=7C_\Phi\DTheta\sqrt{(V+\log(1/\delta))/m}$ bounds the probe
generalization gap, $\omega=C_\Phi(3\|\bar\theta-\theta^\star\|_2+s_m)$,
$\kappa=e^{2\Delta+4C_\Phi s_m}(1+e^{2\Delta})/2$, and
$s_m=e^{\Delta/2}\Delta/(2\sqrt{\lambda_m})+\sqrt{2\epsfit/\lambda_m}$ bounds how far
quantization and inexact fitting move the student (Lemma~\ref{lem:stability}).
\item[(ii)] With $K$ and $d$ fixed, if $\epsopt\to0$ as $n\to\infty$, then
$\E R(\Pbar)\le(1+o(1))\bigl(\sqrt{d/(2Kn)}+\sqrt{\epsopt}\bigr)^2$.
\end{enumerate}
\end{theorem}

Theorem~\ref{thm:training} prices each resource separately: with the prefactors
near one, nodes and samples enter the student's risk through $Kn$, probes through
$\sqrt{(V+\log(1/\delta))/m}$, and bits through $K^{-1}(2^b-1)^{-2}$. Its
bound is a prefactor multiplied by the square of the sum of two square roots: the risk of the unquantized
parameter average, and what the student loses by seeing that average only on $m$
probes, through a quantized teacher, with an inexact fit. Part (ii) follows from the classical expansion of the maximum-likelihood estimator~\citep{vandervaart2000}: averaging
the $K$ local estimates attains the risk $d/(2Kn)$ of pooling all $Kn$ samples.
In that regime (as for a fine grid, a close fit, and $\bar\theta$ near $\theta^\star$; Appendix~\ref{app:proof_training}), on the probe event of (i), with the prefactors $e^\omega$ and $\kappa$ set to one and $R(\Pbar)$ replaced by the bound (ii) on its expectation, and using $(x+y+z)^2\le3(x^2+y^2+z^2)$,
\begin{equation}\label{eq:additive}
\E\,R(\Phat)\;\lesssim\;3\Bigl(\sqrt{\frac{d}{2Kn}}+\sqrt{\epsopt}\Bigr)^2+3\rho\,\Gm(\delta)+12\rho\,C_K(b)+6\rho\,\epsfit .
\end{equation}
Inverting its terms gives allocation rules for $b$, $Kn$, and $m$
(Appendix~\ref{app:proof_training}).

The resource law is informative about bits only for grid steps
$\Delta\lesssim1$, because $\kappa\ge e^{2\Delta}(1+e^{2\Delta})/2$ grows like $e^{4\Delta}$: at
$L=40$ the 8-bit step $\Delta=0.31$ qualifies but the 4-bit step $\Delta=5.3$ does not,
so our 4-bit results are findings that the resource law neither predicts nor contradicts.

Real nodes rarely share one data distribution, so the next result completes the
answer to P1 by pricing heterogeneity. Here $\Pstar$ remains the target, such as the population the nodes
jointly serve, while each node's next-token law may drift; all nodes keep the
context law $\Pstar_X$.

\begin{corollary}[Heterogeneous nodes: a first-order local law]\label{cor:het}
In (A3), let node $i$ draw $Y\mid X\sim P_i(\cdot\mid X)$, the conditional law of $P_i$, instead of
$\Pstar(\cdot\mid X)$, with drift
$D_i:=\E_{X\sim\Pstar_X}\KL(\Pstar(\cdot\mid X)\|P_i(\cdot\mid X))$ and mean drift
$D:=K^{-1}\sum_iD_i$. Assume \emph{spread drift}: for a constant $c_0$,
$\sup_{x,v}|\log(P_i(v\mid x)/\Pstar(v\mid x))|\le c_0\sqrt{D_i}$, so that no node's
log-probabilities depart from the target's much more than its drift $D_i$ allows. Assume also interior population fits
$\theta_i\in\arg\min_{\Theta}\E\KL(P_i\|P_\theta)$, the parameters that maximize node
$i$'s expected log-likelihood; local MLEs that satisfy the analogous conditions of
Appendix~\ref{app:assumptions} around $\theta_i$, uniformly in the drifts;
and $\epsopt=0$. Then, with $K$ and $d$ fixed, as
$n\to\infty$ and $\max_iD_i\to0$,
\begin{equation*}
\E R(\Pbar)\;\le\;\frac{d}{2Kn}\;+\;D\;+\;O\Bigl(\frac1K\sum_iD_i^{3/2}\Bigr)\;+\;o\Bigl(\frac{d}{Kn}\Bigr),
\end{equation*}
and Theorem~\ref{thm:training}(i), whose proof uses neither (A3) nor the MLE clause
of (A1), holds unchanged, so~\eqref{eq:training_law} carries this bound to the student.
\end{corollary}

To first order, heterogeneity adds at most the mean drift, and adding nodes does not remove it:
drift in which all nodes shift the same way survives averaging, so a target risk
$R_0$ leaves a drift allowance $D\lesssim R_0-d/(2Kn)$. Drifts in independent directions
partly cancel, so for them the bound is loose.

\subsection{Coverage under compressed calibration and query messages (P2)}\label{sec:inference}

The answer to P2 depends on whether the threshold is recomputed: recomputing it
through the compressed pipeline keeps the full $1-\alpha$ guarantee (Proposition~\ref{prop:recal}),
and keeping a threshold calibrated on unquantized scores, for example after a
budget change, costs a bounded coverage loss (Theorem~\ref{thm:fixedref}). Both results concern a scoring model fixed before
calibration.
\begin{itemize}
\item[(B1)] \emph{Exchangeability.} Calibration and test examples are i.i.d.,
hence exchangeable, and are independent of the training data, the probes, and all
codec randomness.
\item[(B2)] \emph{Pre-shared candidates and bounded scores.} The candidate list is
shared in advance and contains the true answer; every node scores every
candidate; scores lie in a pre-shared interval of width $\Smax$.
\item[(B3)] \emph{Score codec.} Node $i$ encodes each score with the scalar codec of
Definition~\ref{def:codec} at $b_i$ bits, centered at the middle of the interval, with $L=\Smax/2$, so $\Delta_i=\Smax/(2^{b_i}-1)$ and every score satisfies the range
condition; dithers are independent across nodes and of all data. The error
$\xi_i=\tilde s_i(X,y)-s_i(X,y)$ of a decoded score is then uniform on $[-\Delta_i/2,\Delta_i/2]$, with variance $v_i=\Delta_i^2/12$, and
independent of the query and the calibration data.
\item[(B4)] \emph{Threshold reconstruction.} For each $\delta\in(0,1)$, the
reconstruction $\qhat$ of the split-conformal threshold $q_{1-\alpha}$ satisfies
$\Prob(|\qhat-q_{1-\alpha}|>\rcal(\Bcal,\delta))\le\delta$ for a radius $\rcal$ computed for
the chosen summary.
\item[(B5)] \emph{Density.} The unquantized node-averaged score
$S^\star=K^{-1}\sum_is_i(X,Y)$ of the true answer has a density bounded by $\fmax$.
\end{itemize}
For (B4), the conservative histogram has $q_{1-\alpha}\le\qhat\le q_{1-\alpha}+\Smax/\Nbin$
surely; Appendix~\ref{app:assumptions} covers other summaries. Condition (B5) caps
the probability near any threshold, so small shifts of the threshold or the score
change coverage little.

\begin{proposition}[Recalibration]\label{prop:recal}
Suppose the calibration scores are computed with the same quantized pipeline as
the test score, with fresh dither for each example, and the hub uses their
split-conformal threshold or any reconstruction that is never below it. Under
(B1)--(B3), $\Prob(Y\in\Cset(X))\ge1-\alpha$, where the probability is over calibration,
test, and codec randomness.
\end{proposition}

Recalibration therefore makes compression free in coverage and confines its cost
to set size, as \citet{zhu2024wfcp} observe when calibration
and test scores are quantized alike. Corollary~\ref{cor:setsize} bounds the set size when the
true answer tends to score better than a random candidate (assumption (B6) in
Appendix~\ref{app:assumptions}). A \emph{conservative} summary, one whose
threshold is never below the exact one, allows a finite $\Bcal$-bit message at no
cost in coverage.

\begin{theorem}[Coverage at a fixed threshold]\label{thm:fixedref}
Let $q_{1-\alpha}$ be the split-conformal threshold of the unquantized calibration
scores, $\qhat$ its reconstruction under (B4), and $\Cset(X)=\{y:\tilde s(X,y)\le\qhat\}$
the deployed set with quantized scores. Under (B1)--(B5), with $r=\rcal(\Bcal,\delta)$,
\begin{equation}\label{eq:fixedref}
\Prob\bigl(Y\in\Cset(X)\bigr)\;\ge\;1-\alpha-\delta-\fmax\,r-\Drag,
\qquad
\Drag=\fmax\sqrt{\frac1{K^2}\sum_{i=1}^Kv_i}.
\end{equation}
If $\qhat\ge q_{1-\alpha}$ almost surely, the terms $\delta$ and $\fmax r$ vanish.
\end{theorem}

The coverage a fixed threshold can lose shrinks quickly with the budget, because each extra bit more than halves the slack and independent node errors partially cancel: with a common step
$\Delta_s$, the slack in~\eqref{eq:fixedref} is $\Drag=\fmax\Delta_s/\sqrt{12K}$ at fixed
$\fmax$, and $b_i\ge\log_2\bigl(1+\fmax\Smax/(\eta_R\sqrt{12K})\bigr)$ bits per score keep
it within a tolerance $\eta_R$ (Appendix~\ref{app:proof_inference}). For smooth score
densities the coverage change itself is second order in the step, so each extra
bit divides it by about four (Remark~\ref{rem:second_order}). For the conservative histogram the
calibration terms vanish, and $\Bcal$ only inflates the threshold, by at most
$\Smax/\Nbin\propto\Bcal^{-1}$, since each bin count costs about $\log_2(\ncal+1)$
bits.

\subsection{From training error to coverage (P3)}\label{sec:propagation}

The answer to P3 is that a threshold reused after a model update loses coverage at most of order the
square root of the scoring model's training error. The result compares the student of
Section~\ref{sec:training}, used as the scorer of Section~\ref{sec:inference}, with
an \emph{oracle}, a hypothetical scorer that knows the true conditional
distribution of the answer. The oracle is never computed in deployment; it is only a common reference.

\begin{corollary}[Oracle-relative propagation]\label{cor:propagation}
Let $\Xp$ be the scoring context (the query with any retrieved passages); $P^\circ(\cdot\mid\Xp)$ the true conditional law of the answer; $s(\Xp,y)=-\log\Phat(y\mid\Xp)$ the deployed unquantized score, where $\Phat(\cdot\mid\Xp)$ is a distribution over $\Yset$, such as the model's probabilities renormalized over the $C$ candidates; and $s^\circ(\Xp,y)=-\log P^\circ(y\mid\Xp)$ the oracle score. Let
$\KLbar=\E\,\KL(P^\circ(\cdot\mid\Xp)\|\Phat(\cdot\mid\Xp))$, and assume that the
conditional density of $s^\circ(\Xp,Y)$ given $s(\Xp,Y)-s^\circ(\Xp,Y)$ is bounded by
the same constant $\fmax$ as in (B5), that is, whatever the student's error at a
test point, the oracle score does not concentrate near any threshold. Then for
every threshold $q$ chosen independently of the test point,
\begin{equation}\label{eq:propagation}
\bigl|\Prob(s(\Xp,Y)\le q)-\Prob(s^\circ(\Xp,Y)\le q)\bigr|\;\le\;\Dtrain:=\fmax\bigl(\KLbar+\sqrt{2\,\KLbar}\bigr).
\end{equation}
As a reference case, if the threshold that Theorem~\ref{thm:fixedref} reconstructs is the oracle's calibration threshold $q^\circ$ (the split-conformal threshold of the oracle scores of the calibration examples) and the student's score also satisfies (B5), then
$\Prob(Y\in\Cset(X))\ge1-\alpha-\delta-\fmax r-\Drag-\Dtrain$. For node-averaged scores,
the same holds with $s$, $s^\circ$, and $\KLbar$ averaged over nodes and with the density condition imposed on the averaged oracle score.
\end{corollary}

The resulting coverage bound scales with the square root of the training risk:
probes, pooled samples, and bits enter as
$m^{-1/4}$, $(Kn)^{-1/2}$, and $K^{-1/2}(2^b-1)^{-1}$. These rates follow because $\KLbar\le R(\Phat)$ when scoring contexts follow the training distribution and the answer is the next
token, and $\KLbar\le\rho_+R(\Phat)$ under a context shift with density ratio at
most $\rho_+$. The corollary needs a candidate-normalized score, the negative log-probability
of a distribution over $\Yset$, and so excludes the length-normalized scores of our FC-RAG experiments.

The coverage lost by reusing a threshold across a model update is bounded,
because the oracle cancels in the comparison. Let models $A$ and $B$, a
student before and after an update, each satisfy the density condition, with
slacks $\Dtrain^A$ and $\Dtrain^B$, and let $q_A$ be the split-conformal threshold of
$A$'s calibration scores. Applying~\eqref{eq:propagation} to each model at $q_A$
and combining the triangle inequality with the split-conformal guarantee for $A$ gives
$\Prob(s_B(\Xp,Y)\le q_A)\ge1-\alpha-\Dtrain^A-\Dtrain^B$. Recalibrating after an update restores the guarantee of Proposition~\ref{prop:recal} and confines the cost of training error to set size.

\section{Experiments}\label{sec:experiments}

\begin{figure}[t]
\centering
\includegraphics[width=\textwidth]{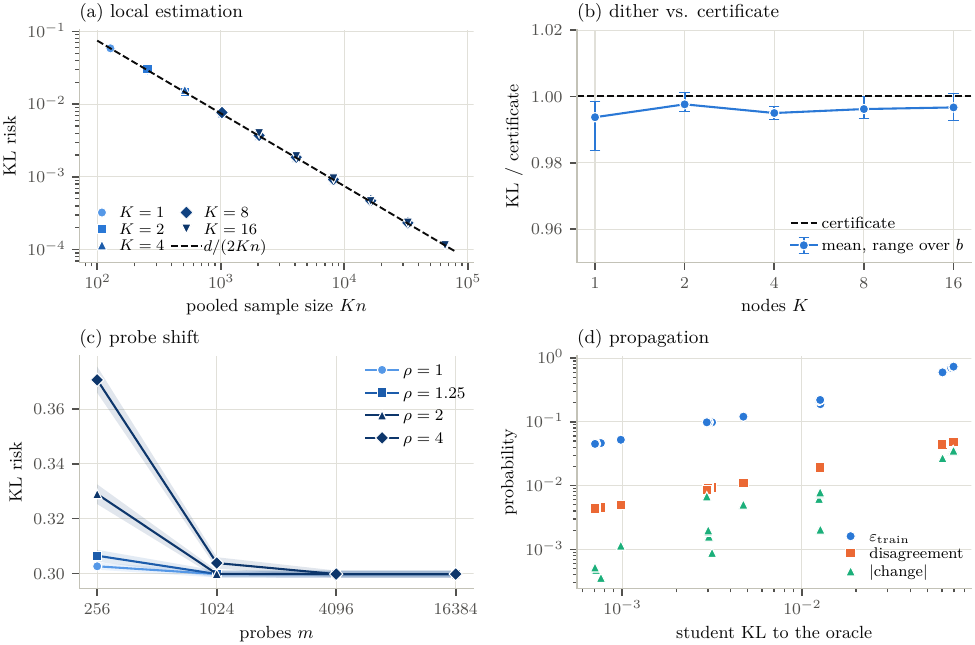}
\caption{Synthetic mechanism checks recover the predicted rates in nodes,
samples, and grid step, and show that the probe term and the propagation bound are
conservative (settings in Appendix~\ref{app:exp_mechanism}). (a) The KL risk of averaged local estimates
follows $d/(2Kn)$ with no fitted constant, every cell within $8.5\%$
(well-specified categorical model, without quantization or distillation). (b) The
codec certificate is nearly attained: measured teacher KL is $0.984$ to $1.001$
times the certificate of Theorem~\ref{thm:codec} (points: mean over the six bit
widths $b\in\{2,3,4,5,6,8\}$; bars: their range); the 95\% interval of every cell (not drawn)
reaches below one, and values below one mean the certificate holds. (c) Under probe shift
($\rho>1$; a larger, context-dependent model than in (a)),
more probes restore the student's KL risk faster than the bound requires (bands:
95\% bootstrap intervals over 20 seeds); the floor near $0.30$ is the risk of the averaged local estimates, which more probes
cannot reduce. (d) Over 480 FPLD students, the absolute coverage change at the oracle's threshold stays below $\Dtrain$ of Corollary~\ref{cor:propagation}, and the probability that student and oracle disagree on including the true answer, which the bound controls, falls about as $\mathrm{KL}^{1/2}$ (points: seed means for each of 24 training
cells; Table~\ref{tab:propagation} lists the 8- and 3-bit cells).}
\label{fig:components}
\end{figure}

The experiments check the theory's predictions where its assumptions hold
(Figure~\ref{fig:components}) and answer P4 on real language models (Tables~\ref{tab:matched}--\ref{tab:fcrag_main}; Figures~\ref{fig:training} and~\ref{fig:fcrag}), counting every uploaded payload bit.
Table~\ref{tab:evidence} (Appendix~\ref{app:experiments}) maps every study. Each
experimental condition is an \emph{arm}. Intervals are 95\%: Student-$t$
for an arm's mean, paired-$t$ for a contrast between arms that share seeds, and
bootstrap for synthetic and FC-RAG results; $\pm$ is a standard deviation (SD). A
\emph{contrast} is a paired difference in perplexity, the arm under study minus its
reference, so positive means worse, and a loss is \emph{not detectable} when the
contrast's interval contains zero.

\subsection{Mechanism checks}\label{sec:exp_components}

\emph{Purpose.} Each synthetic experiment isolates one term of the theory, with a
small softmax model that satisfies the assumptions (settings in
Appendix~\ref{app:exp_mechanism}).

\paragraph{Averaging local estimates attains the pooled rate $d/(2Kn)$.}
With $K$ nodes each estimating a categorical distribution on $16$ outcomes
($d=15$) from $n$ samples and the hub averaging their log-probabilities, a log-log
regression gives exponents $-0.993\,[-1.002,-0.983]$ in $K$ and
$-1.007\,[-1.016,-0.999]$ in $n$, and every $(K,n)$ cell lies within $8.5\%$ of the
unfitted $d/(2Kn)$ (Figure~\ref{fig:components}a).

\paragraph{The codec certificate is attained up to the factor $1-\|p\|_2^2$.}
When $K$ nodes encode one fixed logit row with independent dithers, the KL exponent in $K$ ranges from $-1.002$ to $-0.996$ across the six bit widths, and the exponent in $\Delta^2$ from $0.999$ to $1.000$ across the five node counts. The measured KL is $0.984$ to $1.001$ times the certificate for a teacher with $1-\|p\|_2^2\approx0.996$, and every 95\% interval extends below one (Figure~\ref{fig:components}b).

\paragraph{More probes offset a probe-distribution shift, faster than the bound requires.}
Since $\rho\,\Gm\propto\rho/\sqrt m$, the resource law predicts that quadrupling $m$ offsets
doubling $\rho$. At $m=256$, raising $\rho$ from 1 to 4 increases the KL risk by
$22.5\%$, from $0.3026$ to $0.3707$, and quadrupling $m$, a quarter of the sixteen-fold
increase the resource law asks for, returns it to $0.3038$ (Figure~\ref{fig:components}c;
Table~\ref{tab:probe_contrasts}). The probe-interface error is the divergence $\E_X\KL(P_{\mathrm{all}}(\cdot\mid X)\,\|\,\Phat_m(\cdot\mid X))$ between the aggregate $P_{\mathrm{all}}$ that the hub would form by probing every context and the student $\Phat_m$ trained on $m$ probes. It falls with slopes $-2.2$ to $-1.6$ in $m$, far faster
than the bound's $-1/2$.

\paragraph{Training error moves coverage by less than $\Dtrain$.}
For 480 students distilled from $K=4$ nodes with 250 to 16{,}000 samples each, at 3
to 8 bits or with full-precision (32-bit, FP32) logits, the absolute coverage change at the oracle's threshold never exceeds $\Dtrain$ of Corollary~\ref{cor:propagation}
(Figure~\ref{fig:components}d; Table~\ref{tab:propagation}). The probability that
the student and the oracle disagree on including the true answer, which the bound controls, falls
with log-log slope $0.527\,[0.522,0.532]$ in the student's KL, close to the square-root scaling of Section~\ref{sec:propagation}; the absolute change falls with slope
$0.956\,[0.904,1.023]$, because disagreements partly cancel. Reusing one student's
threshold for the others keeps mean coverage within $0.03$ of the $0.90$ target, above
the two-model bound of Section~\ref{sec:propagation} in every cell
(Table~\ref{tab:propagation}).

\subsection{Training a real language model}\label{sec:exp_training}

\emph{Purpose.} This study measures what compression costs FPLD on a real
language model (P1, P4), against parameter-exchange baselines.

\begin{table}[tp]
\caption{At equal uplink, gauge+dither beats the naive and CFD-style
codecs at every bit width in every seed, and at the final round its 4-bit arm matches FP32 logits at
one eighth of their uplink; exchanging parameters is more accurate when nodes
share an architecture. Matched
GPT-2-small/WikiText-2 suite ($K=4$, $T=5$, three paired seeds, $L=40$); the 8-bit
arm's three-seed gap to FP32 logits does not persist over ten seeds
(Section~\ref{sec:exp_training}). Nodes
upload parameter updates (Parameters), full models (Models), or
probe outputs (Outputs). Naive midpoint: raw logits rounded on $[-40,40]$ to bin
midpoints, without centering or dither; CFD-style $Q_b$: the quantizer of Communication-Efficient Federated Distillation (CFD;~\citealp{cfd2022}), which rounds each probability to a multiple of $1/(2^b-1)$; EF: error feedback; LoRA: low-rank adapters; lr: learning rate.
Perplexity (lower is better) is the three-seed
validation mean (pretrained: $59.15$); bold: best of the table's output codecs per width.
LoRA-FedAvg is shown at two learning rates, each at FP32 and 8 bits (equal
perplexity within $0.001$). Uplink sums the logical payload (message bits, without
headers) over nodes and rounds,
excluding the model broadcast (9.271 GiB for every arm that broadcasts the full model); a dash marks a value
that does not apply (the pooled-data reference trains on all shards together,
overfits under this schedule, and moves raw data, which the uplink
accounting does not charge). Appendix~\ref{app:exp_matched} has SDs, runtimes, and
contrasts.}
\label{tab:matched}
\centering\footnotesize
\setlength{\tabcolsep}{3.6pt}
\begin{tabular}{@{}lllcc@{}}
\toprule
Interface & Method & Precision & Perplexity & Uplink\\
\midrule
\multirow{4}{*}{Parameters} & FedAvg & FP32 & 30.76 & 9.271 GiB\\
 & quantized FedAvg + EF & 8 / 4 / 2 bits & 30.76 / 30.78 / 33.29 & 2.318 / 1.159 / 0.579 GiB\\
 & LoRA-FedAvg, lr $5\times10^{-5}$ & FP32 / 8 bits & 34.94 & 22.5 / 5.63 MiB\\
 & LoRA-FedAvg, lr $3\times10^{-4}$ & FP32 / 8 bits & 32.23 & 22.5 / 5.63 MiB\\
Models & FedDF & FP32 models & 35.26 & 9.271 GiB\\
\midrule
\multirow{4}{*}{Outputs} & FPLD, FP32 logits & FP32 & 43.57 & 7.669 GiB\\
 & FPLD, gauge+dither & 8 / 4 / 2 bits & \textbf{44.22} / \textbf{43.33} / \textbf{65.88} & 1.917 / 0.959 / 0.479 GiB\\
 & FPLD, naive midpoint & 8 / 4 / 2 bits & 58.95 / 61.48 / 132.2 & 1.917 / 0.959 / 0.479 GiB\\
 & CFD-style $Q_b$ & 8 / 4 / 2 bits & 48.86 / 74.12 / 125.4 & 1.917 / 0.959 / 0.479 GiB\\
\midrule
Pooled data & pooled-data reference & --- & 34.49 & ---\\
\bottomrule
\end{tabular}
\end{table}

\begin{figure}[tp]
\centering
\includegraphics[width=\textwidth]{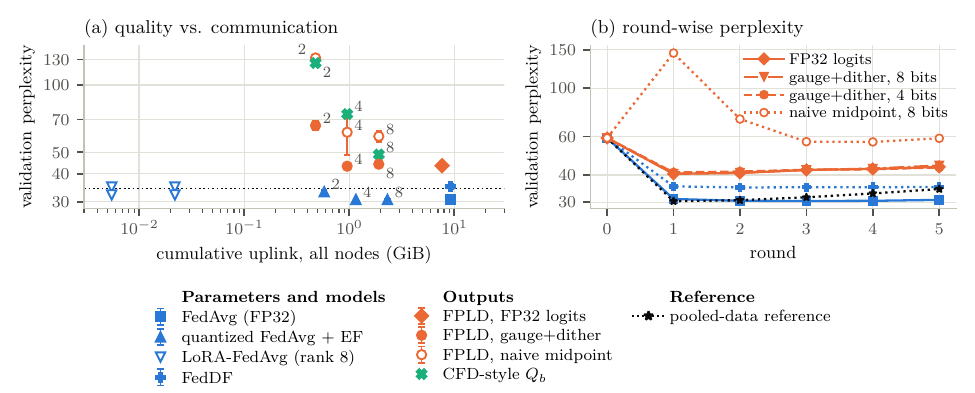}
\caption{Gauge+dither beats the naive and CFD-style codecs at every bit width in
every seed (arms as in Table~\ref{tab:matched}; three paired seeds;
logarithmic perplexity and uplink axes; lower is better). (a) Final perplexity against cumulative
uplink; labels give bits, bars are 95\% $t$-intervals where shown (the naive 2-bit
arm's is $[33,232]$), the dotted line is the pooled-data reference, and
LoRA-FedAvg appears at both learning rates. (b) Round-wise perplexity of FedAvg,
FedDF, the pooled reference, and four FPLD arms; in these seeds 4- and 8-bit
gauge+dither stay within $0.66$ of FP32 logits, whereas the naive 8-bit codec rises to $145$ after round 1; round 0 is the pretrained model.}
\label{fig:training}
\end{figure}

\emph{Setup.} Four nodes fine-tune GPT-2-small ($d=124{,}439{,}808$;
\citealp{radford2019gpt2}) on disjoint shards of WikiText-2~\citep{merity2017wikitext},
a corpus of Wikipedia articles, for $T=5$ rounds of one local epoch, with
$m=2{,}048$ probes from WikiText-2's test split, which is never used for evaluation;
perplexity is measured on the validation split. All arms share initial parameters,
data partitions, schedules, and random seeds, so they can be compared seed by seed
(the \emph{matched suite}; Appendix~\ref{app:exp_matched}). We report each arm at the
final round and at its \emph{best round}, the round of lowest validation perplexity.

The parameter-exchange arms are FedAvg, quantized FedAvg~\citep{fedpaq2020}, and
FedAvg of rank-8 low-rank adapters (LoRA;~\citealp{lora2022}). Quantized FedAvg uploads 8-, 4-, or
2-bit parameter updates with error feedback (EF;~\citealp{karimireddy2019ef}), which
carries each round's rounding error into the next update, and LoRA trains and
averages only small low-rank updates to the weight matrices. FedDF~\citep{feddf2020}
uploads full models and distills the nodes' averaged predictions into their
averaged model.

The output arms run FPLD with FP32 logits;
with gauge+dither at $L=40$, which clips at most $3.4\times10^{-6}$ of a run's coordinates
(Table~\ref{tab:codecs}); with a naive codec that rounds raw logits on $[-40,40]$ to bin midpoints
without centering or dither; or with the CFD-style codec, the simplex quantizer $Q_b$ of Communication-Efficient
Federated Distillation (CFD;~\citealp{cfd2022}), which rounds each probability to a
multiple of $1/(2^b-1)$. All output codecs send $b$ bits for each of the $V$
coordinates. A pooled-data reference trains on all shards together; it overfits
under this schedule, which is why it trails FedAvg. Two reduced suites, detailed in
Appendix~\ref{app:exp_scale}, repeat FPLD with FP32, 8-bit, and 4-bit logits on GPT-2-medium
($d=354{,}823{,}168$) and on the AG News corpus~\citep{zhang2015charcnn}.

\paragraph{Gauge+dither beats the naive and CFD-style codecs; parameter exchange is more accurate for a shared architecture.}
At identical payloads, gauge+dither reaches perplexity $44.22$, $43.33$, and $65.88$ at
8, 4, and 2 bits, against $58.95$, $61.48$, and $132.2$ for the naive codec and
$48.86$, $74.12$, and $125.4$ for the CFD-style codec, which leaves at most $2^b-1$ of
the $50{,}257$ token probabilities nonzero (Table~\ref{tab:matched};
Figure~\ref{fig:training}a). FedAvg and 4-bit
quantized FedAvg reach $30.8$ and FedDF $35.3$, while 4-bit FPLD trails 4-bit
quantized FedAvg by $12.54\,[11.35,13.73]$ at $17\%$ less uplink. The gap belongs to the
output interface, which transfers $2{,}048$ next-token distributions per round, not to
the codec: FPLD with FP32 logits trails FedAvg by $12.8$. FPLD gives up this accuracy in
return for an upload independent of model size and an interface defined across
architectures (Section~\ref{sec:exp_hetero}).

\begin{table}[t]
\caption{Centering is needed at every bit width, and dither on coarse grids. Codec
factorial on the grid of Definition~\ref{def:codec}, whose levels include $\pm L$ ($L=40$; matched
GPT-2-small/WikiText-2 suite): final validation perplexity, mean $\pm$ SD over three
paired seeds (FP32 logits: $43.57\pm0.44$). The first row reruns the gauge+dither arm of Table~\ref{tab:matched} (seed 0 on different hardware), and the last row differs from that table's naive codec only in
the grid (levels at $\pm L$ rather than bin midpoints); Appendix~\ref{app:exp_matched}
explains both and gives best-round values and paired contrasts.}
\label{tab:factorial}
\centering\small
\begin{tabular}{@{}ccccc@{}}
\toprule
Centering & Dither & 8 bits & 4 bits & 2 bits\\
\midrule
yes & yes & $43.98\pm0.35$ & $43.29\pm0.43$ & $65.88\pm1.15$\\
yes & no & $44.18\pm0.16$ & $47.46\pm0.49$ & $6{,}324\pm472$\\
no & yes & $54.82\pm3.66$ & $51.99\pm1.88$ & $78.11\pm3.61$\\
no & no & $55.13\pm3.66$ & $58.49\pm9.77$ & $118.25\pm3.74$\\
\bottomrule
\end{tabular}
\end{table}

\begin{table}[t]
\caption{Mixed GPT-2/GPT-Neo nodes, which coordinate-wise parameter averaging
cannot combine, train a common student with no detectable 8-bit codec cost: both
codec intervals contain zero and no codec-by-architecture interaction is
detected. Mixed-architecture factorial:
final perplexity of the common GPT-2-small student (eight paired seeds), with codec contrasts (last column), architecture
contrasts (last row), and their difference-in-differences (95\% paired-$t$
intervals; positive means worse). All cells use function-space synchronization,
so the homogeneous FP32 cell ($45.44$) differs from Table~\ref{tab:matched}'s
FP32-logit arm ($43.57$).}
\label{tab:arch}
\centering\small
\setlength{\tabcolsep}{5pt}
\begin{tabular}{@{}lccc@{}}
\toprule
Nodes & FP32 logits & 8-bit gauge+dither & 8-bit minus FP32\\
\midrule
four GPT-2-small & 45.44 & 45.52 & $+0.08\,[-0.97,1.13]$\\
two GPT-2-small + two GPT-Neo-125M & 47.15 & 47.69 & $+0.54\,[-0.49,1.57]$\\
mixed minus homogeneous & $+1.71\,[0.78,2.63]$ & $+2.17\,[1.21,3.12]$ & $+0.46\,[-1.11,2.02]$\\
\bottomrule
\end{tabular}
\end{table}

\paragraph{Centering is needed at every bit width, and dither on coarse grids.}
A $2\times2$ factorial removes each component of the codec on the grid of
Definition~\ref{def:codec} (Table~\ref{tab:factorial}, three paired seeds). Without
centering, $55$--$57\%$ of coordinates clip in round 1, and the dither-only codec is worse than gauge+dither by $10.8\,[1.6,20.1]$, $8.7\,[3.3,14.1]$, and $12.2\,[3.5,20.9]$
perplexity at 8, 4, and 2 bits. Without dither, the centered codec loses
$4.17\,[2.18,6.16]$ at 4 bits and collapses at 2 bits (perplexity $6{,}324$).
This behavior is consistent with quantization errors that fail to cancel across
nodes. At 8 bits, where the step $\Delta=0.31$ is fine, the two agree
($+0.20\,[-0.36,0.75]$).

\paragraph{Eight-bit FPLD shows no detectable loss in any round, and 4-bit FPLD shows none at the final round.}
Over ten paired seeds, the 8-bit contrast (grid step $\Delta=0.31$) with FP32 logits is $-0.23\,[-0.79,0.33]$ at the final round and $+0.07\,[-0.07,0.21]$ at the best round, with one-sided 95\% upper bounds of $0.22$ and $0.18$, and its interval contains zero in every round. The three-seed gap of Table~\ref{tab:matched} does not persist (Appendix~\ref{app:exp_matched}). A run at $L=64$ and $b=8$, with all $6.2\times10^9$ encoded coordinates in range and a codec error variance within $0.01\%$ of $\Delta^2/12$, reaches $43.95\,[42.27,45.62]$ (three seeds; FP32 logits: $43.57$).

Over the same ten seeds and at one eighth of the uplink ($0.959$ against $7.669$ GiB, summed over nodes and rounds), 4-bit FPLD reaches $43.47\,[43.09,43.86]$ against $43.97\,[43.49,44.45]$ for FP32 logits at the final round, a paired difference of $-0.49\,[-1.16,0.17]$ with one-sided upper bound $0.05$. It costs more early: $+0.86\,[0.64,1.08]$ after round 1 and $+0.75\,[0.60,0.90]$ at the best round. At $L=40$, the 4-bit step ($\Delta=5.3$) lies outside the regime of the resource law (Section~\ref{sec:training}).

\paragraph{On a larger model and a second corpus, neither 8- nor 4-bit FPLD shows a detectable loss.}
On GPT-2-medium with WikiText-2 and on GPT-2-small with AG News, a corpus of news
articles, every interval for 8- or 4-bit gauge+dither against FP32 logits contains zero at the final round and at the best round, and the one-sided 95\% upper bounds at the best
round are at most $0.80$ perplexity (Table~\ref{tab:scale}, three paired seeds). On
GPT-2-medium, 8- and 4-bit FPLD reach $40.99$ and $39.01$ against $39.95$ for FP32
logits; the 8-bit upload stays at GPT-2-small's $98.2$ MiB per node per round, below
the $169.2$ MiB of a dense 4-bit update of GPT-2-medium, since the FPLD payload does
not depend on $d$.

\paragraph{FPLD reaches its lowest perplexity in the first rounds under the fixed schedule; the later rise also occurs with FP32 logits.}
With FP32, 8-bit, and 4-bit logits, FPLD reaches its lowest perplexity in round 1
or 2 ($40.4$, $40.5$, and $41.1$) and then rises by $2.2$ to $3.7$
(Figure~\ref{fig:training}b); the pooled reference rises similarly, while FedAvg and
FedDF stay flat. At the best round, FPLD with FP32 logits trails FedAvg by $10.0$
rather than $12.8$ (Table~\ref{tab:matched_full}).

\begin{figure}[t]
\centering
\includegraphics[width=\textwidth]{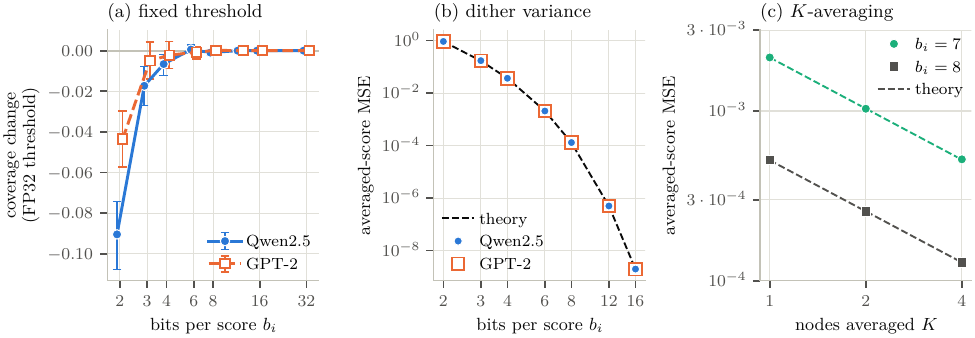}
\caption{With both scorers, quantizing query scores causes no detectable coverage change at the fixed FP32 threshold from 6 bits per score upward (every 95\% interval contains zero), and the error variance of the averaged score matches $\Delta_s^2/(12K)$ across bit widths and node counts.
All-node FC-RAG on four MMLU subjects (multiple-choice exam questions; $K=4$ nodes,
four candidates, $\alpha=0.1$), scored by pretrained models rather than FPLD students;
the fixed FP32 threshold is calibrated on unquantized scores and kept while query
scores are quantized. (a) Coverage change at the fixed FP32 threshold for the Qwen2.5 and
GPT-2 scorers, averaged for each question over the replicates in which it is a
test question, with 95\%
subject-stratified bootstrap intervals (points offset horizontally for
legibility). (b) Mean squared error (MSE) of the averaged score (equal to its error variance,
since the error has mean zero) against $\Delta_s^2/(12K)$ for both scorers. (c) Stored
FP32 GPT-2 scores, re-encoded with fresh dither and averaged over $K=1$, $2$, or $4$
nodes, follow $\Delta_s^2/(12K)$ (dashed; log-log slopes $-1.000$ at 8 bits, $-0.997$
at 7 bits).}
\label{fig:fcrag}
\end{figure}

\begin{table}[t]
\caption{At 8 bits per score, neither scorer shows a detectable coverage change at the fixed FP32 threshold, and the estimated local slack of
Theorem~\ref{thm:fixedref} exceeds the observed coverage loss, even where that loss is
measurably nonzero (Qwen2.5 at 4 bits); recalibration keeps coverage near $0.90$
and set size within $0.07$ of FP32 at every bit width.
Setup as in Figure~\ref{fig:fcrag}; recalibration recomputes the threshold from the
quantized scores. Top-1
accuracy (the fraction of questions whose lowest-score option is correct) and
set size are for FP32 scores. The local slack evaluates $\Drag$ with an estimate
of the score density near each split's threshold in place of $\fmax$, averaged over
the five splits. Recalibrated coverage and set size range over $b_i\in\{2,3,4,6,8,12,16,32\}$.
Tables~\ref{tab:allnode} and~\ref{tab:allnode_qwen} (Appendix~\ref{app:exp_fcrag}) give
every bit width.}
\label{tab:fcrag_main}
\centering\footnotesize
\setlength{\tabcolsep}{2.6pt}
\begin{tabular}{@{}lccccccc@{}}
\toprule
Scorer & top-1 & set size & \shortstack{cov.\ change\\8 bits} & \shortstack{cov.\ change\\4 bits} & \shortstack{local $\Drag$\\8 / 4 bits} & \shortstack{recal.\\coverage} & \shortstack{recal.\\set size}\\
\midrule
Qwen2.5-1.5B-Instruct & 0.447 & 3.34 & $-0.001\,[-0.002,0.000]$ & $-0.007\,[-0.012,-0.002]$ & 0.0011 / 0.057 & 0.904--0.907 & 3.34--3.41\\
GPT-2-small & 0.297 & 3.58 & $0.000\,[-0.001,0.002]$ & $-0.002\,[-0.009,0.005]$ & 0.0011 / 0.027 & 0.890--0.901 & 3.57--3.60\\
\bottomrule
\end{tabular}
\end{table}

\subsection{Heterogeneous architectures and data}\label{sec:exp_hetero}

\emph{Purpose.} This study tests whether FPLD trains a common student across nodes
with different architectures or different data, and at what codec cost (P4).

\emph{Setup.} Two GPT-2-small and two GPT-Neo-125M nodes~\citep{black2021gptneo}
differ in network structure and parameter shapes but share the 50{,}257-token
vocabulary and tokenizer, so their logits on a probe refer to the same tokens. They
use function-space synchronization, the 8-bit codec at $L=64$ (no coordinate
clipped in any round), and a GPT-2-small student. A $2\times2$ factorial over eight paired seeds
compares this system with four GPT-2-small nodes, running both with 8-bit and with FP32 logits (Appendix~\ref{app:exp_arch}).

\paragraph{Mixed-architecture nodes train a common student with no detectable codec cost.}
The mixed system lowers the student's perplexity from its pretrained $59.15$ to
$47.69$ with 8-bit logits and to $47.15$ with FP32 logits, and the 8-bit codec
tracks FP32 logits in every round (Figure~\ref{fig:arch_rounds}, Appendix~\ref{app:exp_arch}). With 8-bit
logits, the mixed system uses one quarter of the FP32-logit uplink and, with the synchronization downlink counted, cuts total traffic $1.7$- to $2.7$-fold (Appendix~\ref{app:exp_arch}). Both codec intervals in
Table~\ref{tab:arch} contain zero, and no codec-by-architecture interaction is
detected; the one-sided 95\% upper bounds on the codec's cost are $0.92$
(homogeneous) and $1.37$ (mixed) perplexity, and $1.71$ on the interaction.

\paragraph{No codec cost is detected under topic heterogeneity.}
A ten-seed factorial runs FP32 and 8-bit logits ($L=40$) on random and on
topic-clustered partitions, in which each node holds a distinct mixture of topics;
topic clustering shifts each node's contexts, a case Corollary~\ref{cor:het} does not cover.
The 8-bit contrast is $-0.23\,[-0.79,0.33]$ on the random partition (the ten-seed
contrast of Section~\ref{sec:exp_training}) and $+0.10\,[-0.41,0.60]$ on the topic
partition, with interaction $+0.33\,[-0.44,1.10]$; topic partitioning itself
costs $+0.91\,[0.23,1.58]$ perplexity without quantization
(Appendix~\ref{app:exp_partition}).

\subsection{Answering questions across nodes}\label{sec:exp_fcrag}

\emph{Purpose.} This study measures what compressing calibration and query
messages costs FC-RAG coverage (P2, P4).

\emph{Setup.} Four nodes each hold the demonstrations (questions with answers) of one high-school MMLU subject: statistics, physics, biology, or world history (118 demonstrations in all). MMLU~\citep{hendrycks2021mmlu} is a benchmark of four-option multiple-choice exam questions. For each of the 914 test questions of these four subjects, every
node retrieves its three most similar demonstrations, by MiniLM sentence
embeddings (vector representations of text;~\citealp{reimers2019sbert,wang2020minilm}). Each node then scores the four
options with a pretrained model, Qwen2.5-1.5B-Instruct~\citep{yang2024qwen25} or, as
a weak scorer, GPT-2-small, rather than an FPLD student, since the guarantees under
test hold for any scorer. Scores are encoded on $[0,20]$, which no score leaves; calibration and
test each use half of the questions, over five split-and-dither replicates
(Appendix~\ref{app:exp_fcrag}). Since each subject sits at one node, the study tests
compression and calibration, not the value of pooling evidence.

\paragraph{Eight bits per score reproduce full-precision behavior with both scorers.}
At $b_i=8$, a node sends $B_i=32$ bits per query, a quarter of FP32, and the
fixed-threshold coverage and set-size changes stay within $0.003$ of zero
(Figure~\ref{fig:fcrag}a; Table~\ref{tab:fcrag_main}; Tables~\ref{tab:allnode} and~\ref{tab:allnode_qwen}). Top-1 accuracy, the fraction of questions whose lowest-score option is correct, stays within $0.001$ of FP32. The error variance of the averaged score matches $\Delta_s^2/(12K)$ at every bit width (Figure~\ref{fig:fcrag}b), and re-encoding stored
FP32 GPT-2 scores over $K=1$, $2$, and $4$ nodes confirms on real scores that this
variance falls as $1/K$ (Figure~\ref{fig:fcrag}c).

\paragraph{The estimated local slack of Theorem~\ref{thm:fixedref} exceeds every observed coverage loss.}
The bound also holds with $\fmax$ replaced by the largest density of the
true-answer score within half a grid step of the realized threshold, averaged over
calibration draws (Appendix~\ref{app:proof_inference}).
With a kernel density estimate of that density (Appendix~\ref{app:exp_fcrag}), the
estimated $\Drag$ exceeds the observed coverage loss at every bit width for both scorers,
including the measurably nonzero losses for Qwen2.5 at 4 and 3 bits ($0.057$ and
$0.39$ against $0.007$ and $0.017$). Even the global form, which does not depend
on where the threshold falls, gives $\Drag=0.009$ for GPT-2-small and $0.011$ for
Qwen2.5 at 8 bits, with the maximum of the estimated density in place of $\fmax$. The changes at 2 to 4 bits are negative because symmetric
noise moves true-answer scores across the threshold in both directions and, here,
more of them sit just below it (Remark~\ref{rem:second_order}).

\paragraph{Recalibration keeps coverage near the 0.90 target, and all-node FC-RAG matches a pooled-central baseline.}
Recalibration gives coverage $0.904$ to $0.907$ with Qwen2.5 and $0.890$ to $0.901$ with
GPT-2-small at every bit width (means over five replicates), with set sizes within
$0.07$ of FP32 (Table~\ref{tab:fcrag_main}). Qwen2.5's sets of $3.34$
options cover $0.905$ against $0.836$ for random sets of the same size, so they carry
information. Pooled-central baselines, which retrieve from the union of the
corpora, reach coverage $0.893$ to $0.895$ with GPT-2-small and $0.900$ to $0.906$
with Qwen2.5, within the split-to-split variation of all-node scoring, and set
sizes within $0.05$ of it (Appendix~\ref{app:exp_fcrag}).

\paragraph{A 1-kbit calibration summary and single-node routing leave coverage intact.}
With all 457 calibration examples, the conservative histogram sends 1{,}017 bits per
node, small enough for one 340-byte Iridium Short Burst Data message, and changes
coverage by $+0.004\pm0.005$, never downward (Table~\ref{tab:bcal}). The FP32 score
list is 3.6 times as large as the summary. Routing each query to its top-ranked node by
one 8-bit relevance code per node halves the 8-bit payload, summed over the four
nodes, from 128 to 64 bits per query, and changes coverage by $+0.0002\pm0.0018$
against routed FP32 (Appendix~\ref{app:exp_fcrag}).

\section{Related work}\label{sec:related}

The methods reviewed below do not jointly bound the effects of compressing training
logits, calibration summaries, and per-query scores; Table~\ref{tab:position} (Appendix~\ref{app:related})
summarizes the comparison.

\paragraph{Federated distillation.}
Federated distillation~(\citealp{jeong2018fd}; see \citealp{mora2024kdfl}, for a survey) combines nodes through their predictions, but existing
methods give no bound on what the bits of those predictions cost in risk. FedMD~\citep{fedmd2019}
and DS-FL~\citep{itahara2023dsfl} exchange predictions on shared unlabeled data, as
FPLD does, and FedDF~\citep{feddf2020}, one of our baselines, distills uploaded
models; these methods target finite-class classification and give no
bit-dependent guarantee. For language models, FedMKT~\citep{fan2025fedmkt}
exchanges token-aligned outputs between a server model and smaller client models
of different architectures, and \citet{zhang2025commkd} sparsify logits to an
adaptive top-$k$ to fit wireless links, both without a bit-dependent guarantee. Outside federated learning, \citet{anshumann2025sparse} prove that caching only the top-$k$ teacher probabilities gives the student biased estimates of the teacher's distribution.

Closest to our training side is CFD~\citep{cfd2022},
which quantizes soft labels (predicted probability vectors) and codes their changes
between rounds, without a risk bound; our CFD-style
arm isolates its quantizer (Section~\ref{sec:exp_training}). Concurrently,
\citet{achanta2026tallytrain} send one class label per probe.

On the theory side, earlier results price samples or bits, but none prices both in
a distilled student. \citet{park2023ensemble} prove centralized rates for ensemble distillation in
federated kernel ridge regression without quantization, and one-shot averaging of
local estimators~\citep{zhangduchiwainwright2013,huanghuo2019} gives the
pooled-sample rate behind $d/(2Kn)$. Lower bounds for distributed estimation under
communication constraints~\citep{zhang2013infolb,barnes2020fisher,acharya2020inference2}
price bits in the risk of estimating a parameter or a distribution directly,
without probes or a distilled student. For the same probe-logit protocol,
\citet{dubeyhuo2026matching} derive an optimal per-node bit allocation and a lower
bound that matches the dependence of our bit term on $K$ and on the bits per
logit.

\paragraph{Compressed parameters and dithered quantization.}
Our codec applies subtractive dither, a classical quantization tool, to model
outputs rather than parameters, with guarantees in KL risk and coverage.
Compressed parameter exchange~\citep{qsgd2017,karimireddy2019ef,fedpaq2020} bounds
optimization error and needs coordinates that correspond across nodes, so it
cannot combine different architectures. Federated fine-tuning that rebuilds updates
from shared random seeds~\citep{qin2024fedkseed,li2025decomfl,ferret2025} sends
messages whose size does not depend on the model, but it still updates one shared
set of parameters, so it too needs a single architecture. FedKD~\citep{wu2022fedkd} likewise exchanges parameters, compressing the gradients of a small model that all nodes share and each co-trains, by distillation, with its own larger local model.

Dithered quantizers for distributed mean estimation and federated learning~\citep{suresh2017dme,uveqfed2021,suresh2022correlated} exploit the cancellation across clients that Theorem~\ref{thm:codec} turns into a KL certificate, but their guarantees are in mean squared error. Subtractive dither makes the quantization error uniform and independent of its input~\citep{schuchman1964dither,gray1993dithered}. \citet{adlertangpolyanskiy2021} quantize a single
large-alphabet distribution under KL divergence, without contexts, probes, or a
distilled student.

\paragraph{Conformal prediction across sites.}
Prior conformal guarantees cover single sites and compressed calibration
messages, but not compressed per-query scores from several nodes.
At a single site, conformal methods calibrate language models for multiple-choice
answering on MMLU~\citep{kumar2023conformalmcqa}, the benchmark of
Section~\ref{sec:exp_fcrag}; for free-form generation~\citep{quach2024clm,mohri2024factuality}; and for
retrieval-augmented generation~\citep{traq2024,kang2024crag}. Federated conformal
prediction calibrates from data held at many
sites~\citep{lu2023fedcp,humbert2023,plassier2023,kang2024robfcp}, and
\citet{lari2026prismfcp} couple it with a training stage that sends only part of
each model update.

None of the works closest to our calibration message charges a per-query message
from the nodes: \citet{zhu2024wfcp} keep coverage when devices quantize their calibration
scores and the server quantizes its test scores alike, and
\citet{wen2025dcp,gcfcp2026} summarize calibration scores by histograms or
mergeable coresets. \citet{zhu2025cdcrc} do charge per-query messages from several
sensors, but as binary label masks under a long-run risk criterion rather than
candidate scores under split-conformal coverage.

Corollary~\ref{cor:propagation} measures the scorer's error by its KL divergence, which the resource law bounds. The closest works instead bound coverage through
total-variation distances when exchangeability fails~\citep{barber2023beyond},
through the total-variation error of an estimated conditional
distribution~\citep{plassier2025pcp}, through a scorer's parameter drift under
fine-tuning~\citep{shihab2026continual}, and through calibration--test score shifts
after configuration changes such as weight quantization~\citep{zhu2026configshift}. Two federated-RAG works complement one-shot guarantees such as ours:
\citet{dubeyhuo2026refresh} monitor federated conformal RAG with tests that stay
valid across model and threshold refreshes, and, concurrently,
\citet{akter2026fedsrc} certify selective risk for federated RAG from privatized
score histograms.

\section{Discussion}\label{sec:discussion}

A designer can choose the bits of each message in a federated language-model
pipeline before deployment and bound what the choice costs. For training, two
results price the budget: a certificate free of the vocabulary size at a fixed
clip level, and a resource law. For answering, recalibration keeps the full $1-\alpha$ guarantee,
and keeping the full-precision threshold loses at most a slack shrinking as $K^{-1/2}$ at fixed $\fmax$.
Linking the two, a bound carries training error into coverage.

\paragraph{Practical takeaways.}
Center logits before rounding them, and add a shared dither when the grid is
coarse; neither costs bits, and the codec and aggregation take about $2\%$ of an FPLD run's time.
Start from 8 bits per logit and per score, which cost
nothing detectable in our experiments, and summarize calibration with a
conservative histogram. Recalibrate after any change to the model or codecs. Exchange parameters or
adapters when nodes share an architecture; outputs stay defined across architectures that share a tokenizer, with an upload size that does not grow with the model.

\paragraph{Limitations and outlook.}
The resource law is asymptotic in the local sample size under softmax-linear representability, and
its constants grow quickly with the grid step, so for fine-tuned neural networks it
guides the design rather than guaranteeing the outcome. The inference results cover finite answer spaces and
marginal coverage, and Corollary~\ref{cor:propagation} is tested on a synthetic model only. The training
experiments use models of 124M to 355M parameters, $K=4$ simulated nodes, in-domain
probes, and a fixed schedule under which FPLD does best in its first rounds. FPLD's
downlink, the hub's broadcast of the common model, is not compressed and grows
with the model; over slow satellite links it is the next cost to remove. The FC-RAG nodes hold MMLU
demonstrations rather than independent evidence corpora, and we model neither
adversarial nodes nor differential privacy. Larger models and more nodes are the
natural next tests.

\subsubsection*{Acknowledgments}
This work was supported by the National Science Foundation under award 2229876.

\bibliography{references}
\bibliographystyle{tmlr}

\clearpage
\appendix
\setcounter{topnumber}{3}

\section{Notation and assumptions}\label{app:assumptions}

\begin{table}[tb]
\caption{Notation.}
\label{tab:notation}
\centering\small
\setlength{\tabcolsep}{4pt}
\begin{tabular}{@{}ll@{}}
\toprule
Symbol & Meaning\\
\midrule
$\Vocab$, $V$ & vocabulary and its size\\
$\Ctx$, $\Pstar(\cdot\mid x)$, $\Pstar_X$ & contexts; target next-token law; target context marginal\\
$K$, $n$, $P_i$ & number of nodes; examples per node; node $i$'s data law\\
$\mathcal D_i$, $M_i$, $|M_i|$ & node $i$'s dataset; a message from node $i$; its length in bits\\
$\mathrm{enc}_i$, $U$ & node $i$'s encoder; randomness shared with the hub (the dither)\\
$Q$, $x_1,\dots,x_m$, $\rho$ & probe law; probe contexts; probe overlap $\operatorname*{ess\,sup} d\Pstar_X/dQ$\\
$\softmax$, $H(p)$ & $\softmax(\ell)_v=e^{\ell_v}/\sum_ue^{\ell_u}$; $H(p)=\operatorname{diag}(p)-pp^\top$\\
$\|u\|_{H(p)}^2$, $\osc(u)$ & $u^\top H(p)u=\operatorname{Var}_{Y\sim p}(u_Y)$; $\max_vu_v-\min_vu_v$\\
$\Phi$, $\Theta$, $\theta^\star$, $d$ & feature map $\Ctx\to\R^{V\times d}$; parameter set; true parameter; dimension\\
$C_\Phi$, $\DTheta$, $I$ & row bound $\|\Phi(x)_v\|_2\le C_\Phi$; radius of $\Theta$; Fisher information at $\theta^\star$\\
$I_d$, $I_V$, $\ones$ & identity matrices on $\R^d$ and $\R^V$; all-ones vector in $\R^V$\\
$p_\theta(x)$, $P_\theta$ & $\softmax(\Phi(x)\theta)$ and the predictor it defines\\
$R(P)$ & KL risk $\E_{X\sim\Pstar_X}\KL(\Pstar(\cdot\mid X)\|P(\cdot\mid X))$, Eq.~\eqref{eq:risk}\\
$\thetahat_i$, $\bar\theta$, $\Pbar$ & node estimates; their average; unquantized aggregate teacher $P_{\bar\theta}$\\
$\thetahat$, $\Phat$ & distilled student and its predictor\\
$T$, $E$ & rounds; local epochs\\
$\nu_{iv}$, $\bar\nu_v$ & codec error of node $i$ at coordinate $v$; its node average (Appendix~\ref{app:proof_codec})\\
$\Lm(\theta)$ & distillation objective, Eq.~\eqref{eq:distill}\\
$b$, $L$, $\Delta$, $\tau$ & bits per logit coordinate; clip level; grid step $2L/(2^b-1)$; temperature\\
$C_K(b)$ & codec certificate $L^2/[6K(2^b-1)^2]$ (Theorem~\ref{thm:codec}, $\tau=1$)\\
$\lambda_m$, $s_m$ & probe identifiability constant; student-stability radius (Lemma~\ref{lem:stability})\\
$\omega$, $\kappa$ & prefactors of Theorem~\ref{thm:training}\\
$\delta$ & failure probability (Theorems~\ref{thm:training} and~\ref{thm:fixedref})\\
$\Gm(\delta)$ & bound $7C_\Phi\DTheta\sqrt{(V+\log(1/\delta))/m}$ on the probe generalization gap\\
$\epsopt$, $\epsfit$ & local-optimization and distillation slacks\\
$\eta_B$, $\eta_n$, $\eta_m$, $\eta_R$ & tolerances in the allocation rules\\
$D_i$, $D$, $c_0$ & drift of node $i$; mean drift; spread constant (Corollary~\ref{cor:het})\\
$\alpha$, $\Yset$, $C$ & miscoverage level; candidate answers; their number\\
$s_i$, $\tilde s_i$, $\tilde s$ & node score; its decoded quantized version; hub average\\
$S^\star$, $\xi_i$, $\bar\xi$ & unquantized node-averaged true-label score; node score error; its average\\
$b_i$, $\Delta_i$, $v_i$, $B_i$ & bits per score; score step $\Smax/(2^{b_i}-1)$; score-error variance $\Delta_i^2/12$; bits per query $Cb_i$\\
$\Delta_s$, $K_r$ & common score step (Section~\ref{sec:inference}); number of routed nodes (Appendix~\ref{app:pseudocode})\\
$\ncal$, $q_{1-\alpha}$, $\qhat$ & calibration size; split-conformal threshold; its reconstruction\\
$\Bcal$, $\rcal$, $\Nbin$, $\Smax$ & calibration bits; reconstruction radius; number of histogram bins; width of the score range\\
$r'$, $\delta'$, $r''$ & upper reconstruction radius; its failure probability; recalibrated inflation (Corollary~\ref{cor:setsize})\\
$\fmax$, $\floc$ & density bound of the unquantized score; its local version (Lemma~\ref{lem:perturb})\\
$\Drag$, $\Dtrain$, $r$ & query-score and training coverage slacks; $r=\rcal(\Bcal,\delta)$\\
$\Xp$, $P^\circ$, $s^\circ$, $\rho_+$ & scoring context; its answer law; oracle score; context density ratio (Corollary~\ref{cor:propagation})\\
$q^\circ$, $\KLbar$, $\KLbar_i$ & oracle calibration threshold; expected KL of the scorer, and of node $i$'s scorer (Corollary~\ref{cor:propagation})\\
\bottomrule
\end{tabular}
\end{table}

\begin{table}[tbp]
\caption{Glossary of terms, with the section that defines each.}
\label{tab:glossary}
\centering\footnotesize
\setlength{\tabcolsep}{4pt}
\begin{tabular}{@{}>{\raggedright\arraybackslash}p{0.27\textwidth}>{\raggedright\arraybackslash}p{0.62\textwidth}l@{}}
\toprule
Term & Meaning & Section\\
\midrule
all-node / routed FC-RAG & every node scores every candidate / only the highest-ranked nodes, chosen by short relevance codes, score the candidates & \ref{sec:fcrag}\\
arm & one experimental condition, a method at one setting & \ref{sec:experiments}\\
best round & the round with the lowest validation perplexity in a seed's run, chosen on the reported split & \ref{sec:exp_training}\\
calibration & computing the conformal threshold from held-out labeled examples & \ref{sec:setup}\\
candidate-normalized score & the negative log-probability of a distribution over the candidates; length-normalized scores are not of this kind & \ref{sec:propagation}\\
certificate & a bound that depends only on the codec's settings and the number of nodes, here $C_K(b)$ of Theorem~\ref{thm:codec} & \ref{sec:training}\\
clip level $L$ & the largest centered magnitude the codec's grid represents & \ref{sec:codec}\\
conservative summary & a calibration summary whose threshold is never below the exact one, such as the conservative histogram & \ref{sec:inference}\\
coverage / set size & the probability that the prediction set contains the true answer / the number of candidates in the set & \ref{sec:setup}\\
delta & a node's parameter update, which federated averaging uploads & \ref{app:exp_codec}\\
dense & sending every coordinate ($d$ for a delta, $V$ for a logit row) at the stated bit width & \ref{sec:fpld}\\
distillation & fitting a student model to a teacher's predicted distributions rather than to labels & \ref{sec:intro}\\
dither & a random offset, generated from a seed shared by encoder and decoder, added before rounding and subtracted after & \ref{sec:codec}\\
exact / empirical run & no coordinate clipped / at most $3.4\times10^{-6}$ of coordinates clipped at $L=40$ & \ref{app:experiments}\\
fixed-threshold regime & a threshold computed from unquantized scores is kept while the deployed pipeline quantizes the scores & \ref{sec:inference}\\
FP32 & 32-bit floating point, used as the uncompressed reference precision & \ref{sec:exp_components}\\
function-space synchronization & nodes with different architectures stay synchronized through the student's distributions on the probes rather than through parameters & \ref{sec:fpld}\\
gauge / gauge+dither & centering a logit row, which leaves the softmax unchanged / the codec of Definition~\ref{def:codec} & \ref{sec:codec}\\
logit / logit row & an unnormalized log-probability / the $V$ logits a model assigns to the next token after one context & \ref{sec:intro}\\
LoRA & low-rank adapters: only the two small factors of a low-rank update to each adapted weight matrix are trained and sent & \ref{sec:exp_training}\\
matched suite & the GPT-2-small/WikiText-2 arms that share initial parameters, data partitions, schedules, and random seeds & \ref{sec:exp_training}\\
mechanism check / protocol test & an experiment that isolates one term / runs a full protocol on real models & \ref{sec:experiments}\\
midpoint / endpoint grid & $2^b$ bin midpoints of $[-L,L]$ (the naive codec) / $2^b$ levels that include $\pm L$ (Definition~\ref{def:codec}) & \ref{sec:exp_training}\\
node / hub & a participating site, a client in federated-learning terms / the aggregator, a server & \ref{sec:intro}\\
not detectable & a contrast whose 95\% interval contains zero & \ref{sec:experiments}\\
oracle & a hypothetical scorer that knows the true answer distribution; a reference point, never computed in deployment & \ref{sec:propagation}\\
perplexity & the exponential of the mean per-token log-loss on held-out text; lower is better & \ref{sec:intro}\\
probe & a shared unlabeled input context on which nodes report logits & \ref{sec:setup}\\
probe overlap $\rho$ & the worst-case ratio of the target's context density to the probe density & \ref{sec:setup}\\
probe-interface error & $\E_X\KL(P_{\mathrm{all}}\|\Phat_m)$ between the aggregate $P_{\mathrm{all}}$ the hub would form by probing every context and the student $\Phat_m$ trained on $m$ probes & \ref{sec:exp_components}\\
range condition & every centered coordinate lies in $[-L,L]$, so nothing is clipped & \ref{sec:codec}\\
recalibration regime & the threshold is recomputed from calibration scores produced by the deployed, quantized pipeline & \ref{sec:inference}\\
resource law & the bound of Theorem~\ref{thm:training} on the student's KL risk in terms of samples, probes, and bits & \ref{sec:training}\\
scorer & the model whose outputs define the conformal scores & \ref{sec:intro}\\
slack & an additive loss in a guarantee ($\Drag$, $\Dtrain$) or an allowed error ($\epsopt$, $\epsfit$) & \ref{sec:training}\\
spread drift & each node $i$'s log-probabilities depart from the target's by at most $c_0\sqrt{D_i}$ & \ref{sec:training}\\
teacher / student & the softmax of the hub's averaged logits / the common model the hub fits to it by distillation & \ref{sec:fpld}\\
top-1 accuracy & in FC-RAG, the fraction of questions whose lowest-score option is correct; in training tables, the fraction of validation tokens whose most probable predicted token is correct & \ref{sec:exp_fcrag}\\
uplink / downlink & traffic from the nodes to the hub / from the hub to the nodes & \ref{sec:setup}\\
\bottomrule
\end{tabular}
\end{table}

\paragraph{Training assumptions in full.}
Assumption (A1) asks for softmax-linear representability: the target
conditional is $\softmax(\Phi(x)\theta^\star)$ for bounded features and an
interior $\theta^\star$ in a convex compact $\Theta$ of radius $\DTheta$, together
with the following conditions. The Fisher information
$I=\E_{\Pstar_X}[\Phi(X)^\top H(p_{\theta^\star}(X))\Phi(X)]$ is positive
definite. Each node's MLE $\thetahat^{\mathrm{MLE}}$, computed over $\Theta$ from $n$ examples, satisfies $n\,\E[(\thetahat^{\mathrm{MLE}}-\theta^\star)
(\thetahat^{\mathrm{MLE}}-\theta^\star)^\top]\to I^{-1}$ and
$\sqrt n\,\|\E\thetahat^{\mathrm{MLE}}-\theta^\star\|\to0$. Finally,
$\Prob(\|\thetahat^{\mathrm{MLE}}-\theta^\star\|_2>t)=o(1/n)$ for every $t>0$. For
softmax-linear families with bounded features and an interior parameter, the MLE
is asymptotically normal~\citep[Ch.~5]{vandervaart2000}; since convergence in
distribution alone does not give these moment and tail conditions, we state them
as assumptions. Under (A1), logit averaging and
parameter averaging coincide, $K^{-1}\sum_i\Phi(x)\thetahat_i=\Phi(x)\bar\theta$,
which is what makes the unquantized aggregate a member of the model family.

For Corollary~\ref{cor:het}, each
local MLE satisfies the analogous conditions around node $i$'s population
maximizer $\theta_i$, uniformly in the drifts:
$n\,\E[(\thetahat_i^{\mathrm{MLE}}-\theta_i)(\thetahat_i^{\mathrm{MLE}}-\theta_i)^\top]\to\mathcal J_i^{-1}\mathcal S_i\mathcal J_i^{-1}$, where $\mathcal J_i$ and
$\mathcal S_i$ are the negative Hessian and the score covariance of node $i$'s expected
log-likelihood at $\theta_i$; $\|\E\thetahat_i^{\mathrm{MLE}}-\theta_i\|=O(1/n)$; and
$\Prob(\|\thetahat_i^{\mathrm{MLE}}-\theta_i\|_2>t)=o(1/n)$ for every $t>0$.

Assumption (A2) contains the probe conditions of Section~\ref{sec:setup}
together with identifiability of $\theta$ from the realized probes. The constant $\lambda_m$ is
the smallest curvature of the distillation objective~\eqref{eq:distill} over
$\Theta$; it is positive when the probe features are rich enough that no
direction of $\theta$ is invisible on every probe. A direction that shifts all
logits of a context by the same amount is always invisible, which is why (A1)
and (A2) need a parameterization with one reference logit fixed, as in the
tabular and output-layer examples of Section~\ref{sec:training}.

Assumption (A4) is a
deterministic slack in the Fisher norm that absorbs finite local training. Its
independence clause matters for multi-round FPLD: each round starts from a
student distilled on the probes, so for $T\ge2$ the final $\thetahat_i$ can depend
on the probe sample, and Lemma~\ref{lem:probe} would then be applied outside its
hypothesis. The clause is joint, because the proof conditions on the whole
vector $(\thetahat_1,\dots,\thetahat_K)$; independence of each $\thetahat_i$ separately would not
suffice. A single round ($T=1$), a fresh probe draw for the last distillation, or local estimates that do not depend on their starting point, such as unique exact
MLEs, remove this dependence; our GPT-2 runs reuse one probe set across rounds,
which the clause excludes.

Assumption (A5) needs independence across nodes only
for each fixed probe and coordinate, together with fresh dither in the last
round, so that the dither is independent of the logits it encodes. The
theorems use no independence across coordinates or probes, although our
implementation draws the dither independently along all these axes from
per-node, per-round streams. Assumption (A6) is a deterministic slack for the
student fit.

\paragraph{Inference assumptions in full.}
Conditions (B1)--(B5) are stated in Section~\ref{sec:inference}; the marginal statements of Theorem~\ref{thm:fixedref} and Corollary~\ref{cor:setsize} need (B5) globally, because $q_{1-\alpha}$ is random and can fall anywhere in the support of $S^\star$. Conditionally on the calibration data, only
the density within $\max_i\Delta_i/2$ of the realized threshold enters the quantization term; the local
form of Lemma~\ref{lem:perturb} makes this precise, and Appendix~\ref{app:exp_fcrag} evaluates the resulting local slack. Under covariate shift between calibration and test, a weighted conformal threshold~\citep{tibshirani2019} can replace the unweighted one.
For the set-size bound we also use:
\begin{itemize}
\item[(B6)] \emph{Score informativeness.} For $\tilde y$ drawn uniformly from
$\Yset$, independently of $X$ and of the dithers, and for every threshold $u$ in
the range of $\qhat$, $\Prob(\tilde s(X,\tilde y)\le u)\le\Prob(\tilde s(X,Y)\le u)$.
\end{itemize}

\paragraph{A radius for sample-based summaries.}
Suppose a summary's reconstruction satisfies the shifted-tail condition $\Prob(|\qhat-q_{1-\alpha}|>\phi(\Bcal)+t)\le2e^{-c_{\mathrm t}\ncal t^2}$ for all $t\ge0$, with $\phi(\Bcal)$ a reconstruction magnitude and $c_{\mathrm t}>0$ a constant. Then the summary satisfies (B4) with
$\rcal(\Bcal,\delta)=\phi(\Bcal)+\sqrt{\log(2/\delta)/(c_{\mathrm t}\ncal)}$. The conservative
histogram needs no such radius, since $q_{1-\alpha}\le\qhat\le q_{1-\alpha}+\Smax/\Nbin$
with probability one.

\section{Protocol pseudocode}\label{app:pseudocode}

\begin{algorithm}[tb]
\caption{FPLD: Federated Probe-Logit Distillation (shared architecture)}
\label{alg:fpld}
\begin{algorithmic}[1]
\Require rounds $T$; local epochs $E$; bits $b$; clip $L$; probes $x_1,\dots,x_m$; pretrained model $\Phat^{(0)}$
\For{$t=1,\dots,T$}
  \For{node $i=1,\dots,K$ \textbf{in parallel}}
    \State $\Phat_i^{(t)}\gets$ fine-tune $\Phat^{(t-1)}$ on local data for $E$ epochs
    \State for each probe $l$: $J_{i,l}\gets\textsc{Encode}(\text{logits of }\Phat_i^{(t)}\text{ at }x_l;\ b,L,\text{stream}(i,t,l))$ \Comment{Alg.~\ref{alg:codec}; $mVb$ bits}
    \State upload $\{J_{i,l}\}_{l=1}^m$
  \EndFor
  \State hub: $\tilde\ell_l\gets K^{-1}\sum_i\textsc{Decode}(J_{i,l};\ b,L,\text{stream}(i,t,l))$ for each probe $l$
  \State hub: $\Phat^{(t)}\gets$ approximate minimizer of $\frac1m\sum_l\KL(\softmax(\tilde\ell_l)\,\|\,P(\cdot\mid x_l))$; broadcast
\EndFor
\State \Return $\Phat=\Phat^{(T)}$
\end{algorithmic}
\end{algorithm}

\begin{algorithm}[tb]
\caption{Gauge-centered finite-grid subtractive-dither codec (Definition~\ref{def:codec})}
\label{alg:codec}
\begin{algorithmic}[1]
\Function{Encode}{$\ell\in\R^V$; $b$, $L$, stream}
  \State $\Delta\gets2L/(2^b-1)$;\quad $c\gets(\min_v\ell_v+\max_v\ell_v)/2$;\quad $z\gets\ell-c\ones$ \Comment{$c$ is not sent}
  \State \textbf{check} $\max_v|z_v|\le L$ \Comment{range condition; violations are clipped to $\pm L$ and counted}
  \State draw $U$ uniformly from $[-\Delta/2,\Delta/2)^V$ with the shared stream
  \State \Return $J_v=\lfloor(z_v+U_v+L)/\Delta+\tfrac12\rfloor\in\{0,\dots,2^b-1\}$ for all $v$
\EndFunction
\Function{Decode}{$J$; $b$, $L$, stream}
  \State regenerate the same $U$ from the stream;\quad \Return $\tilde z_v=-L+\Delta J_v-U_v$ \Comment{no re-clipping}
\EndFunction
\end{algorithmic}
\end{algorithm}

\begin{algorithm}[tb]
\caption{FC-RAG: all-node inference for one query}
\label{alg:fcrag}
\begin{algorithmic}[1]
\Require query $X$; pre-shared candidates $\Yset$ ($|\Yset|=C$); deployed model $\Phat$; node corpora; bits $b_i$; threshold $\qhat$
\For{node $i=1,\dots,K$ \textbf{in parallel}}
  \State retrieve passages $Z_i$ from the local corpus; for each $y\in\Yset$: $s_i(X,y)\gets-\log\Phat(y\mid X,Z_i)$, clipped to the pre-shared interval
  \State upload scalar-codec indices of $\{s_i(X,y)\}_{y\in\Yset}$ at $b_i$ bits each \Comment{$B_i=Cb_i$ bits}
\EndFor
\State hub: $\tilde s(X,y)\gets K^{-1}\sum_i\tilde s_i(X,y)$ for all $y\in\Yset$
\State \Return $\Cset(X)=\{y\in\Yset:\tilde s(X,y)\le\qhat\}$
\end{algorithmic}
\end{algorithm}

\paragraph{Calibration summary.}
The conservative histogram returns a threshold $\qhat$ that is never below
$q_{1-\alpha}$ and exceeds it by at most one bin width. Each node owns the labels
of a share of the calibration set. Once the averaged
scores of its examples have been formed with the per-query exchange, it holds
the averaged true-label score for each of its examples. In the fixed-threshold regime these are
unquantized scores; under recalibration they are the deployed quantized scores. The node counts its scores in $\Nbin$ pre-shared bins over the
score range and sends the counts within its $\Bcal$-bit budget. Decoded
quantized scores can leave the pre-shared interval by up to $\max_i\Delta_i/2$, so
for them the bins cover the interval widened by that amount on each side. The
hub adds the counts, locates the bin that contains the
$\lceil(\ncal+1)(1-\alpha)\rceil$-th smallest score, and returns that bin's upper edge as
$\qhat$; if the rank exceeds $\ncal$, it returns $+\infty$. The bin width is
$\Smax/\Nbin$ for unquantized scores.

\paragraph{Routed FC-RAG.}
Routing lets only the $K_r$ most relevant nodes score each query. Every node first sends one fixed-width relevance code, a quantized similarity between
the query and its best local passage; the $K_r$ highest-ranked nodes then score
the candidates, and the hub averages over them. At $K_r=1$ and 8 bits, a query costs $K\cdot8$ routing bits plus $C\cdot8$ score bits ($K=C=4$).

\paragraph{Nodes with different architectures.}
Such nodes stay synchronized through the common
student's soft targets (its next-token distributions) on the probes, so no parameters are ever copied between
architectures. Each node keeps its architecture and starts round 1 from its
native pretrained checkpoint. Before each later round, the hub
evaluates the current common student on the $m$ probes and broadcasts the
full-precision soft targets, an $m\times V$ matrix counted as downlink. Each
node distills toward those targets for one epoch, trains for one local epoch on
its data, and uploads encoded probe logits as in Algorithm~\ref{alg:fpld}. The hub averages the decoded probe logits and distills the student, a GPT-2-small model, for three epochs.

\FloatBarrier
\section{Proofs for Section~\ref{sec:training}\label{app:proofs_training}}

\subsection{Proof of Theorem~\ref{thm:codec} and a pathwise companion}\label{app:proof_codec}

\begin{proof}[Proof of Theorem~\ref{thm:codec}]
\emph{Gauge.} Let $\ell_i^{\min}=\min_v\ell_{iv}$ and $\ell_i^{\max}=\max_v\ell_{iv}$. For any $c\in\R$, $\ell_i^{\max}-\ell_i^{\min}\le|\ell_i^{\max}-c|+|\ell_i^{\min}-c|\le2\|\ell_i-c\ones\|_\infty$, with equality at $c_i=(\ell_i^{\min}+\ell_i^{\max})/2$, so the midrange minimizes
$\|\ell_i-c\ones\|_\infty$. Since $\bar z=K^{-1}\sum_i\ell_i-(K^{-1}\sum_ic_i)\ones$ and the
softmax is invariant to adding a constant to every coordinate,
$\softmax(\bar z/\tau)=\softmax(K^{-1}\sum_i\ell_i/\tau)$.

\emph{Valid code and uniform error.} Fix a node $i$ and a coordinate $v$, and write
$t=(z_{iv}+U_{iv}+L)/\Delta$. Under the range condition,
$z_{iv}+U_{iv}\in[-L-\Delta/2,\,L+\Delta/2)$, so $t\in[-\tfrac12,\,2^b-\tfrac12)$ and
$J_{iv}=\lfloor t+\tfrac12\rfloor\in\{0,\dots,2^b-1\}$: every code is valid and nothing is
clipped. The error is $\nu_{iv}:=\tilde z_{iv}-z_{iv}=\Delta(\lfloor t+\tfrac12\rfloor-t)$.
Conditionally on the logits, $U_{iv}$ is uniform on an interval of length $\Delta$, so
$t$ is uniform on an interval of length one and $\lfloor t+\tfrac12\rfloor-t$ is uniform
on $[-\tfrac12,\tfrac12]$ (ties have probability zero). Hence, conditionally on the
logits, $\nu_{iv}$ is uniform on $[-\Delta/2,\Delta/2]$, with $\E \nu_{iv}=0$ and
$\E \nu_{iv}^2=\Delta^2/12$. Independence across nodes gives the stated mean and
variance of $\bar\nu_v:=\tilde{\bar z}_v-\bar z_v=K^{-1}\sum_i\nu_{iv}$.

\emph{KL bound.} Write $\chi_v=\bar\nu_v/\tau$. From
$\tilde p_v=p_ve^{\chi_v}/\sum_up_ue^{\chi_u}$ we get the identity
\begin{equation}\label{eq:kl_lse}
\KL(p\,\|\,\tilde p)=\log\Bigl(\sum_vp_ve^{\chi_v}\Bigr)-\sum_vp_v\chi_v .
\end{equation}
Take conditional expectations. The linear term vanishes because $\E\chi_v=0$, and
Jensen's inequality for the concave logarithm gives
$\E\KL(p\|\tilde p)\le\log\sum_vp_v\E e^{\chi_v}$. By independence across nodes and
the moment-generating function of a uniform variable,
\[
\E e^{\chi_v}=\prod_{i=1}^K\E e^{\nu_{iv}/(K\tau)}=\Bigl(\frac{\sinh a}{a}\Bigr)^{K},\qquad a=\frac{\Delta}{2K\tau},
\]
the same for every $v$. Since $\sum_vp_v=1$, this gives the first inequality
in~\eqref{eq:certificate}. The product formula
$\sinh(a)/a=\prod_{k\ge1}(1+a^2/(k^2\pi^2))$ and $\log(1+x)\le x$ give
$\log(\sinh a/a)\le\sum_ka^2/(k^2\pi^2)=a^2/6$. Multiplying by $K$ gives
$\Delta^2/(24K\tau^2)$, and substituting $\Delta=2L/(2^b-1)$ gives the last expression.
No step uses independence across coordinates.
\end{proof}

\paragraph{Operational meaning.}
The certificate also bounds the expected increase in log-loss that transmission
causes under any reference law fixed with respect to the dither. Let $q$ be any such law on $\Vocab$, for instance the true next-token law. With $\chi$ as in the proof,
$\log(p_v/\tilde p_v)=-\chi_v+\log\sum_up_ue^{\chi_u}$, so by~\eqref{eq:kl_lse}
\[
\KL(q\|\tilde p)-\KL(q\|p)=\KL(p\|\tilde p)+(p-q)^\top\chi .
\]
Since $\E_U\chi_v=0$ for every $v$, $\E_U[\KL(q\|\tilde p)-\KL(q\|p)]=\E_U\KL(p\|\tilde p)$.

\paragraph{Why the dither must be uniform.}
Among dithers supported on $[-\Delta/2,\Delta/2)$, only the uniform one makes the rounding error uniform for every input. For any dither on this support, the error $\Delta(\lfloor t+\tfrac12\rfloor-t)$ is
minus the reduction of $z+U+L$ modulo $\Delta$ into $[-\Delta/2,\Delta/2)$. For each fixed $z$, the law of this reduction is a rotation of the law of $U$ on a circle of circumference $\Delta$, so
the error is uniform only if $U$ is uniform; a zero dither, for example, would
make the error a deterministic function of the input.

\paragraph{Comparison with stochastic rounding.} Unbiased stochastic rounding needs about $0.8$ more bits per logit for its KL bound to match gauge+dither's, because its error variance can reach three times $\Delta^2/12$. It rounds a coordinate up with probability equal to its fractional position $f\in[0,1)$ between grid points. Its error also has mean zero,
but it takes the two values $(1-f)\Delta$ and $-f\Delta$, with variance $f(1-f)\Delta^2$,
which depends on the input and reaches $\Delta^2/4$, three times the dithered
$\Delta^2/12$. In the proof of Theorem~\ref{thm:codec}, Hoeffding's lemma then replaces the
moment-generating function of the uniform error and gives
$\E\,\KL(p\|\tilde p)\le\Delta^2/(8K\tau^2)$, three times the certificate; to match gauge+dither, this bound needs about $\log_2\sqrt3\approx0.8$ more bits per logit.

\begin{corollary}[Pathwise codec bounds]\label{cor:hp_codec}
Under the conditions of Theorem~\ref{thm:codec}, $\KL(p\|\tilde p)\le\Delta^2/(8\tau^2)$
surely, which is $3K$ times the certificate. Moreover, for any $\delta_q\in(0,1)$,
conditionally on the logits and with probability at least $1-\delta_q$ over the
dither, $\KL(p\|\tilde p)\le\Delta^2\log(2V/\delta_q)/(12K\tau^2)$, which is the smaller
bound when $2\log(2V/\delta_q)<3K$.
\end{corollary}

\begin{proof}
By~\eqref{eq:kl_lse}, $\KL(p\|\tilde p)$ is the cumulant-generating function of the
centered variable $\chi_Y-\E_p\chi_Y$ under $Y\sim p$, evaluated at one. Every
$\bar\nu_v$ lies in $[-\Delta/2,\Delta/2]$, so $\chi_Y$ ranges over an interval of length
at most $\Delta/\tau$, and Hoeffding's lemma gives the sure bound $(\Delta/\tau)^2/8$.
For the second bound, a variable $W$ uniform on $[-\Delta/2,\Delta/2]$ has, for
every real $\vartheta$, $\E e^{\vartheta W}=\sinh(\vartheta\Delta/2)/(\vartheta\Delta/2)\le e^{\vartheta^2\Delta^2/24}$, so it is sub-Gaussian
with variance proxy $\Delta^2/12$, and $\bar\nu_v$ has proxy $\Delta^2/(12K)$. Hence
$\Prob(|\bar\nu_v|>t)\le2e^{-6Kt^2/\Delta^2}$, and a union bound over $v$ gives
$\|\bar\nu\|_\infty\le t:=\Delta\sqrt{\log(2V/\delta_q)/(6K)}$ with probability at
least $1-\delta_q$. On this event $\chi_Y\in[-t/\tau,t/\tau]$, and Hoeffding's lemma
gives $(2t/\tau)^2/8=t^2/(2\tau^2)$.
\end{proof}

\subsection{The softmax Hessian trace}\label{app:trace}

\begin{lemma}[Softmax Hessian trace]\label{lem:trace}
For every $p$ in the probability simplex on $V$ points,
$0\le\tr H(p)=1-\|p\|_2^2\le1-1/V$. The lower bound is attained exactly at
one-hot (point-mass) $p$ and the upper bound exactly at the uniform distribution.
\end{lemma}

\begin{proof}
$\tr H(p)=\sum_vp_v-\sum_vp_v^2=1-\|p\|_2^2$. Since $p_v\in[0,1]$,
$\|p\|_2^2\le\sum_vp_v=1$, with equality iff $p$ is one-hot. By Cauchy--Schwarz,
$1=(\sum_vp_v)^2\le V\|p\|_2^2$, with equality iff $p$ is uniform.
\end{proof}

The leading term mentioned after Theorem~\ref{thm:codec} is $\tr H(p)$ times
$\Delta^2/(24K\tau^2)$, so the lemma above bounds it between zero and $1-1/V$ times
the certificate. The function
$h\mapsto\KL(p\|\softmax(\bar z/\tau+h))$ vanishes with its gradient at $h=0$ and
has Hessian $H(p)$ there. With $h=\bar\nu/\tau$ and dithers independent across
coordinates, $\E hh^\top=\Delta^2/(12K\tau^2)\,I_V$, so the second-order term is
$\tfrac12\E\,h^\top H(p)h=\tr H(p)\,\Delta^2/(24K\tau^2)$.

\subsection{Softmax geometry and lemmas for Theorem~\ref{thm:training}}\label{app:geometry}

The proofs in this appendix rest on four elementary facts that compare KL
divergences and Hessian seminorms $\|u\|_{H(p)}^2=u^\top H(p)u$ of nearby softmaxes;
throughout, $\osc(u)=\max_vu_v-\min_vu_v$. Any full-support
distribution is the softmax of its log-probabilities, so these facts apply to arbitrary full-support distributions on $\Vocab$.

\begin{lemma}\label{lem:geometry}
Let $p_a=\softmax(a)$ for a logit vector $a\in\R^V$ (unrelated to the scalar $a$
of~\eqref{eq:certificate}). For all $a,a',u\in\R^V$:
\begin{enumerate}
\item[(a)] $\|u\|_{H(p_{a'})}^2\le e^{\osc(a'-a)}\,\|u\|_{H(p_a)}^2$;
\item[(b)] $\tfrac12e^{-\osc(a'-a)}\|a'-a\|_{H(p_a)}^2\le\KL(p_a\|p_{a'})\le\tfrac12e^{\osc(a'-a)}\|a'-a\|_{H(p_a)}^2$;
\item[(c)] $\KL(p_{a'}\|p_a)\le e^{3\osc(a'-a)}\KL(p_a\|p_{a'})$;
\item[(d)] $\|u\|_{H(p)}^2\le\osc(u)^2/4$ for every distribution $p$; and
$\osc(\Phi(x)w)\le2C_\Phi\|w\|_2$ under (A1).
\end{enumerate}
\end{lemma}

\begin{proof}
(a) Since $\sum_jp_{a,j}e^{(a'-a)_j}\ge e^{\min(a'-a)}$, we have
$p_{a',v}/p_{a,v}=e^{(a'-a)_v}/\sum_jp_{a,j}e^{(a'-a)_j}\le e^{\osc(a'-a)}$. Therefore
$\|u\|_{H(p_{a'})}^2=\operatorname{Var}_{p_{a'}}(u_Y)\le\E_{p_{a'}}(u_Y-\E_{p_a}u_Y)^2\le
e^{\osc(a'-a)}\operatorname{Var}_{p_a}(u_Y)$.
(b) With $f(t)=\KL(p_a\|p_{a+t(a'-a)})$ we have $f(0)=f'(0)=0$ and
$f''(t)=\|a'-a\|_{H(p_{a+t(a'-a)})}^2$, so $\KL(p_a\|p_{a'})=\int_0^1(1-t)f''(t)\,dt$. By
(a), applied in both directions with $\osc(t(a'-a))\le\osc(a'-a)$, $f''(t)$ lies within
factors $e^{\pm\osc(a'-a)}$ of $\|a'-a\|_{H(p_a)}^2$, and $\int_0^1(1-t)\,dt=\tfrac12$.
(c) By (b), then (a), then (b), with $\osc=\osc(a'-a)$:
$\KL(p_{a'}\|p_a)\le\tfrac12e^{\osc}\|a-a'\|_{H(p_{a'})}^2\le\tfrac12e^{2\osc}\|a-a'\|_{H(p_a)}^2\le e^{3\osc}\KL(p_a\|p_{a'})$.
(d) Popoviciu's inequality bounds a variance by a quarter of the squared range. Moreover,
$\osc(\Phi(x)w)\le2\max_v|\Phi(x)_vw|\le2C_\Phi\|w\|_2$.
\end{proof}

The next four lemmas supply the steps of the proof of Theorem~\ref{thm:training}:
a decomposition of the student's risk, a bound on its cross term, a transfer
from the $m$ probes to the probe law $Q$, and the stability of the distilled
student.

\begin{lemma}[Three-point identity]\label{lem:threepoint}
For full-support distributions $p,q,r$,
$\KL(p\|r)=\KL(p\|q)+\KL(q\|r)+(p-q)^\top(\log q-\log r)$. If $q=p_{\bar a}$ and
$r=p_{\hat a}$ are softmaxes, the last term equals $(p-q)^\top(\bar a-\hat a)$.
\end{lemma}

\begin{proof}
$\KL(p\|r)-\KL(p\|q)=p^\top(\log q-\log r)$ and $\KL(q\|r)=q^\top(\log q-\log r)$.
For softmaxes, $\log q-\log r=\bar a-\hat a$ plus a multiple of $\ones$, which
$(p-q)^\top$ annihilates.
\end{proof}

\begin{lemma}[Cross term]\label{lem:cross}
For $a^\star,\bar a,\hat a\in\R^V$, with $o_1=\osc(a^\star-\bar a)$ and $o_2=\osc(\bar a-\hat a)$,
\[
\bigl|(p_{a^\star}-p_{\bar a})^\top(\bar a-\hat a)\bigr|\le2e^{(3o_1+o_2)/2}\sqrt{\KL(p_{a^\star}\|p_{\bar a})\,\KL(p_{\bar a}\|p_{\hat a})}.
\]
\end{lemma}

\begin{proof}
Since the Jacobian of the softmax is $H$, $p_{a^\star}-p_{\bar a}=\int_0^1H_t(a^\star-\bar a)\,dt$
with $H_t=H(p_{\bar a+t(a^\star-\bar a)})$. By Cauchy--Schwarz in each seminorm,
$|(p_{a^\star}-p_{\bar a})^\top w|\le\sup_t\|a^\star-\bar a\|_{H_t}\|w\|_{H_t}$ for $w=\bar a-\hat a$.
By Lemma~\ref{lem:geometry}(a),(b):
$\|a^\star-\bar a\|_{H_t}^2\le e^{o_1}\|a^\star-\bar a\|_{H(p_{a^\star})}^2\le2e^{2o_1}\KL(p_{a^\star}\|p_{\bar a})$,
and $\|w\|_{H_t}^2\le e^{o_1}\|w\|_{H(p_{\bar a})}^2\le2e^{o_1+o_2}\KL(p_{\bar a}\|p_{\hat a})$.
\end{proof}

\begin{lemma}[Probe generalization]\label{lem:probe}
Fix $\bar\theta\in\Theta$ independent of the probes, and let
$g_\theta(x)=\KL(p_{\bar\theta}(x)\|p_\theta(x))$. Under (A1)--(A2), with probability at
least $1-\delta$ over the probes,
$\sup_{\theta\in\Theta}\{\E_Qg_\theta-\frac1m\sum_lg_\theta(x_l)\}\le \Gm(\delta)=7C_\Phi\DTheta\sqrt{(V+\log(1/\delta))/m}$.
\end{lemma}

\begin{proof}
Write $A(a)=\log\sum_ve^{a_v}$. Then
$g_\theta(x)=A(\Phi(x)\theta)-A(\Phi(x)\bar\theta)-p_{\bar\theta}(x)^\top\Phi(x)(\theta-\bar\theta)$ and
$0\le g_\theta\le\osc(\Phi(x)(\theta-\bar\theta))\le4C_\Phi\DTheta=:M$. By symmetrization and
the bounded-differences inequality, with probability at least $1-\delta$ the
supremum is at most $2\mathfrak R_m+M\sqrt{\log(1/\delta)/(2m)}$, where $\mathfrak R_m$ is the
Rademacher complexity of $\{g_\theta\}$~\citep{bartlettmendelson2002}. Terms that
do not depend on $\theta$ contribute nothing to $\mathfrak R_m$. The gradient of $A$ is
a probability vector, so $A$ is 1-Lipschitz in $\ell_2$, and the
vector-contraction inequality of \citet{maurer2016vector} gives, for i.i.d.\ Rademacher signs $\varepsilon_l$ and $\varepsilon_{lv}$ (uniform on
$\{-1,1\}$),
$\E\sup_\theta\sum_l\varepsilon_lA(\Phi(x_l)\theta)\le\sqrt2\,\E\sup_\theta\langle\theta,\sum_{l,v}\varepsilon_{lv}\Phi(x_l)_v\rangle\le\sqrt2\DTheta C_\Phi\sqrt{mV}$.
The linear term has features $\psi(x)=\Phi(x)^\top p_{\bar\theta}(x)$ with
$\|\psi\|_2\le C_\Phi$, and contributes at most $\DTheta C_\Phi\sqrt m$. Hence
$\mathfrak R_m\le\DTheta C_\Phi(\sqrt{2V}+1)/\sqrt m$. Finally,
$2\DTheta C_\Phi(\sqrt{2V}+1)+2\sqrt2C_\Phi\DTheta\sqrt{\log(1/\delta)}\le7C_\Phi\DTheta\sqrt{V+\log(1/\delta)}$.
\end{proof}

\begin{lemma}[Student stability]\label{lem:stability}
Under (A1), (A2), (A4), (A5), and (A6), and for $\bar\theta=K^{-1}\sum_i\thetahat_i$,
$\|\thetahat-\bar\theta\|_2\le s_m:=e^{\Delta/2}\Delta/(2\sqrt{\lambda_m})+\sqrt{2\epsfit/\lambda_m}$.
\end{lemma}

\begin{proof}
The objective~\eqref{eq:distill} is $\Lm(\theta)=\frac1m\sum_l[A(\Phi(x_l)\theta)-\tilde p_l^\top\Phi(x_l)\theta]$
plus a constant, where $\tilde p_l=\softmax(\tilde\ell(x_l))$ is the teacher at probe $x_l$. The Hessian of $\Lm$ is $\frac1m\sum_l\Phi(x_l)^\top H(p_\theta(x_l))\Phi(x_l)\succeq\lambda_mI_d$
on $\Theta$, so $\Lm$ is $\lambda_m$-strongly convex on the convex set $\Theta$, and
its minimizer $\theta_m$ over $\Theta$ satisfies
$\Lm(\theta)-\Lm(\theta_m)\ge\frac{\lambda_m}2\|\theta-\theta_m\|_2^2$ for all $\theta\in\Theta$. Applied at
$\thetahat$ with (A6), this gives $\|\thetahat-\theta_m\|_2\le\sqrt{2\epsfit/\lambda_m}$.
Because $\Lm\ge0$, applying the same bound at $\bar\theta$, which lies in $\Theta$ by
(A4) and convexity, gives
$\frac{\lambda_m}2\|\bar\theta-\theta_m\|_2^2\le \Lm(\bar\theta)=\frac1m\sum_l\KL(\tilde p_l\|p_{\bar\theta}(x_l))$.
Under (A5), the decoded teacher logits differ from $\Phi(x_l)\bar\theta$ by a constant
plus the averaged codec error, whose coordinates lie in $[-\Delta/2,\Delta/2]$, so the logit difference has oscillation at most $\Delta$. By
Lemma~\ref{lem:geometry}(b),(d), each term is at most $\tfrac12e^{\Delta}\Delta^2/4$.
Hence $\|\bar\theta-\theta_m\|_2\le e^{\Delta/2}\Delta/(2\sqrt{\lambda_m})$, and the triangle
inequality completes the proof.
\end{proof}

\subsection{Proof of Theorem~\ref{thm:training}}\label{app:proof_training}

\begin{proof}
\emph{Part (i).} Condition on $\thetahat_1,\dots,\thetahat_K$. By (A4) and convexity, $\bar\theta\in\Theta$, and the probes remain i.i.d.\ from $Q$, since (A4) makes them independent
of the whole vector $(\thetahat_1,\dots,\thetahat_K)$; by (A5), the last round's dither
is independent of the probes and the $\thetahat_i$ jointly, so Theorem~\ref{thm:codec}
applies conditionally on the probe logits. Work on the event of
Lemma~\ref{lem:probe} for this $\bar\theta$, which has probability at least $1-\delta$,
and on which $\lambda_m>0$ by (A2). Write $a^\star=\Phi(x)\theta^\star$,
$\bar a=\Phi(x)\bar\theta$, and $\hat a=\Phi(x)\thetahat$, write $\bar a_l$ and $\hat a_l$ for $\bar a$ and
$\hat a$ at $x=x_l$, and let $\tilde a_l$ be the decoded teacher logits at probe $x_l$.

\emph{Step 1: decomposition.} Lemma~\ref{lem:threepoint} with $p=\Pstar(\cdot\mid x)$,
$q=\Pbar(\cdot\mid x)$, $r=\Phat(\cdot\mid x)$ gives, pointwise in $x$,
$\KL(\Pstar\|\Phat)=T_1+T_2+T_3$ with $T_1=\KL(\Pstar\|\Pbar)$, $T_2=\KL(\Pbar\|\Phat)$, and
$T_3=(p_{a^\star}-p_{\bar a})^\top(\bar a-\hat a)$.

\emph{Step 2: cross term.} By Lemma~\ref{lem:geometry}(d) and
Lemma~\ref{lem:stability}, $\osc(a^\star-\bar a)\le2C_\Phi\|\theta^\star-\bar\theta\|_2$ and
$\osc(\bar a-\hat a)\le2C_\Phi s_m$. Lemma~\ref{lem:cross} then gives
$|T_3|\le2e^{\omega}\sqrt{T_1T_2}$ with $\omega=C_\Phi(3\|\theta^\star-\bar\theta\|_2+s_m)$. Hence,
pointwise, $\KL(\Pstar\|\Phat)\le e^{\omega}(\sqrt{T_1}+\sqrt{T_2})^2$. Integrating over
$X\sim\Pstar_X$ and applying Cauchy--Schwarz to the cross product gives
\begin{equation}\label{eq:step2}
R(\Phat)\le e^{\omega}\Bigl(\sqrt{R(\Pbar)}+\sqrt{R_2}\Bigr)^2,\qquad R_2:=\E_{X\sim\Pstar_X}\KL(\Pbar(\cdot\mid X)\|\Phat(\cdot\mid X)).
\end{equation}

\emph{Step 3: from the target to the probes.} Because $\KL\ge0$ and
$d\Pstar_X\le\rho\,dQ$, we have $R_2\le\rho\,\E_Q\KL(\Pbar\|\Phat)$. On the probe event,
Lemma~\ref{lem:probe} at $\theta=\thetahat\in\Theta$ gives
$R_2\le\rho\bigl[\frac1m\sum_l\KL(p_{\bar a_l}\|p_{\hat a_l})+\Gm(\delta)\bigr]$.

\emph{Step 4: the teacher on the probes.} At each probe,
$\osc(\tilde a_l-\bar a_l)\le\Delta$, $\osc(\hat a_l-\bar a_l)\le2C_\Phi s_m$, and
$\osc(\hat a_l-\tilde a_l)\le\Delta+2C_\Phi s_m$. By Lemma~\ref{lem:geometry}(b), the
triangle inequality for $\|\cdot\|_{H(p_{\bar a_l})}$, and Lemma~\ref{lem:geometry}(a),(b),
\[
\KL(p_{\bar a_l}\|p_{\hat a_l})\le2e^{2\Delta+4C_\Phi s_m}\bigl[\KL(p_{\bar a_l}\|p_{\tilde a_l})+\KL(p_{\tilde a_l}\|p_{\hat a_l})\bigr].
\]
Averaging over probes, (A6) with the feasible point $\bar\theta$ gives
$\frac1m\sum_l\KL(p_{\tilde a_l}\|p_{\hat a_l})\le\frac1m\sum_l\KL(p_{\tilde a_l}\|p_{\bar a_l})+\epsfit$.

\emph{Step 5: codec.} Under (A1), $\bar a_l$ is the logit row of the unquantized aggregate teacher, so Theorem~\ref{thm:codec} with $\tau=1$ gives
$\E_U\KL(p_{\bar a_l}\|p_{\tilde a_l})\le C_K(b)$ for every probe. For the reverse
direction, Lemma~\ref{lem:geometry}(b),(a) with $\osc(\tilde a_l-\bar a_l)\le\Delta$ give
$\KL(p_{\tilde a_l}\|p_{\bar a_l})\le\tfrac12e^{2\Delta}\|\tilde a_l-\bar a_l\|_{H(p_{\bar a_l})}^2$. The
seminorm ignores the constant offset between $\tilde a_l$ and $\bar a_l$, so
$\|\tilde a_l-\bar a_l\|_{H(p_{\bar a_l})}^2=\operatorname{Var}_{p_{\bar a_l}}(\bar\nu_Y)\le\sum_vp_{\bar a_l,v}\bar\nu_v^2$,
where $\bar\nu$ is the averaged codec error at that probe; the right-hand side has expectation
$\Delta^2/(12K)=2C_K(b)$. Hence $\E_U\KL(p_{\tilde a_l}\|p_{\bar a_l})\le e^{2\Delta}C_K(b)$.
Combining Steps 3--5, with $s_m$ deterministic given the probes and
$\kappa=e^{2\Delta+4C_\Phi s_m}(1+e^{2\Delta})/2$,
\[
\E_UR_2\le\rho\Bigl[\Gm(\delta)+2e^{2\Delta+4C_\Phi s_m}\bigl((1+e^{2\Delta})C_K(b)+\epsfit\bigr)\Bigr]
\le\rho\bigl[\Gm(\delta)+\kappa\,(4C_K(b)+2\epsfit)\bigr],
\]
using $1+e^{2\Delta}\ge2$ for the fit term.

\emph{Step 6: combine.} Take $\E_U$ in~\eqref{eq:step2}. The factor $e^\omega$ and
$R(\Pbar)$ do not depend on the dither. Since $\E(c+\sqrt W)^2\le(c+\sqrt{\E W})^2$ for
constants $c\ge0$ and random $W\ge0$,
\[
\E_UR(\Phat)\le e^{\omega}\Bigl(\sqrt{R(\Pbar)}+\sqrt{\rho[\Gm(\delta)+\kappa(4C_K(b)+2\epsfit)]}\Bigr)^2,
\]
which is~\eqref{eq:training_law}.

\emph{Part (ii).} By Lemma~\ref{lem:geometry}(b),(d),
$R(\Pbar)\le\tfrac12e^{2C_\Phi\|\bar\theta-\theta^\star\|_2}\|\bar\theta-\theta^\star\|_I^2$, using
$\E_{\Pstar_X}\|\Phi(X)u\|_{H(p_{\theta^\star}(X))}^2=\|u\|_I^2$. Write
$\bar\theta-\theta^\star=(\bar\theta^{\mathrm{MLE}}-\theta^\star)+\bar\zeta$, where
$\bar\theta^{\mathrm{MLE}}=K^{-1}\sum_i\thetahat_i^{\mathrm{MLE}}$ and $\bar\zeta=\bar\theta-\bar\theta^{\mathrm{MLE}}$
satisfies $\|\bar\zeta\|_I\le\sqrt{2\epsopt}$ by (A4). Independence across nodes (A3) and (A1) give
$\E\|\bar\theta^{\mathrm{MLE}}-\theta^\star\|_I^2=\frac1K\E\|\thetahat_1^{\mathrm{MLE}}-\theta^\star\|_I^2+(1-\frac1K)\|\E\thetahat_1^{\mathrm{MLE}}-\theta^\star\|_I^2=(1+o(1))\,d/(Kn)$,
since $\tr(I\,I^{-1})=d$, with $K$ and $d$ fixed. Minkowski's inequality then gives
$(\E\|\bar\theta-\theta^\star\|_I^2)^{1/2}\le(1+o(1))\sqrt{d/(Kn)}+\sqrt{2\epsopt}$. To remove the
exponential factor, split on $\{\|\bar\theta-\theta^\star\|_2\le t'\}$. On this event the
factor is at most $e^{2C_\Phi t'}$. Its complement has probability at most
$K\cdot o(1/n)$ by the tail condition in (A1), once $\epsopt$ is small enough that $\|\bar\zeta\|_2<t'/2$. On $\Theta$, $R(\Pbar)\le4C_\Phi\DTheta$. Letting $t'\to0$ slowly
gives (ii).
\end{proof}

\paragraph{Size of the prefactors.}
The prefactor $\kappa$ grows like $e^{4\Delta}$, which is why Section~\ref{sec:training}
calls the resource law informative about bits only for $\Delta\lesssim1$. The factor $e^\omega$ tends to one as $\bar\theta\to\theta^\star$ and $s_m\to0$, and
$\kappa$ tends to one as $\Delta\to0$ and $s_m\to0$; both limits need $\lambda_m$ bounded away
from zero. At a fixed step, $\kappa\ge e^{2\Delta}(1+e^{2\Delta})/2$ regardless of the
other quantities. This lower bound is $1.16$ at $\Delta=0.05$, $2.69$ at the 8-bit
step $\Delta=0.314$ for $L=40$, $5.09$ at the 8-bit step $\Delta=0.502$ for $L=64$, and about
$9\times10^8$ at the 4-bit step $\Delta=5.33$ for $L=40$.

\paragraph{Allocation rules.}
Inverting the terms of~\eqref{eq:additive} gives, for tolerances $\eta_B$, $\eta_n$,
and $\eta_m$: $C_K(b)\le\eta_B$ once $b\ge\lceil\log_2(1+L/\sqrt{6K\eta_B})\rceil$;
$d/(2Kn)\le\eta_n$ once $Kn\ge d/(2\eta_n)$; and $\rho\Gm(\delta)\le\eta_m$ once
$m\ge49C_\Phi^2\DTheta^2\rho^2(V+\log(1/\delta))/\eta_m^2$.

\subsection{Proof of Corollary~\ref{cor:het}}\label{app:proof_het}

\begin{proof}
\emph{Population drift.} Let $p_i=P_i(\cdot\mid X)$ and let $\theta_i$ be node $i$'s
population maximizer. It is interior by assumption, so it matches moments,
$\E_{\Pstar_X}[\Phi(X)^\top(p_i(X)-p_{\theta_i}(X))]=0$. For an exponential family this
gives the Pythagorean identity
$\E\KL(p_i\|p_\theta)=\E\KL(p_i\|p_{\theta_i})+\E\KL(p_{\theta_i}\|p_\theta)$ for every $\theta$,
because the difference of the two sides is
$\E[(p_i-p_{\theta_i})^\top\Phi](\theta_i-\theta)=0$. At $\theta=\theta^\star$ it gives
$\E\KL(p_{\theta_i}\|\Pstar)\le\E\KL(p_i\|\Pstar)$. By Lemma~\ref{lem:geometry}(c), applied
pointwise to the log-probabilities with $\osc(\log p_i-\log\Pstar)\le2c_0\sqrt{D_i}$,
$\E\KL(p_i\|\Pstar)\le e^{6c_0\sqrt{D_i}}D_i$. By Lemma~\ref{lem:geometry}(b),(a) with
$\osc(\Phi(\theta_i-\theta^\star))\le2C_\Phi\|\theta_i-\theta^\star\|_2$,
$\E\KL(p_{\theta_i}\|\Pstar)\ge\tfrac12e^{-4C_\Phi\|\theta_i-\theta^\star\|_2}\|\theta_i-\theta^\star\|_I^2$. Hence
\[
\tfrac12\|\theta_i-\theta^\star\|_I^2\le e^{4C_\Phi\|\theta_i-\theta^\star\|_2+6c_0\sqrt{D_i}}\,D_i .
\]
Since $\|\theta_i-\theta^\star\|_2\le2\DTheta$, this forces $\|\theta_i-\theta^\star\|_2=O(\sqrt{D_i})$,
and then $\tfrac12\|\theta_i-\theta^\star\|_I^2\le D_i+O(D_i^{3/2})$. Convexity of the squared
norm gives $\tfrac12\|\bar\theta_{\mathrm{pop}}-\theta^\star\|_I^2\le D+O(K^{-1}\sum_iD_i^{3/2})$
for $\bar\theta_{\mathrm{pop}}=K^{-1}\sum_i\theta_i$.

\emph{Sampling.} With $\epsopt=0$, $\bar\theta$ averages the local MLEs, and by
independence across nodes,
$\E\|\bar\theta-\theta^\star\|_I^2=\|\E\bar\theta-\theta^\star\|_I^2+\tr(I\operatorname{Cov}\bar\theta)$. By
assumption, $\E\thetahat_i=\theta_i+\beta_i$ with $\|\beta_i\|=O(1/n)$, and
$n\operatorname{Cov}\thetahat_i\to\mathcal J_i^{-1}\mathcal S_i\mathcal J_i^{-1}$, where $\mathcal J_i$ and $\mathcal S_i$ are the
negative Hessian and the score covariance of node $i$'s expected log-likelihood at
$\theta_i$. Both differ from $I$ by $O(\sqrt{D_i})$, because
$\|\theta_i-\theta^\star\|_2=O(\sqrt{D_i})$ and $P_i(\cdot\mid x)/\Pstar(\cdot\mid x)=1+O(\sqrt{D_i})$, so
$\tr(I\operatorname{Cov}\bar\theta)=(1+O(\sqrt{\max_iD_i})+o(1))\,d/(Kn)$. The bias gives
$\|\E\bar\theta-\theta^\star\|_I^2=\|\bar\theta_{\mathrm{pop}}-\theta^\star\|_I^2+O(\sqrt D/n)+O(1/n^2)$, and
$\sqrt D/n=o(d/(Kn))$ as $D\to0$ with $K$ and $d$ fixed.

\emph{Local expansion.} As in the proof of Theorem~\ref{thm:training}(ii),
$R(\Pbar)\le\tfrac12e^{2C_\Phi\|\bar\theta-\theta^\star\|_2}\|\bar\theta-\theta^\star\|_I^2$. Write
$u=\bar\theta-\bar\theta_{\mathrm{pop}}$ and fix $t'\in(0,1]$. The complement of $\{\|u\|_2\le t'\}$ has
probability $K\cdot o(1/n)$ by the tail condition around each $\theta_i$, and
$R(\Pbar)\le4C_\Phi\DTheta$ there. On the event,
$e^{2C_\Phi\|\bar\theta-\theta^\star\|_2}\le e^{2C_\Phi\|\bar\theta_{\mathrm{pop}}-\theta^\star\|_2}(1+2C_\Phi e^{2C_\Phi}\|u\|_2)$, and $e^{2C_\Phi\|\bar\theta_{\mathrm{pop}}-\theta^\star\|_2}-1=O(K^{-1}\sum_j\sqrt{D_j})$.
Expand $\|\bar\theta-\theta^\star\|_I^2=\|\bar\theta_{\mathrm{pop}}-\theta^\star\|_I^2+2\langle\bar\theta_{\mathrm{pop}}-\theta^\star,u\rangle_I+\|u\|_I^2$
and take expectations, using the sampling facts above and $\E\|u\|_2=O(n^{-1/2})$.
Besides the drift term $\tfrac12\|\bar\theta_{\mathrm{pop}}-\theta^\star\|_I^2\le D+O(K^{-1}\sum_iD_i^{3/2})$ and the sampling term $(1+o(1))\,d/(2Kn)$, four terms remain. (1) The drift excess
$\tfrac12(e^{2C_\Phi\|\bar\theta_{\mathrm{pop}}-\theta^\star\|_2}-1)\|\bar\theta_{\mathrm{pop}}-\theta^\star\|_I^2$ is
$O((K^{-1}\sum_j\sqrt{D_j})(K^{-1}\sum_iD_i))=O(K^{-1}\sum_iD_i^{3/2})$ by Chebyshev's sum
inequality. (2) The cross term is $O(\sqrt D/n)=o(d/(Kn))$. (3) The terms carrying
$\|u\|_2$ are of order $D\,\E\|u\|_2+\sqrt D\,\E\|u\|_2^2+t'\E\|u\|_I^2=O(Dn^{-1/2})+o(1/n)$
once $t'\to0$ slowly, and Young's inequality
$Dn^{-1/2}\le\tfrac23D^{3/2}+\tfrac13n^{-3/2}$ with $D^{3/2}\le K^{-1}\sum_iD_i^{3/2}$ (convexity) puts
them in $O(K^{-1}\sum_iD_i^{3/2})+o(d/(Kn))$. (4) The excess factor on the sampling
term contributes $o(d/(Kn))$. Collecting terms gives the stated bound.

The proof of Theorem~\ref{thm:training}(i) uses only the representability and feature bound
of (A1); (A2); the membership and independence clauses of (A4); (A5); and
(A6). None of these involves the homogeneity in (A3) or the MLE clause of (A1),
which drift would break, so part (i) applies unchanged. The $O(\cdot)$ constants
depend on $C_\Phi$, $\DTheta$, $I$, $c_0$, and the uniform rates of the MLE conditions.
\end{proof}

\section{Proofs for Sections~\ref{sec:inference} and~\ref{sec:propagation}}\label{app:proof_inference}

\begin{proof}[Proof of Proposition~\ref{prop:recal}]
Given the training outcome, the $\ncal+1$ quantized true-label scores
$S_j=\tilde s(X_j,Y_j)$ are functions of i.i.d.\ pairs $(X_j,Y_j)$ and of independent,
identically distributed dithers (B1, B3), so they are exchangeable. By (B2) the
true answer is among the candidates, so $Y\in\Cset(X)$ exactly when the test score
does not exceed the threshold. The standard split-conformal argument~\citep{vovk2005,lei2018distribution} gives $\Prob(S_{\ncal+1}\le q_{1-\alpha})\ge1-\alpha$ for
the split-conformal threshold $q_{1-\alpha}$ of $S_1,\dots,S_{\ncal}$. Any
$\qhat\ge q_{1-\alpha}$ gives a superset of the set at threshold $q_{1-\alpha}$, and averaging over the training
outcome preserves the bound. A histogram reconstruction is never below $q_{1-\alpha}$ provided its bins cover the widened range of the decoded scores
(Appendix~\ref{app:pseudocode}); widening the upper end suffices, since scores below
the first bin are counted in it without lowering the returned edge.
\end{proof}

\begin{lemma}[Score perturbation]\label{lem:perturb}
Let $\mathcal G$ contain the calibration data and the trained scorer, let $q$ be
$\mathcal G$-measurable, and let $\bar\xi=K^{-1}\sum_i\xi_i$ be independent of the pair formed by
the unquantized test score $S^\star$ and $\mathcal G$. Assume the $\xi_i$ are independent, with
mean zero and variances $v_i$, and that $|\bar\xi|\le w$ surely. If, given $\mathcal G$, $S^\star$
has a density bounded by a $\mathcal G$-measurable $\floc$ on $[q-w,q+w]$, then
$|\Prob(S^\star+\bar\xi\le q\mid\mathcal G)-\Prob(S^\star\le q\mid\mathcal G)|\le\floc\E|\bar\xi|\le\floc\sqrt{K^{-2}\sum_iv_i}$.
Under (B3), $w=\max_i\Delta_i/2$, and (B5) allows $\floc=\fmax$ for every $q$.
\end{lemma}

\begin{proof}
Condition on $\bar\xi=u$, so that $|u|\le w$. Moving the threshold from $q$ to $q-u$
changes the conditional distribution function of $S^\star$ by at most $\floc|u|$, since
only the density on $[q-w,q+w]$ enters. Integrating over $u$ gives the first
inequality. Cauchy--Schwarz gives the second, since $\E\bar\xi^2=K^{-2}\sum_iv_i$ by
independence and zero means.
\end{proof}

\begin{proof}[Proof of Theorem~\ref{thm:fixedref}]
Let $S^\star$ be the unquantized node-averaged score of the test pair, so that the
deployed score is $S^\star+\bar\xi$ by (B2)--(B3). Let $\mathcal E=\{\qhat\ge q_{1-\alpha}-r\}$,
which has $\Prob(\mathcal E^c)\le\delta$ by (B4). On $\mathcal E$,
$\{S^\star+\bar\xi\le q_{1-\alpha}-r\}\subseteq\{Y\in\Cset(X)\}$, so
$\Prob(Y\in\Cset(X))\ge\Prob(S^\star+\bar\xi\le q_{1-\alpha}-r)-\delta$. The threshold
$q_{1-\alpha}-r$ is measurable with respect to the calibration data, and by (B1)
and (B3) the error $\bar\xi$ is independent of $S^\star$ and of that data.
Lemma~\ref{lem:perturb} then gives
$\Prob(S^\star+\bar\xi\le q_{1-\alpha}-r)\ge\Prob(S^\star\le q_{1-\alpha}-r)-\Drag$. By (B5),
$\Prob(S^\star\le q_{1-\alpha}-r)\ge\Prob(S^\star\le q_{1-\alpha})-\fmax r$. By (B1),
$\Prob(S^\star\le q_{1-\alpha})\ge1-\alpha$. If $\qhat\ge q_{1-\alpha}$ almost surely,
repeat the argument with $r=0$ and $\mathcal E^c=\emptyset$.
\end{proof}

\paragraph{Local form.}
Theorem~\ref{thm:fixedref} also holds with $\fmax$ in $\Drag$ replaced by the
expected local density $\E\floc$ near the realized threshold. Conditionally on
the calibration data, the local form of Lemma~\ref{lem:perturb}
bounds the quantization term by $\floc\sqrt{K^{-2}\sum_iv_i}$, where $\floc$ is the
largest density of $S^\star$ within $\max_i\Delta_i/2$ of the realized threshold
$q_{1-\alpha}-r$; averaging over the calibration data gives the claim.
Appendix~\ref{app:exp_fcrag} estimates this local slack for each calibration
split. With a common step, $\Drag\le\eta_R$ whenever
$b_i\ge\lceil\log_2(1+\fmax\Smax/(\eta_R\sqrt{12K}))\rceil$.

\begin{remark}[Second-order behavior]\label{rem:second_order}
For smooth densities the coverage change is second order in the step, whereas
$\Drag$ is first order because it bounds the change for every bounded density. Let $f$ and $F$ be the density and distribution function of $S^\star$, and fix a
threshold $q$. Since $\bar\xi$ is symmetric and independent of $S^\star$, the change is
exactly
\[
\Prob(S^\star+\bar\xi\le q)-\Prob(S^\star\le q)=\E[F(q-\bar\xi)]-F(q)
=\tfrac12\,\E\int_0^{|\bar\xi|}\bigl[f(q+t)-f(q-t)\bigr]\,dt,
\]
so it is negative whenever the density is lower just above $q$ than just below it,
within $\max_i\Delta_i/2$ of $q$, as in an upper tail. If $f$ has a bounded second
derivative on $[q-\max_i\Delta_i/2,\,q+\max_i\Delta_i/2]$, a Taylor expansion gives
\[
\E[F(q-\bar\xi)]-F(q)=\tfrac12f'(q)\,\frac1{K^2}\sum_iv_i+O\bigl(\E|\bar\xi|^3\bigr),
\qquad \E|\bar\xi|^3=O(\max_i\Delta_i^3K^{-3/2}).
\]
The first-order term vanishes, so for smooth densities and steps that are small relative to the scale on which $f$ varies, the coverage change scales with $\Delta_i^2$
rather than $\Delta_i$, and halving the step divides it by about four, not two.
\end{remark}

\begin{corollary}[Expected set size]\label{cor:setsize}
Assume (B1)--(B6), and that any randomness of the summary is independent of the
test point and its dithers, as for the conservative histogram, which is deterministic. In the setting of
Theorem~\ref{thm:fixedref},
$\E|\Cset(X)|\le C\bigl(1-\alpha+\frac1{\ncal+1}+\delta'+\fmax r'+\Drag\bigr)$, where $r'$ is any
radius with $\Prob(\qhat>q_{1-\alpha}+r')\le\delta'$; for the conservative histogram,
$r'=\Smax/\Nbin$ with $\delta'=0$. Under recalibration,
$\E|\Cset(X)|\le C(1-\alpha+\frac1{\ncal+1}+\fmax r'')$, where $r''$ bounds how far $\qhat$
exceeds the split-conformal threshold of the quantized calibration scores: $r''=0$
for that threshold and $r''=(\Smax+\max_i\Delta_i)/\Nbin$ for the conservative
histogram over the widened range.
\end{corollary}

\begin{proof}
Summing acceptance indicators over the candidates and applying (B6)
conditionally on $\qhat$, which is independent of the test point and its dithers,
gives
$\E|\Cset(X)|=\sum_{y\in\Yset}\Prob(\tilde s(X,y)\le\qhat)=C\,\Prob(\tilde s(X,\tilde y)\le\qhat)\le C\,\Prob(Y\in\Cset(X))$.
The upper bound on coverage mirrors the proof of Theorem~\ref{thm:fixedref}:
$\Prob(Y\in\Cset(X))\le\Prob(S^\star+\bar\xi\le q_{1-\alpha}+r')+\delta'\le\Prob(S^\star\le q_{1-\alpha})+\fmax r'+\Drag+\delta'$,
and $\Prob(S^\star\le q_{1-\alpha})\le1-\alpha+1/(\ncal+1)$ because $S^\star$ has a density by
(B5), so calibration and test scores have no ties. Under recalibration, the
quantized scores $S^\star+\bar\xi$ are continuous and exchangeable, so the same
no-ties bound applies to them at their split-conformal threshold. Their density,
the density of $S^\star$ averaged over the law of $\bar\xi$, is also at most $\fmax$, so
raising the threshold by at most $r''$ adds at most $\fmax r''$.
\end{proof}

\begin{proof}[Proof of Corollary~\ref{cor:propagation}]
Write $s=s(\Xp,Y)$, $s^\circ=s^\circ(\Xp,Y)$, and $\Gamma=s-s^\circ$. First condition on
the calibration data, so that $q$ is fixed. Since
$|\mathbb{I}\{s\le q\}-\mathbb{I}\{s^\circ\le q\}|$ is at most the indicator that $s^\circ$ lies between
$q-\Gamma$ and $q$, the bounded conditional density of $s^\circ$ given $\Gamma$ gives
$|\Prob(s\le q)-\Prob(s^\circ\le q)|\le\fmax\E|\Gamma|$. Let $\gamma=\log(P^\circ/\Phat)$ at the
realized answer, so that $\Gamma=\gamma$. Since $P^\circ$ is the conditional law of $Y$ given
$\Xp$, $\E_{Y\mid\Xp}|\gamma|=\E_{Y\mid\Xp}\gamma+2\E_{Y\mid\Xp}(-\gamma)_+$. The first term is
$\KL(P^\circ\|\Phat)$. For the second, $\log(1+x)\le x$ gives
$\E_{Y\mid\Xp}(-\gamma)_+=\sum_{y:\Phat>P^\circ}P^\circ\log(\Phat/P^\circ)\le\sum_{y:\Phat>P^\circ}(\Phat-P^\circ)=d_{\mathrm{TV}}(P^\circ,\Phat)$,
the total-variation distance, where the last equality uses that $\Phat(\cdot\mid\Xp)$ sums to one over $\Yset$.
Pinsker's inequality, $d_{\mathrm{TV}}\le\sqrt{\KL/2}$, and Jensen's inequality over
$\Xp$ then give $\E|\Gamma|\le\KLbar+\sqrt{2\KLbar}$. For the reference case, chain three steps as in the proof of Theorem~\ref{thm:fixedref}. First, the event
$\qhat\ge q^\circ-r$ has probability at least $1-\delta$, where $q^\circ$ is the oracle's
threshold. Second, Lemma~\ref{lem:perturb} applies to the student's score, which
has a density by assumption. Third, \eqref{eq:propagation} applies at the threshold
$q^\circ-r$, and the density of $s^\circ$ gives $\Prob(s^\circ\le q^\circ-r)\ge1-\alpha-\fmax r$.

For the node-averaged score of FC-RAG, $s=K^{-1}\sum_is_i$ with
$s_i=-\log\Phat(y\mid X,Z_i)$, take as oracle the average $s^\circ=K^{-1}\sum_is_i^\circ$ of
the node oracle scores $s_i^\circ=-\log P_i^\circ(y\mid X,Z_i)$, with $P_i^\circ(\cdot\mid X,Z_i)$ the true answer law
given node $i$'s scoring context, and impose the density
condition on this averaged oracle score. Then $\Gamma=K^{-1}\sum_i\Gamma_i$ with $\Gamma_i=s_i-s_i^\circ$, and, with
$\KLbar_i=\E\,\KL(P_i^\circ(\cdot\mid X,Z_i)\|\Phat(\cdot\mid X,Z_i))$,
$\E|\Gamma|\le K^{-1}\sum_i\E|\Gamma_i|\le K^{-1}\sum_i(\KLbar_i+\sqrt{2\KLbar_i})\le\KLbar+\sqrt{2\KLbar}$
for the average $\KLbar=K^{-1}\sum_i\KLbar_i$, by concavity of the square root.
\end{proof}

\section{Experimental details}\label{app:experiments}

This appendix gives the settings, codec disclosures, and full results behind
Section~\ref{sec:experiments}. All real-model runs use pinned model, tokenizer, and dataset revisions; paired seeds; and complete logical payload accounting. Nodes are simulated
sequentially on one GPU, so runtimes measure computation, not networked
latency. Numbers in this section are read from the result files
(Appendix~\ref{app:repro}).
Table~\ref{tab:evidence} maps every study to its stage, codec status, and role.
Intervals are 95\% Student-$t$ for arms, paired-$t$ for paired contrasts
(differences between arms that share seeds), paired-seed bootstrap for synthetic
contrasts, and subject-stratified query bootstrap for FC-RAG.

\begin{table}[tb]
\caption{Evidence map. Stages: T training, C calibration, Q query. Codec status
(Table~\ref{tab:codecs}): \emph{exact}, no coordinate violated the range condition
of Definition~\ref{def:codec} in any node and round (verified by counting);
\emph{empirical}, gauge+dither runs at $L=40$ that clipped at most
$3.4\times10^{-6}$ of coordinates. The codec factorial's cells without centering or
without dither are, by design, not instances of Definition~\ref{def:codec}.}
\label{tab:evidence}
\centering\small
\setlength{\tabcolsep}{4pt}
\begin{tabular}{@{}>{\raggedright\arraybackslash}p{0.5\textwidth}cll@{}}
\toprule
Study (scale) & Stage & Codec status & Role\\
\midrule
Local estimation, exact dither, probe shift (500; 1{,}000; 20 seeds) & T & exact & mechanism, \S\ref{sec:exp_components}\\
Propagation from FPLD students to coverage (24 cells, 20 seeds) & T, C & exact, $L=14$ & mechanism, \S\ref{sec:exp_components}\\
Matched GPT-2 suite (3 seeds; 10 for FP32 and 8/4-bit FPLD) & T & empirical, $L=40$ & protocol, \S\ref{sec:exp_training}\\
Codec factorial, centering $\times$ dither at 8, 4, 2 bits (3 seeds) & T & empirical, $L=40$ & protocol, \S\ref{sec:exp_training}\\
GPT-2-medium and AG News suites, FPLD at FP32, 8, 4 bits (3 seeds each) & T & empirical, $L=40$ & protocol, \S\ref{sec:exp_training}\\
Exact GPT-2 run at $L=64$ (3 seeds) & T & exact, $L=64$ & protocol, \S\ref{sec:exp_training}\\
Mixed GPT-2 / GPT-Neo factorial ($2\times2$, 8 seeds) & T & exact, $L=64$ & protocol, \S\ref{sec:exp_hetero}\\
Topic-partition factorial ($2\times2$, 10 seeds) & T & empirical, $L=40$ & protocol, \S\ref{sec:exp_hetero}\\
All-node FC-RAG with GPT-2 and Qwen2.5 scorers, histogram, routing (914 MMLU questions) & Q, C & exact & protocol, \S\ref{sec:exp_fcrag}\\
Cached-score $K$-averaging replay & Q & exact & mechanism, \S\ref{sec:exp_fcrag}\\
\bottomrule
\end{tabular}
\end{table}

\subsection{Codecs and payload accounting}\label{app:exp_codec}

\begin{table}[tb]
\caption{Every gauge+dither run clips at most $3.4\times10^{-6}$ of its coordinates,
and the $L=64$, microbenchmark, and FC-RAG runs clip none; only the uncentered and
naive-midpoint arms clip heavily ($55$--$57\%$ of coordinates in round 1). The range condition $|z_v|\le L$ of
Definition~\ref{def:codec} is checked on every encoded coordinate; the
implementation clips violating coordinates to $\pm L$ and counts them. A code \emph{overloads} when, after clipping and dither, its rounded index falls
outside $\{0,\dots,2^b-1\}$ and must be clamped. No gauge-centered run overloads a code; the uncentered dithered arms of the codec factorial overload $1.6\times10^{-7}$ of their coordinates, pooled over the three bit widths, and at most $8.8\times10^{-7}$ in any seed-round (one round of one seed's run). The $L=64$ runs also satisfy the stricter margin
$|z_v|\le L-\Delta/2$ and never use an endpoint code. Under the range condition the per-coordinate error variance is $\Delta^2/12=L^2/[3(2^b-1)^2]$. \emph{Biased}: nonzero-mean error.}
\label{tab:codecs}
\centering\footnotesize
\setlength{\tabcolsep}{2.2pt}
\begin{tabular}{@{}llll@{}}
\toprule
Study & Codec & Range condition (all nodes and rounds) & Status\\
\midrule
Exact-dither microbenchmark & dither, $L=2$ & holds, $24.6$M coordinates & exact\\
Matched suite, gauge+dither & $L=40$, $b\in\{8,4,2\}$ & clipped in $8/15$, $12/15$, $5/15$ seed-rounds & empirical\\
\quad (clipped fraction) & & $\le1.1\times10^{-7}$, $\le1.5\times10^{-7}$, $\le3.7\times10^{-8}$ & \\
Ten-seed codec runs, random partition & $L=40$, $b\in\{8,4\}$ & clipped in $35/50$ and $35/50$ seed-rounds & empirical\\
\quad (clipped fraction) & & $\le1.1\times10^{-7}$, $\le1.6\times10^{-7}$ & \\
Topic-partition runs, ten seeds & $L=40$, $b=8$ & clipped in $34/50$ seed-rounds, $\le2.7\times10^{-7}$ & empirical\\
Codec factorial, gauge+dither & $L=40$, $b\in\{8,4,2\}$ & clipped in $8/15$, $12/15$, $5/15$ seed-rounds & empirical\\
\quad (clipped fraction) & & $\le3.7\times10^{-8}$, $\le1.6\times10^{-7}$, $\le3.7\times10^{-8}$ & \\
Codec factorial, no dither & $L=40$, $b\in\{8,4,2\}$ & clipped in $11/15$, $10/15$, $14/15$ seed-rounds & not dithered\\
\quad (clipped fraction) & & $\le5.9\times10^{-8}$, $\le3.7\times10^{-8}$, $\le5.9\times10^{-5}$ & \\
Codec factorial, uncentered & $L=40$, $b\in\{8,4,2\}$ & clips $55$--$57\%$ in round 1, $\ge10\%$ later & biased\\
GPT-2-medium suite, gauge+dither & $L=40$, $b\in\{8,4\}$ & clipped in $15/15$ and $15/15$ seed-rounds & empirical\\
\quad (clipped fraction) & & $\le6.1\times10^{-7}$, $\le3.4\times10^{-6}$ & \\
AG News suite, gauge+dither & $L=40$, $b\in\{8,4\}$ & clipped in $15/15$ and $15/15$ seed-rounds & empirical\\
\quad (clipped fraction) & & $\le1.6\times10^{-6}$, $\le2.8\times10^{-6}$ & \\
Exact run & $L=64$, $b=8$ & holds, $6.18\times10^9$ coordinates, $15/15$ seed-rounds & exact\\
Mixed-architecture factorial & $L=64$, $b=8$ & holds in all 16 dithered runs & exact\\
FC-RAG, both scorers, and replays & scalar dither on $[0,20]$ & holds; no score clipped & exact\\
Matched-suite midpoint ablation & naive midpoint, $L=40$ & clips $56$--$57\%$ in round 1 & biased\\
\bottomrule
\end{tabular}
\end{table}

\begin{table}[tb]
\caption{For GPT-2-small, at equal bit width, FPLD probe logits are $0.827$ times the size of a dense model delta. Payload per node per round ($d=124{,}439{,}808$, $m=2{,}048$ probes, $V=50{,}257$). Quantized model deltas include one FP32 scale
for each of the 148 parameter tensors.}
\label{tab:payload}
\centering\small
\begin{tabular}{@{}lcc@{}}
\toprule
Message & bits per coordinate & MiB per node per round\\
\midrule
Dense model delta & 32 / 8 / 4 / 2 & 474.700 / 118.676 / 59.338 / 29.669\\
FPLD probe logits & 32 / 8 / 4 / 2 & 392.633 / 98.158 / 49.079 / 24.540\\
\bottomrule
\end{tabular}
\end{table}

\paragraph{Certificate values.}
Section~\ref{sec:training} gives the certificate of Theorem~\ref{thm:codec} at 8 and 4 bits; at 2 bits and the matched suite's $L=40$, where a few coordinates clip (Table~\ref{tab:codecs}), it is $7.4$ nats ($K=4$, $\tau=1$). In the exact $L=64$
run, the measured error variance matches $\Delta^2/12$ to within $0.01\%$.

\paragraph{Payload comparison.}
A $b$-bit FPLD round uploads $mVb$ bits per node and a $b'$-bit dense delta costs about $db'$ bits (Table~\ref{tab:payload}), so FPLD uploads less whenever $mVb<db'$. At equal nominal width, FPLD's
payload is $mV/d=0.827$ times that of a dense delta. An 8-bit FPLD round is
therefore $17.3\%$ smaller than an 8-bit delta but $1.65$ and $3.31$ times the size of a 4-bit and a 2-bit delta, respectively. The break-even precision of a dense update against 8-bit FPLD is
$mV\cdot8/d=6.617$ bits per parameter, and for fixed $(m,V)$, 8-bit FPLD undercuts a
dense 4-bit update once $d>2mV=205{,}852{,}672$. GPT-2-medium ($d=354{,}823{,}168$; 292
parameter tensors) is above this crossover: a dense 4-bit delta costs $169.2$ MiB
per node per round, while the 8-bit FPLD round stays at $98.2$ MiB.

\paragraph{Accounting boundary.}
Each budget counts only the codes a node uploads for its message; all other
traffic is reported on separate lines. $B_i$ counts the $C$ score codes a node
uploads per query; with $C=K=4$, all-node
FC-RAG uploads $4b_i$ bits per node and $16b_i$ bits per query in total. The query
broadcast, candidate identifiers when the list is not pre-shared, routing codes,
and transport headers are excluded from $B_i$. $\Bcal$ counts the one-shot summary each owner sends after the averaged true-label scores of its calibration examples are
formed; forming them uses the per-query exchange, for example
$457\times128=58{,}496$ logical bits per split for all-node 8-bit scores. Downlinks are
reported separately: the shared-architecture protocol broadcasts the student every
round ($9.271$ GiB per matched-suite run in FP32), and function-space
synchronization broadcasts an $m\times V$ matrix of soft targets.

\subsection{Mechanism checks}\label{app:exp_mechanism}

\paragraph{Local estimation.}
This check tests the sampling rate $d/(2Kn)$ of Theorem~\ref{thm:training}(ii)
with an exactly computed KL risk. The model is an intercept-only categorical
softmax with $V=16$ and true logits
evenly spaced on $[-1,1]$, so $d=V-1=15$. Each node forms the Jeffreys-smoothed estimate
$(\text{count}+\tfrac12)/(n+V/2)$, which keeps every estimate interior, and the hub
averages the local log-probabilities and renormalizes. $K\in\{1,2,4,8,16\}$ is
crossed with $n\in\{128,256,\dots,4{,}096\}$ over 500 seeds. Each seed draws one nested
$K_{\max}\times n_{\max}$ sample array, so every $(K,n)$ cell reuses data and
contrasts are paired. The largest departures from
$d/(2Kn)$, at $n=128$, come from the local estimates' $O(1/n)$ bias.

\paragraph{Exact dither.}
This check compares the measured teacher KL with the certificate of
Theorem~\ref{thm:codec} when no coordinate violates the range condition. The
teacher has $V=256$ logits evenly spaced on $[-0.5,0.5]$, so
$1-\|p\|_2^2\approx0.996$, and is encoded with $L=2$; its logits lie at least half of
the coarsest step inside $[-L,L]$. $K\in\{1,2,4,8,16\}$ is crossed with
$b\in\{2,3,4,5,6,8\}$ over 1{,}000 paired seeds, and all $24{,}576{,}000$ encoded
coordinates satisfy the range condition. As an implementation check, the
averaged-logit error variance is $0.996$ to $1.006$ times $\Delta^2/(12K)$.

\paragraph{Probe shift.}
This check tests whether quadrupling the number of probes $m$ offsets doubling
the probe overlap $\rho$, as the probe term's dependence on $\rho/\sqrt m$
predicts. Here $K=4$ nodes each hold $n=3{,}000$ examples of a context-table model with 256
contexts and $V=256$ tokens, form Jeffreys-smoothed ($\tfrac12$) local estimates,
and encode with the 8-bit dithered codec at $L=20$. For each of 20 seeds, the
target context law $\Pstar_X$ is uniform on a random 64-context support, and the
probe law is $Q_\lambda=(1-\lambda)\Pstar_X+\lambda\,\mathcal U$, with $\mathcal U$ uniform on the other 192 contexts and
$\lambda=1-1/\rho$, so the density ratio is exactly $\rho$. The KL risk is evaluated over
the complete target support. The probe-interface error is the KL divergence from $P_{\mathrm{all}}$, the aggregate that the hub would form by probing every context, to the student $\Phat_m$ trained on the $m$ probes.
Table~\ref{tab:probe_contrasts} lists the six offset ratios.

\begin{table}[tb]
\caption{Quadrupling $m$ offsets doubling $\rho$ in every comparison (pooled KL-risk
ratio $0.985$), and the probe-interface error falls much faster than the bound's
rate. Top: probe-shift offset ratios, seed-paired geometric means after
doubling $\rho$ and quadrupling $m$, which leaves $\rho/\sqrt m$ unchanged, with 95\%
paired-seed bootstrap intervals (20 seeds); the pooled row is seed-blocked over
all six ratios. Bottom: log-log slope in $m$ of the probe-interface error at
each $\rho$, over $m\in\{256,1{,}024,4{,}096,16{,}384\}$, against the bound's $-1/2$
(95\% paired-seed bootstrap intervals).}
\label{tab:probe_contrasts}
\centering\small
\begin{tabular}{@{}cccc@{}}
\toprule
$\rho\to2\rho$ & $m\to4m$ & KL-risk ratio & probe-interface ratio\\
\midrule
$1\to2$ & $256\to1{,}024$ & $0.991\,[0.988,0.994]$ & $0.065\,[0.040,0.110]$\\
$1\to2$ & $1{,}024\to4{,}096$ & $1.000\,[1.000,1.000]$ & $0.477\,[0.472,0.483]$\\
$1\to2$ & $4{,}096\to16{,}384$ & $1.000\,[1.000,1.000]$ & $0.495\,[0.491,0.499]$\\
$2\to4$ & $256\to1{,}024$ & $0.924\,[0.914,0.933]$ & $0.093\,[0.055,0.151]$\\
$2\to4$ & $1{,}024\to4{,}096$ & $1.000\,[1.000,1.000]$ & $0.452\,[0.437,0.466]$\\
$2\to4$ & $4{,}096\to16{,}384$ & $1.000\,[1.000,1.000]$ & $0.494\,[0.488,0.499]$\\
\midrule
pooled & & $0.985\,[0.984,0.987]$ & $0.261\,[0.233,0.294]$\\
\midrule
$\rho$ & & & slope in $m$\\
\midrule
$1$ & & & $-1.58\,[-1.64,-1.50]$\\
$1.25$ & & & $-1.71\,[-1.76,-1.63]$\\
$2$ & & & $-1.92\,[-1.94,-1.90]$\\
$4$ & & & $-2.15\,[-2.18,-2.12]$\\
\bottomrule
\end{tabular}
\end{table}

\paragraph{Propagation.}
This check tests Corollary~\ref{cor:propagation} on a softmax-linear model that
satisfies (A1), with $V=16$ and contexts uniform on $[-1,1]^6$, so that the oracle
score $s^\circ$ has a density. For each of 20 seeds, $K=4$ nodes fit local MLEs on
$n\in\{250,1{,}000,4{,}000,16{,}000\}$ nested samples and send logits on $m=2{,}000$ fresh probes at
$b\in\{3,4,5,6,8\}$ bits ($L=14$, no coordinate clipped) or in FP32, and the hub
distills a student in one round. The answer is the next token, so the student's
expected KL to the oracle is its KL risk. Coverage is computed exactly over the
answer distribution on 200{,}000 shared contexts, at the oracle's $0.9$ quantile $q^\circ$.
We estimated $\fmax$ using the largest histogram density of $s^\circ$ conditional on
$\Gamma=s-s^\circ$, over ten $\Gamma$-quantile bins. The estimates average $1.1$ to
$1.7$ per cell; the maximum lies at confident answers ($s^\circ\approx0.15$), whereas the largest density within 1 of $q^\circ$ averages $0.32$ to $0.38$.
Table~\ref{tab:propagation} lists the results.

\begin{table}[tb]
\caption{Absolute coverage changes at the oracle's threshold are smaller than $\Dtrain$ in
every cell; coverage under reused student thresholds exceeds the two-model lower
bound. Propagation check
(Corollary~\ref{cor:propagation}): means over 20 seeds at 8 and 3 bits; FP32 logits
agree with 8 bits within $0.001$. KL: the student's expected KL divergence to the oracle ($\KLbar$, Corollary~\ref{cor:propagation}). Change: coverage at the oracle's threshold minus
$1-\alpha$. Disagreement: probability that student and oracle disagree on including the true answer. Reused: coverage of each row's student $B$ at the split-conformal threshold of student $A$, the one
with $n=1{,}000$ and 8 bits ($\ncal=1{,}000$; 200 draws), against the bound
$1-\alpha-\Dtrain^A-\Dtrain^B$.}
\label{tab:propagation}
\centering\small
\begin{tabular}{@{}rccccccc@{}}
\toprule
$n$ & bits & KL & change & disagreement & $\Dtrain$ & reused & bound\\
\midrule
250 & 8 & 0.0696 & $-0.0357$ & 0.0489 & 0.734 & 0.873 & $-0.055$\\
& 3 & 0.0605 & $-0.0271$ & 0.0443 & 0.597 & 0.881 & $0.082$\\
1{,}000 & 8 & 0.0126 & $-0.0079$ & 0.0193 & 0.220 & 0.900 & $0.459$\\
& 3 & 0.0127 & $-0.0002$ & 0.0185 & 0.189 & 0.908 & $0.490$\\
4{,}000 & 8 & 0.0030 & $-0.0020$ & 0.0092 & 0.098 & 0.906 & $0.581$\\
& 3 & 0.0047 & $+0.0051$ & 0.0110 & 0.121 & 0.913 & $0.559$\\
16{,}000 & 8 & 0.0007 & $-0.0004$ & 0.0044 & 0.045 & 0.907 & $0.634$\\
& 3 & 0.0030 & $+0.0068$ & 0.0086 & 0.098 & 0.915 & $0.581$\\
\bottomrule
\end{tabular}
\end{table}

\subsection{Matched GPT-2 / WikiText-2 suite}\label{app:exp_matched}

\paragraph{Settings.}
The matched suite compares FPLD with parameter-exchange and distillation
baselines under shared data, probes, schedules, and seeds. There are $K=4$
nodes, each with $4{,}000$ sequences of 128 tokens from the
WikiText-2 training split, and training runs $T=5$ rounds of $E=1$ local epoch
(batch 8, learning rate $5\times10^{-5}$). Distillation at the hub runs three epochs per round at learning
rate $5\times10^{-5}$ and temperature one, on the last-position logits of
$m=2{,}048$ probe contexts from the test split. Evaluation uses the 1{,}931
validation sequences. The pooled-data reference uses the same schedule and the
same total number of sequence visits as the federated arms.

\paragraph{Baselines.}
Quantized FedAvg
compresses the accumulated node deltas that FedAvg communicates, which makes
it the matched compressed-update comparator; compressing individual minibatch
gradients would require communication at every step. FedDF initializes the
hub's model by averaging the uploaded node models and distills on their
averaged logits on the probes~\citep{feddf2020}; it is charged the FP32 model upload. The
CFD-style arm applies the constrained simplex quantizer of \citet{cfd2022} to each node's predicted distributions, which the hub averages after decoding, under the same data, probes, and schedule. It omits CFD's
data curation, entropy coding, and full protocol, so it is a controlled codec
ablation, not a reproduction.

\begin{table}[tb]
\caption{At the best round, as at the final round, gauge+dither beats the naive and
CFD-style codecs at every bit width in every seed, and exchanging parameters remains more
accurate. Full matched suite (three paired seeds; mean $\pm$ SD). ``Best-round
perplexity'' is the mean, over seeds, of the lowest validation perplexity attained
in each run. Uplink is summed
over nodes and rounds; runtime is the mean sequential wall time. Top-1 accuracy is the fraction of
validation tokens whose most probable predicted token is correct. ``Temporal'':
CFD's between-round delta payload (Appendix~\ref{app:exp_matched}). A dash marks a
value that does not apply; the pooled-data reference moves raw data, which the
uplink accounting does not charge.}
\label{tab:matched_full}
\centering\footnotesize
\setlength{\tabcolsep}{3.5pt}
\begin{tabular}{@{}lcccccc@{}}
\toprule
Method & bits & \shortstack{Final\\perplexity} & \shortstack{Best-round\\perplexity} & top-1 accuracy & uplink (GiB) & time (h)\\
\midrule
Pooled-data reference & --- & $34.49\pm0.10$ & 30.34 & 0.3835 & --- & 0.306\\
FedAvg & 32 & $30.76\pm0.01$ & 30.40 & 0.3894 & 9.271 & 0.309\\
Quantized FedAvg + EF & 8 & $30.76\pm0.01$ & 30.40 & 0.3893 & 2.318 & 0.315\\
Quantized FedAvg + EF & 4 & $30.78\pm0.02$ & 30.43 & 0.3894 & 1.159 & 0.317\\
Quantized FedAvg + EF & 2 & $33.29\pm0.27$ & 31.45 & 0.3795 & 0.579 & 0.316\\
FedDF & 32 & $35.26\pm0.41$ & 35.01 & 0.3732 & 9.271 & 0.463\\
FPLD, FP32 logits & 32 & $43.57\pm0.44$ & 40.44 & 0.3412 & 7.669 & 0.464\\
FPLD, gauge+dither & 8 & $44.22\pm0.51$ & 40.48 & 0.3396 & 1.917 & 0.472\\
FPLD, gauge+dither & 4 & $43.33\pm0.49$ & 41.09 & 0.3413 & 0.959 & 0.472\\
FPLD, gauge+dither & 2 & $65.88\pm1.15$ & 58.95 & 0.2934 & 0.479 & 0.472\\
FPLD, naive midpoint & 8 & $58.95\pm1.44$ & 55.15 & 0.3361 & 1.917 & 0.468\\
FPLD, naive midpoint & 4 & $61.48\pm5.15$ & 58.79 & 0.3333 & 0.959 & 0.466\\
FPLD, naive midpoint & 2 & $132.16\pm40.01$ & 113.09 & 0.3048 & 0.479 & 0.466\\
\midrule
LoRA-FedAvg, lr $5\times10^{-5}$ & 32 / 8 & $34.94\pm0.05$ & 34.94 & 0.3680 & 22.5 / 5.63 MiB & 0.202\\
LoRA-FedAvg, lr $3\times10^{-4}$ & 32 / 8 & $32.23\pm0.04$ & 32.23 & 0.3801 & 22.5 / 5.63 MiB & 0.204\\
CFD-style $Q_b$ & 8 & $48.86\pm0.85$ & 41.90 & 0.3366 & 1.917 (temporal 0.388) & 0.484\\
CFD-style $Q_b$ & 4 & $74.12\pm1.29$ & 50.25 & 0.3219 & 0.959 (temporal 0.192) & 0.484\\
CFD-style $Q_b$ & 2 & $125.42\pm4.14$ & 63.09 & 0.3077 & 0.479 (temporal 0.096) & 0.482\\
CFD-style $Q_b$ & 1 & $977\pm170$ & 106.4 & 0.3050 & 0.078 MiB (temporal 0.020 MiB) & 0.481\\
\bottomrule
\end{tabular}
\end{table}

\paragraph{Paired contrasts.}
FPLD with FP32 logits or with 8- or 4-bit gauge+dither trails FP32 FedAvg by a
similar margin, whereas 4-bit quantized FedAvg nearly matches it. Relative to FP32
FedAvg, the paired perplexity differences are
$+12.57\,[11.37,13.76]$ for 4-bit FPLD, $+13.46\,[12.19,14.73]$ for 8-bit FPLD,
$+12.81\,[11.72,13.90]$ for FPLD with FP32 logits, $+4.50\,[3.51,5.49]$ for FedDF,
$+3.73\,[3.47,3.99]$ for the pooled reference, and $+0.025\,[0.018,0.031]$ for
4-bit quantized FedAvg. The 4-bit FPLD arm trails the pooled reference by
$+8.83\,[7.74,9.93]$ and FedDF by $+8.07\,[7.70,8.44]$, and the naive midpoint codec
is worse than FP32 logits by $+15.38\,[11.97,18.79]$ at 8 bits.

\paragraph{Seed extensions.}
The three-seed 8-bit gap of Table~\ref{tab:matched} does not persist over the ten-seed
contrasts of Section~\ref{sec:exp_training}, which add seeds 3--9 to the three
seeds of the suite; this paragraph gives their sources and their values by seed
group and round. For seeds 3--9 the FP32 and 4-bit arms come from replication
runs and the 8-bit arm from the random-partition arm of the partition factorial
(Appendix~\ref{app:exp_partition}), all on RTX~6000 GPUs. At the final round, the
4-bit gauge+dither codec differs from FP32 logits by $-0.25\,[-0.35,-0.14]$ on seeds
0--2. At 8 bits the difference is $+0.65\,[0.15,1.15]$ on seeds 0--2 and
$-0.61\,[-1.16,-0.07]$ on seeds 3--9. Seed-to-seed variation is larger in the
extension than in the first three seeds. Round by round over the ten seeds, the
mean 8-bit difference lies between $-0.34$ and $+0.07$ with every interval
containing zero, whereas the 4-bit difference falls from $+0.86\,[0.64,1.08]$ after
round 1 to $+0.44\,[0.20,0.67]$, $+0.27\,[-0.06,0.60]$, $+0.04\,[-0.45,0.53]$, and
$-0.49\,[-1.16,0.17]$ after rounds 2 to 5.

\paragraph{Best round.}
Under the fixed schedule, FPLD with FP32 or gauge+dither logits and the CFD-style
arm reach their lowest perplexity in round 1 or 2 and then rise, whereas FedAvg,
FedDF, and the naive midpoint codec reach theirs in rounds 2 to 5
(Table~\ref{tab:matched_full}). Reporting each seed's best round narrows the
gap between FPLD with FP32 logits and FedAvg from $12.8$ to $10.0$ perplexity
($40.44$ against $30.40$). Gauge+dither still beats the naive and CFD-style codecs at every width in every seed ($40.48$, $41.09$, and $58.95$ at 8, 4, and 2 bits, against
$55.15$, $58.79$, and $113.09$ for the naive codec and $41.90$, $50.25$, and $63.09$
for the CFD-style codec). Over the ten seeds, the best-round difference from FP32
logits is $+0.07\,[-0.07,0.21]$ at 8 bits and $+0.75\,[0.60,0.90]$ at 4 bits.

\paragraph{Codec factorial.}
At each arm's best round, as at the final round, removing centering costs
perplexity at every bit width and removing dither costs it at 4 and 2 bits. Against gauge+dither, removing dither changes best-round perplexity by $-0.25\,[-0.82,0.33]$ at 8 bits and $+1.98\,[0.96,2.99]$ at 4 bits, and removing centering by $+12.1\,[4.5,19.8]$, $+10.8\,[5.0,16.5]$, and $+17.6\,[15.1,20.0]$ at 8, 4, and 2 bits. All four cells of Table~\ref{tab:factorial} use the endpoint grid of
Definition~\ref{def:codec} with $L=40$: the dithered cells subtract the shared
dither, the undithered cells round to the nearest grid point, and the uncentered
cells clip raw logits to $[-L,L]$. The uncentered dithered cell reuses the dither
stream of the gauge+dither cell, so the two dithered cells see common random
numbers, and every cell shares data, probes, and local-update seeds with the
matched suite. All arms of one seed and one bit width ran on one GPU type (V100
for seed 0 at 8 and 4 bits, RTX~6000 otherwise; the matched suite ran on RTX~6000),
so every contrast within the factorial compares runs on the same hardware.

The
gauge+dither cell reruns the matched suite's arm: seeds 1 and 2 reproduce it
exactly, and the final-round paired differences, $-0.24\,[-1.29,0.81]$ and $-0.04\,[-0.19,0.12]$ at
8 and 4 bits, come entirely from seed 0, rerun on a V100; the 2-bit values are
identical. Uncentered quantization clips $55$--$57\%$ of coordinates in round 1,
with a mean error of $+13$ to $+20$ logits for the naive codec. At each
arm's best round the four cells (gauge+dither, centering only, dither only,
neither) reach $40.85$, $40.60$, $52.98$, and $53.07$ at 8 bits; $41.22$, $43.19$,
$51.97$, and $56.41$ at 4 bits; and $58.95$, $759.8$, $76.52$, and $96.22$ at 2 bits. Final-round top-1 accuracy of gauge+dither is $0.340$, $0.341$, and $0.293$ at 8,
4, and 2 bits, and $0.144$ for the collapsed 2-bit cell without dither.

\paragraph{Runtime ledger.}
FPLD's extra runtime over quantized FedAvg comes mostly from distillation at the
hub; the codec and aggregation take about $2\%$ of an FPLD run. A 4-bit FPLD run takes $1{,}697.6$ s on average: $1{,}064.1$ s of node local
training, $157.1$ s of probe evaluation, $33.0$ s of codec and aggregation,
$399.7$ s of distillation at the hub, $5.5$ s of other orchestration, and $38.1$ s of
evaluation. A 4-bit quantized FedAvg run takes $1{,}140.4$ s, so an FPLD run takes
$1.49$ times as long.

\paragraph{LoRA details.}
LoRA adapts the fused query-key-value projection (\texttt{c\_attn}) of all 12
layers at rank 8, for 294{,}912 trainable coordinates in 24 adapter tensors. Its
8-bit uplink includes one FP32 scale per tensor, and FP32 and 8-bit LoRA agree
to within $0.001$ perplexity. The arm at learning rate $5\times10^{-5}$ is paired with
the other arms; the arm at $3\times10^{-4}$ is a sensitivity analysis. The adapter
factors are averaged separately, unlike FLoRA's stacked aggregation~\citep{flora2024}, so the arm is a reference point, not a tuned competitor.

\paragraph{CFD-style details.}
The CFD-style
arm evolves its own global model under the shared schedule, and at vocabulary
scale its simplex code has at most $2^b-1$ nonzero token probabilities per probe.
Its ``fixed'' payload sends every quantized coordinate; its ``temporal'' payload
sends full or sparse deltas between rounds under a separately tested decodable
layout. At one bit the output is one-hot, so the primary representation is a
16-bit vocabulary index per probe: $0.078$ MiB in total, against $0.2396$ GiB
under the conservative $mV$-coordinate layout.

\subsection{GPT-2-medium and AG News suites}\label{app:exp_scale}

\paragraph{Settings.}
The two suites test the codec on a larger model and on a second corpus. Both
suites keep the design of the matched suite ($K=4$ nodes, each with $4{,}000$
sequences of 128 tokens, $T=5$ rounds of one local epoch, batch 8, learning
rate $5\times10^{-5}$, $m=2{,}048$ probes, three distillation epochs per round at
temperature one, $L=40$, and seeds 0--2), and each changes one element. The GPT-2-medium
suite replaces the model with GPT-2-medium ($d=354{,}823{,}168$; 24 layers of width
1{,}024; the same 50{,}257-token vocabulary) on WikiText-2. The AG News suite keeps
GPT-2-small and replaces the corpus with AG News~\citep{zhang2015charcnn}: the
articles of its training split, joined by a newline, supply the nodes' shards
and a disjoint evaluation set of 1{,}931 held-out sequences, and the probes come
from its test split, whose labels are never used.

Table~\ref{tab:scale} reports the results: every 8- or 4-bit interval against FP32 logits contains zero, and the best-round one-sided 95\% upper bounds are at most $0.80$ perplexity. Both suites run on RTX~6000 GPUs, and every FPLD arm reaches its lowest perplexity in round 1. At $L=40$ the centered logits clip in every seed-round, on at most $3.4\times10^{-6}$ of coordinates (Table~\ref{tab:codecs}).

\begin{table}[tb]
\caption{The codec shows no detectable loss on a larger model or a second corpus:
every interval for 8- or 4-bit gauge+dither against FP32 logits contains zero,
at the final round and at each arm's best round. FPLD on GPT-2-medium and on AG
News (three paired seeds; mean $\pm$ SD; 95\% paired-$t$ intervals). At the best
round, the one-sided 95\% upper bounds on the 8- and 4-bit codec cost are $0.80$
and $0.58$ on GPT-2-medium and $0.05$ and $0.71$ on AG News. Top-1: next-token accuracy on held-out tokens, as in Table~\ref{tab:matched_full}.
Uplink sums the logical payload over nodes and rounds; time is the mean
sequential wall time on one GPU. The pretrained perplexity is $43.40$ for GPT-2-medium and $78.1$ to
$78.4$ for AG News, whose held-out evaluation set depends on the seed.}
\label{tab:scale}
\centering\footnotesize
\setlength{\tabcolsep}{3pt}
\begin{tabular}{@{}llccccccc@{}}
\toprule
Suite & Logits & \shortstack{Final\\perplexity} & \shortstack{Best-round\\perplexity} & top-1 & uplink (GiB) & time (h) & \shortstack{minus FP32\\(final)} & \shortstack{minus FP32\\(best round)}\\
\midrule
\multirow{3}{*}{GPT-2-medium} & FP32 & $39.95\pm0.80$ & $32.64$ & 0.348 & 7.669 & 1.24 & --- & ---\\
 & 8 bits & $40.99\pm0.97$ & $32.80$ & 0.345 & 1.917 & 1.26 & $+1.04\,[-2.22,4.30]$ & $+0.16\,[-0.79,1.10]$\\
 & 4 bits & $39.01\pm0.47$ & $32.84$ & 0.351 & 0.959 & 1.26 & $-0.93\,[-3.81,1.94]$ & $+0.20\,[-0.35,0.76]$\\
\midrule
\multirow{3}{*}{AG News} & FP32 & $46.78\pm0.92$ & $43.00$ & 0.340 & 7.669 & 0.47 & --- & ---\\
 & 8 bits & $46.85\pm0.69$ & $42.91$ & 0.340 & 1.917 & 0.49 & $+0.07\,[-2.04,2.18]$ & $-0.09\,[-0.29,0.11]$\\
 & 4 bits & $46.35\pm1.27$ & $43.37$ & 0.341 & 0.959 & 0.46 & $-0.43\,[-1.34,0.47]$ & $+0.37\,[-0.14,0.87]$\\
\bottomrule
\end{tabular}
\end{table}

\subsection{Mixed-architecture factorial}\label{app:exp_arch}

The factorial tests whether the 8-bit codec costs accuracy when nodes run
different architectures; neither codec contrast is detectable, and 8-bit logits
cut total communication $1.71$- to $2.67$-fold. It crosses the node set
(four GPT-2-small nodes, or two
GPT-2-small and two GPT-Neo-125M nodes) with the codec (FP32 logits or 8-bit
gauge+dither at $L=64$); Table~\ref{tab:arch} in Section~\ref{sec:exp_hetero} reports
its cells and contrasts. The common student is GPT-2-small in every cell. A hash
check confirms that the two model families share the 50{,}257-token output map,
special tokens, and probe tokenization. In the mixed cells, the assignment of the two architectures to the four data shards rotates with the seed, and all cells share data, probes, and seeds. Rounds follow the synchronization protocol
of Appendix~\ref{app:pseudocode}.

The 8-bit arms upload $1.917$ GiB against $7.669$ GiB
for FP32 logits. The synchronization downlink is $1.534$ GiB per
run if each round's soft targets are counted once, and $6.135$ GiB if each of the
four node deliveries is counted. The total communication is therefore $3.451$
against $9.202$ GiB (8-bit against FP32 logits, a $2.67$-fold reduction) under
once-per-round counting, and $8.052$ against $13.803$ GiB (a $1.71$-fold reduction)
under per-delivery counting. In the mixed system, the top-1 accuracy difference
between 8-bit and FP32 logits is $-0.135$ percentage points $[-0.356,0.087]$. At
eight seeds the one-sided 95\% upper bounds are $0.92$ and $1.37$ for the
homogeneous and mixed codec contrasts and $1.71$ for the interaction. Seeds 4--7 were run because the four-seed 95\%
interval of the mixed codec contrast had half-width $1.88$ perplexity, above
the $0.75$ target chosen for the extension. Figure~\ref{fig:arch_rounds} shows the
round-wise perplexity of the student: all four cells reach their lowest
perplexity in round 1 and rise afterward, as in the matched suite.

\begin{figure}[tb]
\centering
\includegraphics[width=0.55\textwidth]{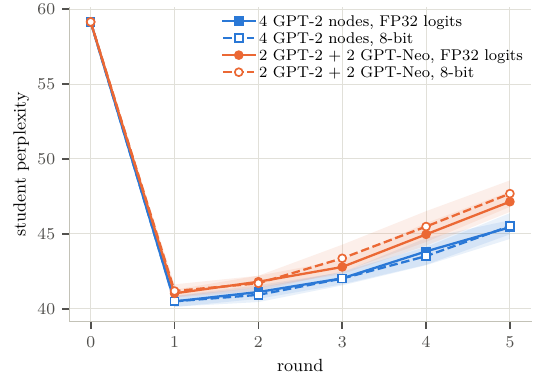}
\caption{With homogeneous or mixed nodes, the 8-bit codec tracks FP32 logits in
every round, and every cell reaches its lowest perplexity in round 1.
Round-wise perplexity of the common GPT-2 student in the mixed-architecture
factorial (eight paired seeds; bands are 95\% $t$-intervals).
Round 0 is the pretrained student ($59.15$).}
\label{fig:arch_rounds}
\end{figure}

\subsection{Topic-partition factorial}\label{app:exp_partition}

Topic-skewed data raises perplexity but does not detectably change the cost of the 8-bit codec. The topic partition clusters training sequences with capacitated $k$-means on token-occurrence vectors weighted by term frequency--inverse document frequency (TF-IDF), so that each node holds a distinct
mixture of topics. The factorial crosses the partition (random or topic) with the codec
(FP32 logits or 8-bit gauge+dither) over ten paired seeds, adding seeds 3--9 to a
three-seed version of the factorial. The paired contrasts are topic minus
random without quantization $+0.91\,[0.23,1.58]$; topic minus random at 8 bits
$+1.24\,[0.35,2.12]$; 8-bit minus FP32 on the random partition
$-0.23\,[-0.79,0.33]$; 8-bit minus FP32 on the topic partition $+0.10\,[-0.41,0.60]$;
and the interaction $+0.33\,[-0.44,1.10]$, with a one-sided 95\% upper bound of
$0.95$.

\subsection{FC-RAG on MMLU}\label{app:exp_fcrag}

\begin{table}[tb]
\caption{With the GPT-2-small scorer, 8 or more bits per score keep coverage and set size at the fixed FP32 threshold, and top-1 accuracy, within $0.001$ of FP32; recalibration keeps coverage between $0.890$ and $0.901$ at every bit width.
All-node FC-RAG on four MMLU
subjects ($K=4$, four candidates, $\alpha=0.1$). Values are means over five
dependent split-and-dither replicates;
Figure~\ref{fig:fcrag}a and Table~\ref{tab:fcrag_main} instead average each question over the replicates in
which it is a test question, so their values differ slightly from the change
columns here. The two change columns evaluate quantized scores at the fixed FP32 threshold;
the ``recal.'' columns use each method's own threshold. The mean squared error (MSE) of the averaged score is
measured against the unquantized score and stays within $2.2\%$ of $\Delta_s^2/(12K)$
at every bit width. The last column estimates the local form
of the slack of Theorem~\ref{thm:fixedref} (Appendix~\ref{app:proof_inference}) from
each split's density window, averaged over the five splits. A dash marks a
value that does not apply.}
\label{tab:allnode}
\centering\footnotesize
\setlength{\tabcolsep}{2.4pt}
\begin{tabular}{@{}lccccccccc@{}}
\toprule
Scores & $B_i$ (bits) & \shortstack{cov.\ change\\at FP32 thr.} & \shortstack{set-size\\change} & \shortstack{recal.\\cov.} & \shortstack{recal.\\set size} & top-1 & MSE & $\Drag$ (est.)\\
\midrule
FP32 & 128 & --- & --- & 0.893 & 3.579 & 0.297 & --- & ---\\
dither, 2 bits & 8 & $-0.044$ & $-0.169$ & 0.901 & 3.604 & 0.262 & $9.33\times10^{-1}$ & $0.74$\\
dither, 3 bits & 12 & $-0.007$ & $-0.026$ & 0.893 & 3.569 & 0.269 & $1.71\times10^{-1}$ & $0.14$\\
dither, 4 bits & 16 & $-0.004$ & $-0.007$ & 0.890 & 3.572 & 0.267 & $3.67\times10^{-2}$ & $0.027$\\
dither, 6 bits & 24 & $0.000$ & $-0.002$ & 0.891 & 3.573 & 0.279 & $2.09\times10^{-3}$ & $0.0046$\\
dither, 8 bits & 32 & $0.000$ & $0.000$ & 0.893 & 3.582 & 0.297 & $1.31\times10^{-4}$ & $0.0011$\\
dither, 12 bits & 48 & $0.000$ & $0.000$ & 0.893 & 3.579 & 0.298 & $5.0\times10^{-7}$ & $6.5\times10^{-5}$\\
dither, 16 bits & 64 & $0.000$ & $0.000$ & 0.893 & 3.579 & 0.297 & $1.9\times10^{-9}$ & $<10^{-5}$\\
dither, 32 bits & 128 & $0.000$ & $0.000$ & 0.893 & 3.579 & 0.297 & $<10^{-15}$ & $<10^{-5}$\\
pooled central, 3 retrieved & --- & --- & --- & 0.893 & 3.562 & 0.289 & --- & ---\\
pooled central, 12 retrieved & --- & --- & --- & 0.895 & 3.559 & 0.289 & --- & ---\\
\bottomrule
\end{tabular}
\end{table}

\begin{table}[tb]
\caption{With a stronger scorer, 8 bits per score again keep top-1 accuracy, as well as coverage and set size at the fixed FP32 threshold, within $0.001$ of FP32; recalibration keeps coverage between $0.904$ and $0.907$, and the estimated local
slack exceeds every observed coverage loss. The setup and columns are those of
Table~\ref{tab:allnode}, with Qwen2.5-1.5B-Instruct as the scorer; its FP32 sets
cover $0.905$ against $0.836$ for random sets of the same size. The MSE of the
averaged score stays within $3\%$ of $\Delta_s^2/(12K)$ at every bit width.}
\label{tab:allnode_qwen}
\centering\footnotesize
\setlength{\tabcolsep}{2.4pt}
\begin{tabular}{@{}lccccccccc@{}}
\toprule
Scores & $B_i$ (bits) & \shortstack{cov.\ change\\at FP32 thr.} & \shortstack{set-size\\change} & \shortstack{recal.\\cov.} & \shortstack{recal.\\set size} & top-1 & MSE & $\Drag$ (est.)\\
\midrule
FP32 & 128 & --- & --- & 0.905 & 3.344 & 0.447 & --- & ---\\
dither, 2 bits & 8 & $-0.085$ & $-0.312$ & 0.904 & 3.414 & 0.333 & $9.32\times10^{-1}$ & $0.94$\\
dither, 3 bits & 12 & $-0.015$ & $-0.061$ & 0.907 & 3.356 & 0.383 & $1.73\times10^{-1}$ & $0.39$\\
dither, 4 bits & 16 & $-0.005$ & $-0.014$ & 0.905 & 3.348 & 0.393 & $3.69\times10^{-2}$ & $0.057$\\
dither, 6 bits & 24 & $0.000$ & $-0.001$ & 0.906 & 3.348 & 0.441 & $2.08\times10^{-3}$ & $0.0051$\\
dither, 8 bits & 32 & $-0.001$ & $+0.001$ & 0.905 & 3.345 & 0.446 & $1.31\times10^{-4}$ & $0.0011$\\
dither, 12 bits & 48 & $0.000$ & $0.000$ & 0.905 & 3.345 & 0.445 & $5.0\times10^{-7}$ & $6.8\times10^{-5}$\\
dither, 16 bits & 64 & $0.000$ & $0.000$ & 0.905 & 3.344 & 0.447 & $2.0\times10^{-9}$ & $<10^{-5}$\\
dither, 32 bits & 128 & $0.000$ & $0.000$ & 0.905 & 3.344 & 0.447 & $<10^{-15}$ & $<10^{-5}$\\
pooled central, 3 retrieved & --- & --- & --- & 0.906 & 3.393 & 0.453 & --- & ---\\
pooled central, 12 retrieved & --- & --- & --- & 0.900 & 3.383 & 0.460 & --- & ---\\
\bottomrule
\end{tabular}
\end{table}

\paragraph{Design details.}
The design keeps every score inside the pre-shared interval and admits no
exact duplicate of a test question among the demonstrations. The four subjects
are high-school statistics, physics, biology, and world
history. Their 118 demonstrations (28, 22, 37, and 31) come from the
development and validation splits, and the 914 test questions from the test
split. The pre-shared score interval is $[0,20]$ ($\Smax=20$); no score is
clipped, and the calibration fraction is one half. The five split-and-dither
replicates measure sensitivity rather than independent replication.
The length-normalized negative log-probability is used as a black-box nonconformity score, not as a candidate-normalized score in the sense of Section~\ref{sec:propagation}.

\paragraph{Pooled-central baselines.} The pooled-central baselines differ from the mean recalibrated coverage of all-node FP32 scoring by at most $0.006$, within its standard deviation across the five splits ($0.013$ with GPT-2-small and $0.018$ with Qwen2.5). They retrieve the top 3 or top 12 demonstrations from the union of the corpora and use exact central thresholds; their communication is not zero, but moving the corpora falls outside our uplink accounting.

\paragraph{Intervals and audits.}
The fixed-threshold intervals come from a
subject-stratified bootstrap over the 891 questions that fall in at least one
of the five test halves, after averaging each question's repeated appearances;
at 8 bits, with GPT-2-small, they are $[-0.001,0.002]$ for coverage and $[-0.002,0.002]$ for set size.
An exact-match audit finds no test question among the demonstrations;
it normalizes text (Unicode NFKC, case folding, whitespace) and compares each question, with its ordered choices, against every demonstration; near-duplicates are not audited.
Across all nodes, retrieval ranks a demonstration from the query's own subject first for $95.7\%$ of
queries, a subject-level diagnostic rather than a measure of passage relevance.
The Qwen2.5-1.5B-Instruct scorer, at a pinned revision with FP32 weights,
reuses the retrieval, prompts, candidate scoring, and 1{,}024-token context of
the GPT-2 scorer, so $19.4\%$ of its all-node calls are truncated, as are $19.3\%$ of
GPT-2's; its scores also stay inside $[0,20]$.

\begin{table}[tb]
\caption{The conservative histogram never lowers coverage; at a 1{,}024-bit cap,
the FP32 score list is 3.6 times as large as the summary, which changes coverage by
$+0.004\pm0.005$. Conservative histogram calibration summary with all 457 calibration examples. Each node sends its sample count in 9 bits and every bin count but the last, which the total implies, in 8 bits, enough for the largest node's share. ``Inflation'' is the upward threshold shift relative to the
exact threshold. The 7-bin and 63-bin grids are
nested ($63=9\times7$), and in these splits the $k$-th smallest score falls in a fine bin
whose upper edge is also a coarse edge, so the two rows coincide; the other
grids are not nested, and the effect is not monotone in the cap.}
\label{tab:bcal}
\centering\small
\begin{tabular}{@{}cccccc@{}}
\toprule
Cap (bits/node) & bins & inflation & coverage change & set-size change & FP32 list / summary\\
\midrule
64 & 7 & $0.119\pm0.069$ & $+0.011\pm0.007$ & $+0.038\pm0.024$ & 64.1$\times$\\
128 & 15 & $1.262\pm0.069$ & $+0.070\pm0.011$ & $+0.273\pm0.039$ & 30.2$\times$\\
256 & 31 & $0.487\pm0.069$ & $+0.032\pm0.008$ & $+0.122\pm0.033$ & 14.7$\times$\\
512 & 63 & $0.119\pm0.069$ & $+0.011\pm0.007$ & $+0.038\pm0.024$ & 7.2$\times$\\
1{,}024 & 127 & $0.033\pm0.046$ & $+0.004\pm0.005$ & $+0.012\pm0.019$ & 3.6$\times$\\
\bottomrule
\end{tabular}
\end{table}

\paragraph{Evaluating Theorem~\ref{thm:fixedref}.}
The estimated local slack of Theorem~\ref{thm:fixedref} exceeds every observed
coverage loss. We estimate the density of the unquantized node-averaged
true-label score $S^\star$
with a Gaussian kernel over all 914 questions, using the bandwidth
$0.9\min(\hat\sigma,\mathrm{IQR}/1.34)\,N^{-1/5}=0.18$ of Silverman's rule of thumb,
with $\hat\sigma$ and IQR the standard deviation and interquartile range of the
$N=914$ scores.
Conditionally on a split's calibration data, the local form of the slack
(Appendix~\ref{app:proof_inference}) needs the density only within $\Delta_s/2$ of that
split's FP32 threshold; the five thresholds range from $2.67$ to $2.83$, where the
estimate itself is $0.089$ to $0.095$. The largest estimated density in these
windows is $0.09$ to $0.10$ for $b_i\ge6$, $0.13$ to $0.15$ at 4 bits, and $0.30$ to $0.36$ at
3 bits; at 2 bits every window reaches the mode of the density ($0.77$).

The local slack in the last column of Table~\ref{tab:allnode} exceeds every observed loss averaged over the replicates, by a factor of 17 at 2 bits, the only bit width whose change is measurably nonzero, and it is informative from 4 bits upward. It is $\hat f\Delta_s/\sqrt{12K}$, with $\hat f$ the largest estimated density in the window, averaged over the five splits. Using the maximum
estimated density, $0.77$, in the global slack formula gives $\Drag=0.0088$, $0.035$, $0.15$, and $0.32$ at 8,
6, 4, and 3 bits, still above every observed loss. The observed changes at 2
to 4 bits are negative, as Remark~\ref{rem:second_order} predicts when the density
is lower just above the threshold than just below it. We do not infer a scaling rate from these changes:
the 3- and 4-bit intervals contain zero, and at 2 bits the window reaches the mode.

For the Qwen scorer the local slack also covers changes that are measurably
nonzero. The same procedure gives a bandwidth of $0.13$, five
thresholds from $1.94$ to $2.10$ where the estimate is $0.081$ to $0.102$, and the largest window densities of $0.08$ to $0.12$ for $b_i\ge6$, $0.27$ to $0.31$ at 4 bits, and $0.87$
to $0.97$ at 3 bits; at 2 bits every window again reaches the mode ($0.97$). The local
slack (last column of Table~\ref{tab:allnode_qwen}) exceeds every observed loss,
including the 3- and 4-bit losses, whose intervals now exclude zero
($-0.017\,[-0.027,-0.008]$ and $-0.007\,[-0.012,-0.002]$ per question; the table's
replicate means are $-0.015$ and $-0.005$). Using the maximum estimated density,
$0.97$, in the global slack formula gives $\Drag=0.011$, $0.045$, $0.19$, and $0.40$ at 8, 6, 4, and 3 bits.

\paragraph{Routing and $K$-averaging.}
Routing to one node halves the query cost; against routed FP32 scoring, its
8-bit codes change mean coverage, set size, and top-1 accuracy by at most
$0.0006$; and the replay confirms that the variance of the quantization error
falls as $1/K$. With one routed node and 8-bit codes, a query costs 64 bits (Appendix~\ref{app:pseudocode}) against 128 for all-node scoring, and on 910 of 914 queries the router picks the same node as routing on the unquantized relevance values. Relative to
routed FP32 scoring, the changes are $+0.0002\pm0.0018$ in coverage,
$-0.0001\pm0.0033$ in set size, and $-0.0006\pm0.0071$ in top-1 accuracy.

In the $K$-averaging replay, the variance of the quantization error falls as $1/K$: each quantized average over a subset of $K$ nodes is compared with the FP32 average over the same nodes, and at 7 and 8 bits the ratio of measured to theoretical variance lies between
$0.998$ and $1.005$ for $K=1,2,4$, with log-log slopes $-0.997$ and $-1.000$. At 8 bits
the replayed error variances are $5.13$, $2.56$, and $1.28$ (all $\times10^{-4}$) for
$K=1$, $2$, and $4$. At $K=4$ and 8 bits, both values average over all 914 questions. The replay value ($1.28\times10^{-4}$) averages over 100 dither draws and the all-node value ($1.31\times10^{-4}$) over five, one per replicate; the difference is within Monte Carlo error. The replay isolates the cancellation of quantization error on
real scores; it is not a benchmark of physical node scaling.

\FloatBarrier
\section{Position relative to related work}\label{app:related}

Table~\ref{tab:position} compares the closest lines of work discussed in
Section~\ref{sec:related} by what they communicate and whether finite bits enter a risk bound.

\begin{table}[htb]
\caption{Among these lines of work, only this one bounds risk in terms of the bits of all three messages: training logits ($b$), the calibration summary ($\Bcal$), and per-query scores ($B_i$). Position relative to the closest lines of work (Section~\ref{sec:related}).}
\label{tab:position}
\centering\footnotesize
\setlength{\tabcolsep}{3.5pt}
\begin{tabular}{@{}>{\raggedright\arraybackslash}p{0.27\textwidth}>{\raggedright\arraybackslash}p{0.28\textwidth}>{\raggedright\arraybackslash}p{0.21\textwidth}>{\raggedright\arraybackslash}p{0.18\textwidth}@{}}
\toprule
Line of work & Communicated object & Finite bits enter a risk bound? & Guarantee\\
\midrule
Compressed parameter exchange & gradients, model or adapter deltas & yes, for optimization error & convergence\\
Distributed mean estimation & quantized vectors & yes, for mean squared error & estimation error\\
Federated distillation (incl.\ CFD) & soft labels or models & no & empirical or convergence\\
Ensemble-distillation theory & ensemble predictions & no & generalization\\
Conformal RAG & none (single site) & no & coverage or risk\\
Federated conformal prediction & calibration summaries, incl.\ quantized or noisy scores & calibration only & coverage\\
Conformal remote inference & per-query binary label masks & yes, masks & long-run false-negative rate\\
\textbf{This work} & logits, calibration summary, scores & $b$, $\Bcal$, $B_i$ & KL risk and coverage\\
\bottomrule
\end{tabular}
\end{table}

\FloatBarrier
\section{Reproducibility, compute, and licenses}\label{app:repro}

\paragraph{Reproducibility.}
Every experiment writes a machine-readable result file that records its
configuration and seeds; the real-model runs also record source hashes and pinned
asset revisions. All reported experimental values, and the figures and tables
presenting them, are generated from these result files by non-interactive scripts.
Code, configurations, and result files are available from the authors on
request.

\paragraph{Compute.}
The synthetic studies run on a CPU in minutes to hours. The real-model studies run
on single GPUs of a shared academic cluster, with nodes simulated
sequentially. A matched-suite run takes $0.2$--$0.5$ GPU-hours per method and
seed (Table~\ref{tab:matched_full}) and a GPT-2-medium FPLD run about $1.3$
(Table~\ref{tab:scale}); the per-run times of every study are recorded in the
result files. No multi-GPU training is used.

\paragraph{Licenses.}
WikiText-2 (CC BY-SA), AG News (research use), and
MMLU (MIT) are used without redistribution. GPT-2 and
GPT-2-medium~\citep[MIT;][]{radford2019gpt2},
GPT-Neo-125M~\citep[MIT;][]{black2021gptneo},
Qwen2.5-1.5B-Instruct~\citep[Apache 2.0;][]{yang2024qwen25}, and all-MiniLM-L6-v2
(Apache 2.0) are loaded from the Hugging Face Hub at pinned revisions.

\paragraph{Use of generative AI and AI-assisted technologies.}
We used generative AI tools (ChatGPT, Google Gemini, and Claude) for language
editing and code formatting support. All data, results, and mathematical
derivations are the authors' own work.

\end{document}